\documentclass[11pt]{article}

\usepackage{authblk}
\usepackage{amsmath, amssymb, amsthm, mathrsfs, bm}
\usepackage{geometry, extarrows, xcolor, graphicx, booktabs, subfigure}
\usepackage{natbib}
\usepackage{titlesec}
\usepackage[colorlinks,citecolor=blue]{hyperref}
\usepackage{enumerate}
\usepackage{array}
\newcolumntype{A}{>{\centering\arraybackslash}m{0.24\columnwidth}}
\newcolumntype{D}{>{\centering\arraybackslash}m{0.16\columnwidth}}
\newcolumntype{B}{>{\centering\arraybackslash}m{0.21\columnwidth}}
\newcolumntype{C}{>{\centering\arraybackslash}m{0.3\columnwidth}}
\usepackage{float}
\usepackage{graphbox}
\usepackage[linesnumbered,ruled,vlined]{algorithm2e}
\SetKwInput{KwInput}{Input}
\SetKwInput{KwReturn}{Return}

\SetCommentSty{mycommfont}

\usepackage{fullpage}

\theoremstyle{plain}
\newtheorem{theorem}{Theorem}
\newtheorem{proposition}{Proposition}
\newtheorem{lemma}{Lemma}

\newtheorem{definition}{Definition}

\theoremstyle{remark}
\newtheorem{example}{Example}
\newtheorem{remark}{Remark}

\DeclareMathOperator*{\argmin}{argmin}

\def\bbE{\mathbb{E}}
\def\bbR{\mathbb{R}}

\def\cD{\mathcal{D}}
\def\cE{\mathcal{E}}
\def\cL{\mathcal{L}}
\def\cN{\mathcal{N}}
\def\cX{\mathcal{X}}
\def\cK{\mathcal{K}}

\def\dkl{\mathrm{KL}}

\def\tr{\mathrm{tr}}

\newtheorem{assumption}{Assumption}

\usepackage{pifont}

\title{Distributional Principal Autoencoders}
\author[1]{Xinwei Shen\thanks{\texttt{xwshen@uw.edu}}}
\author[1]{Nicolai Meinshausen\thanks{\texttt{meinshausen@gmail.com}}}
\affil[1]{Department of Statistics, University of Washington}

\date{}

\begin{document}
\maketitle
 \begin{abstract}
 Dimension reduction techniques typically incur information loss, yielding reconstructions that differ in distribution from the original data. We propose Distributional Principal Autoencoder (DPA), a nonlinear dimensionality reduction method whose encoder is chosen to minimise the unexplained variability of the conditional distribution of the data given its low-dimensional embedding, and whose decoder samples from this conditional. As a result, DPA reconstructions follow the same distribution as the original data regardless of the retained latent dimension. The encoder objective induces an ordering of the latent components and supports an adaptive choice of the retained dimension. Ordinary autoencoders are recovered as the special case with a deterministic decoder, and principal component analysis as the further special case with a linear encoder. With a stochastic encoder, DPA additionally enables posterior inference for latent variables. Our numerical results on climate data, single-cell data, and image benchmarks demonstrate the practical feasibility of DPA and its success in recovering the original distribution of the data. DPA embeddings are shown to preserve meaningful structures of data such as the seasonal cycle for precipitation and cell types for gene expression.
 \end{abstract}


\section{Introduction}\label{sec:intro}
High-dimensional data is common in modern statistics and machine learning. Dimensionality reduction and data compression have been the subject of an extensive body of literature over the past decades. Classical linear approaches such as Principal Component Analysis (PCA) \citep{jolliffe2002principal} and their nonlinear deep learning extensions, autoencoders (AE) \citep{rumelhart1986learning,hinton2006reducing}, share a common formulation in which an encoder maps observations to a low-dimensional embedding and a decoder reconstructs the conditional mean. A separate line of work, including Variational Autoencoders (VAE) \citep{kingma2013auto} and related variants \citep[e.g.][]{makhzani2015adversarial,tolstikhin2018wasserstein}, uses encoder-decoder architectures to fit likelihood-based latent-variable generative models; we discuss the relation to that line in Section~\ref{sec:related}.

After mapping high-dimensional data into a lower-dimensional latent space, there is often a need to reconstruct them in the original space. PCA and autoencoders typically minimise a mean squared reconstruction error. Given a low-dimensional latent value, they aim to reconstruct the conditional mean of all samples that share this embedding, ignoring other characteristics of the conditional distribution such as variance and tail behaviour. As a result, their reconstructed data in general follow a different distribution from the original data, since they preserve the conditional mean but discard the variability around it.

In many applications, such distortion of the data distribution can lead to biased and unreliable downstream statistical estimation or inference. Consider distributional regression with high-dimensional responses, where the goal is to estimate the conditional distribution of a high-dimensional target variable given some predictors. This setting encompasses many modern applications, including high-resolution image generation, spatial climate field prediction, and gene expression modelling. A natural strategy is to (i) reduce the dimension of the response data to lower-dimensional embeddings, (ii) learn a distributional regression model of the embeddings on the predictors, and (iii) reconstruct the response from the predicted embeddings. For example, stable diffusion~\citep{rombach2022high} first reduces the dimension of high-resolution images using a VAE and then applies a diffusion model in the lower-dimensional space to save computational cost. However, if the dimension reduction step fails to preserve the underlying data distribution, then the distribution of the final predictions will in general differ from the original distribution of the target, even if step (ii) is fit perfectly.

As a specific motivating example, we consider as the target global monthly precipitation fields on 1 degree  latitude-longitude grids for CMIP6 models~\citep{gmd-9-1937-2016,gmd-8-3379-2015}, resulting in a spatial dimension of $360\times180$. Compressing such high-dimensional spatial fields into a low-dimensional vector typically loses a large amount of information when aiming for mean reconstructions alone. As shown in Figure~\ref{fig:gcm_illus}, reconstructions from PCA or autoencoders tend to smooth out the fields and fail to capture much information about the distribution of precipitation for each location, especially for small latent dimensions. In the extreme case of latent dimension 0, AE just reconstructs any data by the temporal mean precipitation per location. Figure~\ref{fig:qqplot} illustrates how quantiles of precipitation at a random location behave for true data versus reconstructions.  The high quantiles of the original precipitation data are clearly underestimated by AE or PCA reconstructions. This matters for climate prediction, since underestimating heavy rainfall can lead to inadequate flood-risk planning.

\begin{figure}
\centering
\begin{tabular}{@{}c@{}A@{}A@{}A@{}A@{}}
 &	\small $k=512$ & \small $k=32$ & \small $k=2$ & \small $k=0$ \\
	\rotatebox[origin=c]{90}{\small{True}}
	& \includegraphics[width=0.23\textwidth]{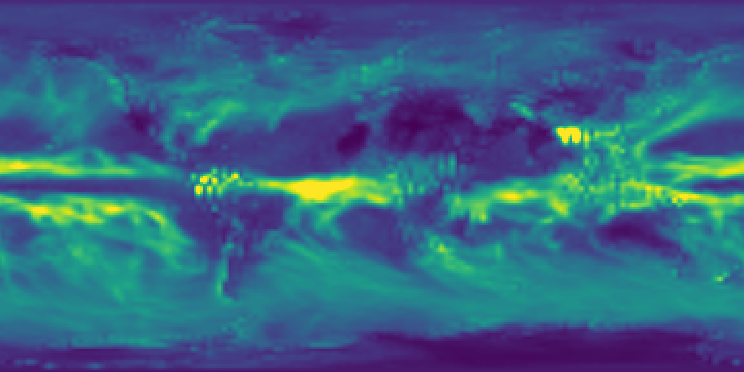} &
	\includegraphics[width=0.23\textwidth]{fig/gcm/illus/true.png} &
	\includegraphics[width=0.23\textwidth]{fig/gcm/illus/true.png} &
	\includegraphics[width=0.23\textwidth]{fig/gcm/illus/true.png} \\
	\rotatebox[origin=c]{90}{\small{PCA}}
	& \includegraphics[width=0.23\textwidth]{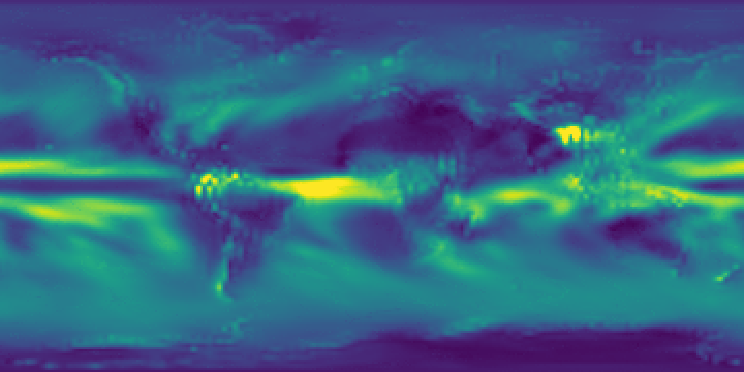} &
	\includegraphics[width=0.23\textwidth]{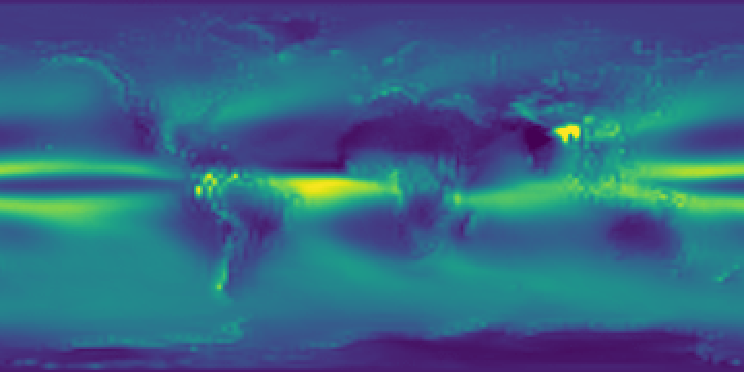} &
	\includegraphics[width=0.23\textwidth]{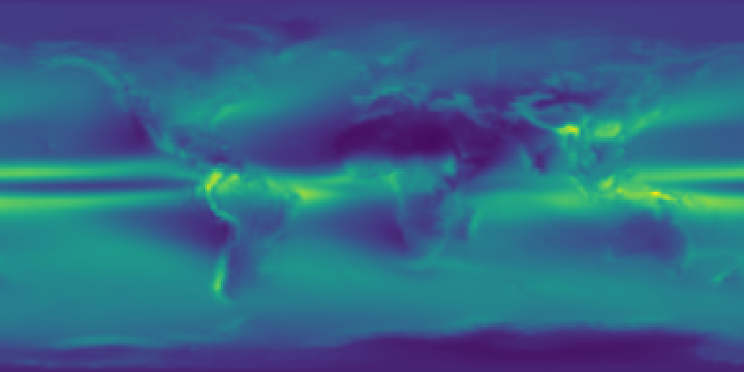} &
	\includegraphics[width=0.23\textwidth]{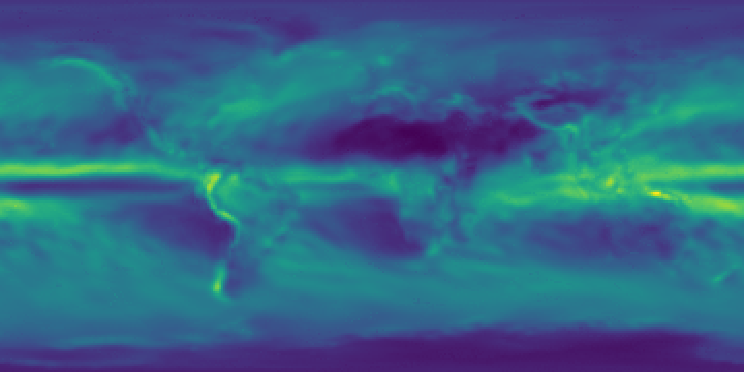} \\	
	\rotatebox[origin=c]{90}{\small{AE}}
	& \includegraphics[width=0.23\textwidth]{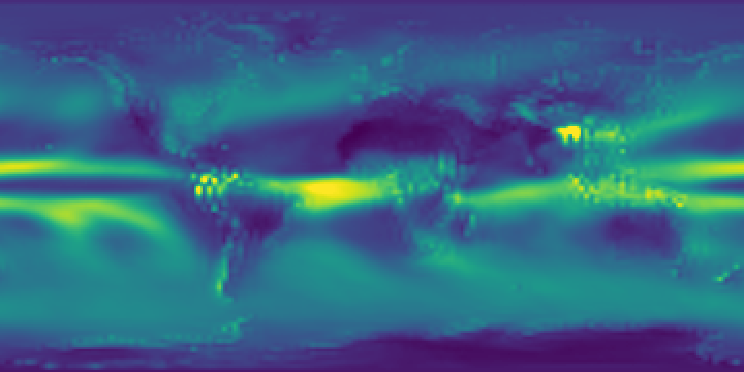} &
	\includegraphics[width=0.23\textwidth]{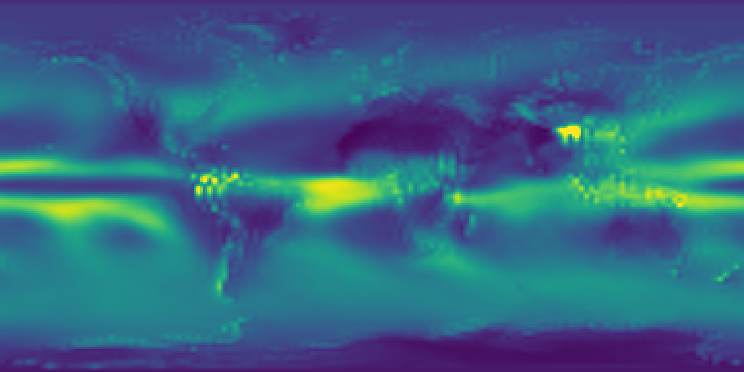} &
	\includegraphics[width=0.23\textwidth]{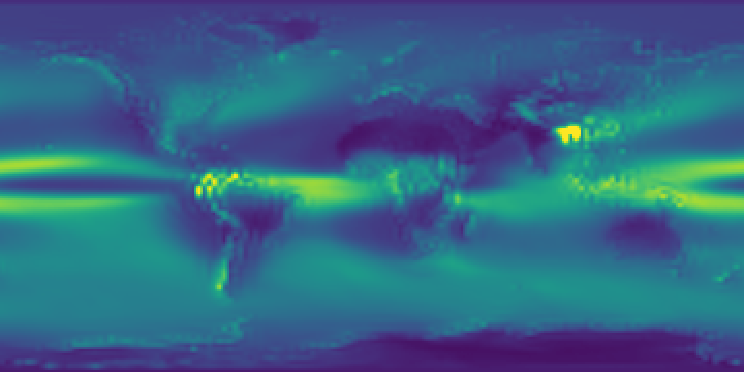} &
	\includegraphics[width=0.23\textwidth]{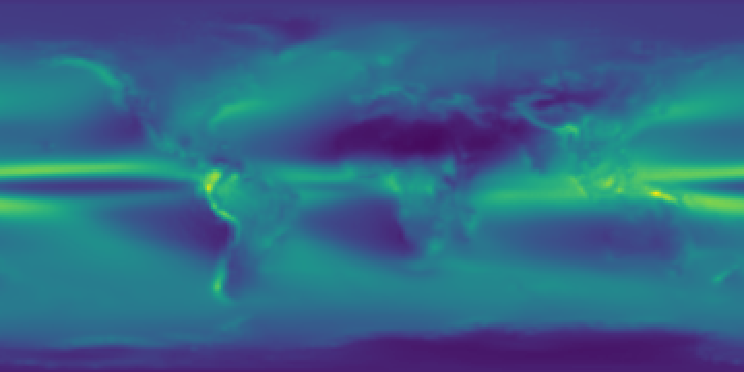} \\	
	\rotatebox[origin=c]{90}{\scriptsize{DPA mean}}
	& \includegraphics[width=0.23\textwidth]{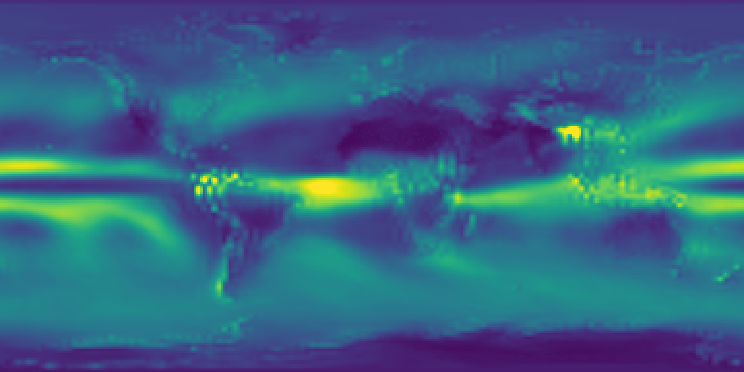} &
	\includegraphics[width=0.23\textwidth]{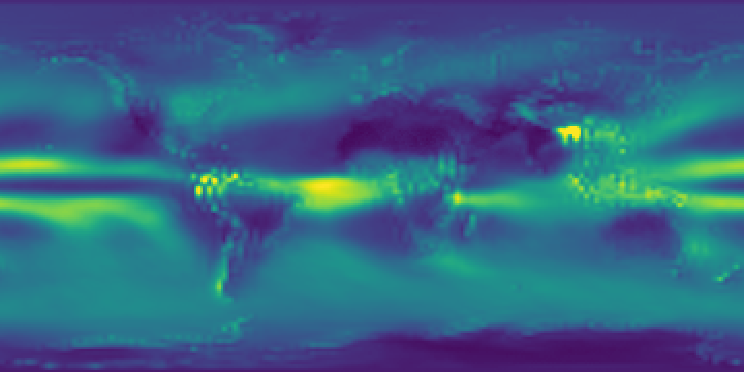} &
	\includegraphics[width=0.23\textwidth]{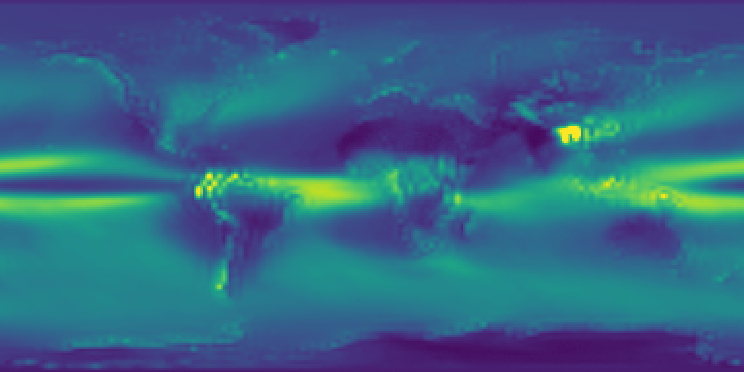} &
	\includegraphics[width=0.23\textwidth]{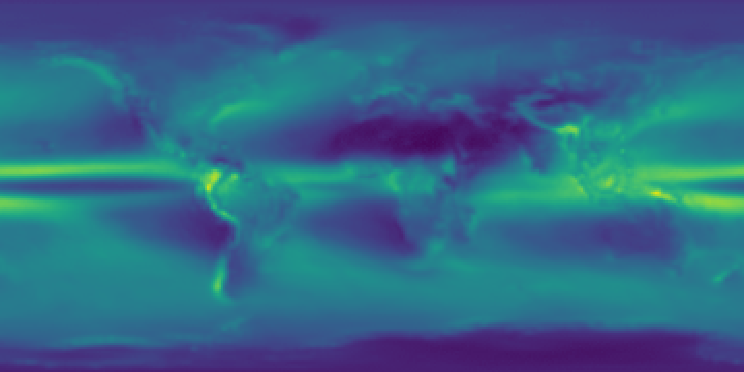} \\	
	\rotatebox[origin=c]{90}{\small{DPA samples}}
	& \includegraphics[width=0.23\textwidth]{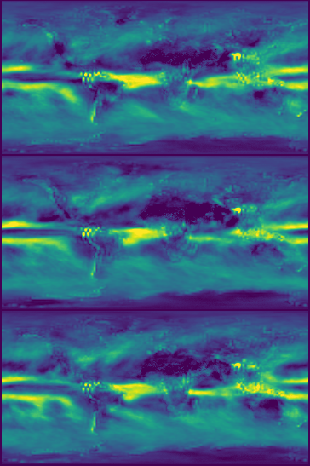} &
	\includegraphics[width=0.23\textwidth]{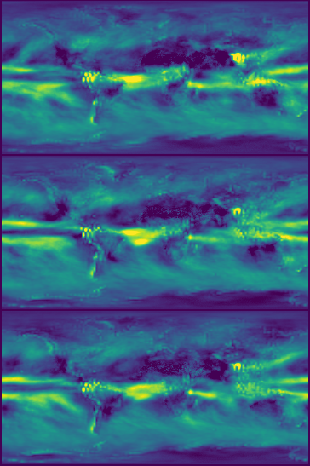} &
	\includegraphics[width=0.23\textwidth]{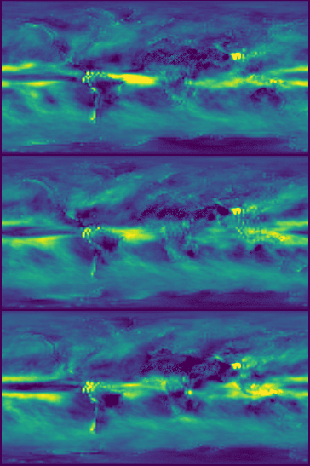} &
	\includegraphics[width=0.23\textwidth]{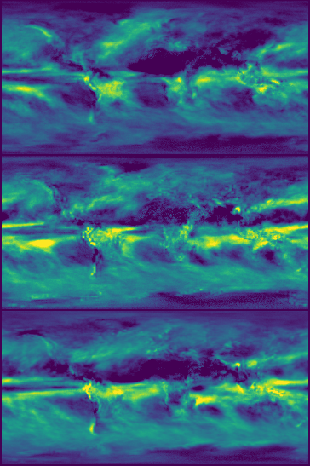} \\
 & \multicolumn{4}{c}{\includegraphics[align=c,width=0.5\textwidth]{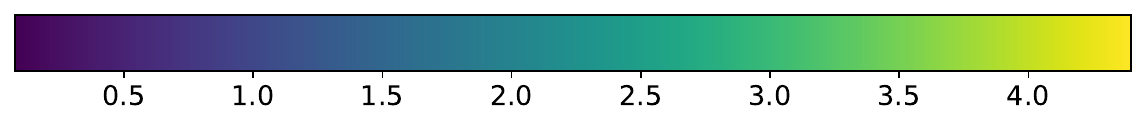}}
\end{tabular}\vspace{-0.1in}
\caption{Global monthly precipitation fields (square-root transformed, original unit $\mbox{kg} \cdot \mbox{m}^{-2} \mbox{s}^{-1}$). Top row: a test data; second row: PCA reconstructions; third row: AE reconstructions; fourth row: mean reconstructions from DPA; remaining three rows: reconstructed samples from DPA. Columns: different latent dimensions $k$. }\label{fig:gcm_illus}
\end{figure}

\begin{figure}
\centering
	\begin{tabular}{@{}ccc@{}}
		\small PCA & \small AE & \small DPA \\
		\includegraphics[width=0.3\textwidth]{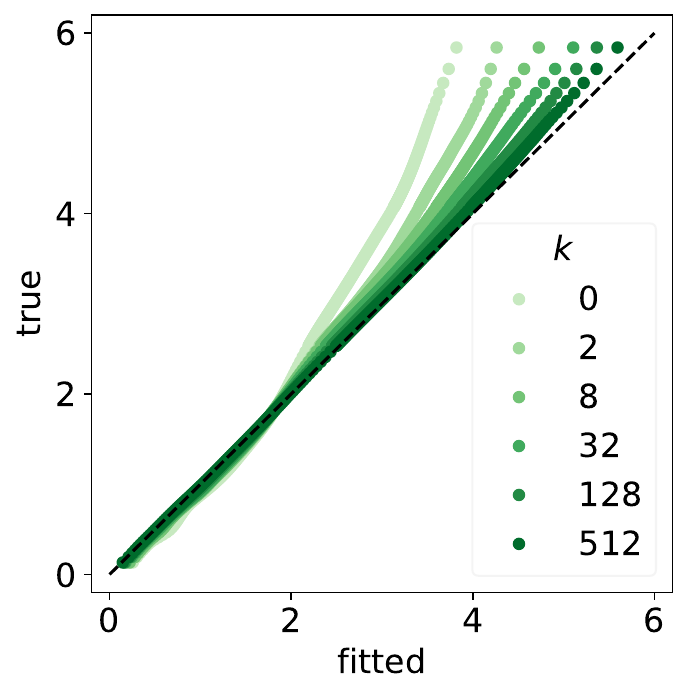} &
		\includegraphics[width=0.3\textwidth]{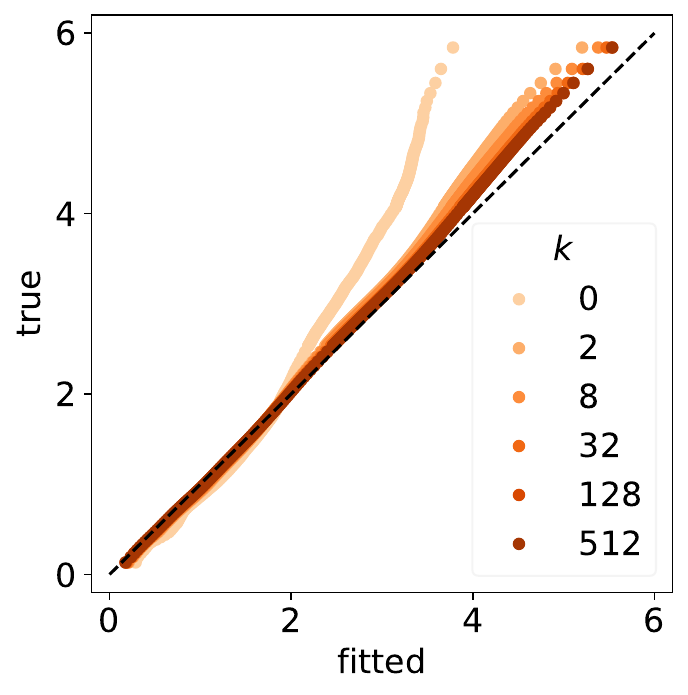} &
		\includegraphics[width=0.3\textwidth]{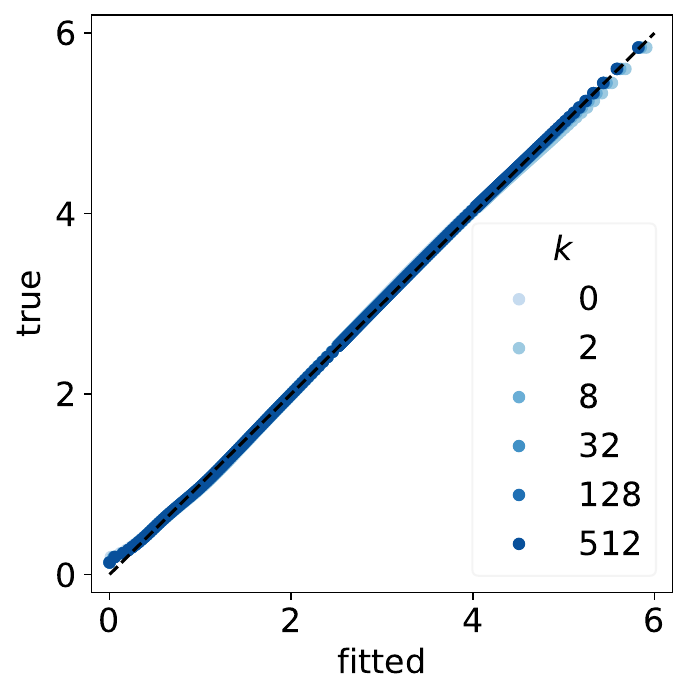} 
	\end{tabular}\vspace{-0.1in}
	\caption{Q--Q plots of precipitations at a random location for test data versus fitted distributions. }\label{fig:qqplot}
\end{figure}

Here, we develop a nonlinear dimensionality reduction method called \emph{Distributional Principal Autoencoder (DPA)}. The first aim is to reconstruct data such that the distribution of the reconstructed data is identical to the distribution of the original data. We approach this goal via an encoder-decoder framework using deep neural networks for both maps. 
%
Let $X\in\bbR^p$ be the data, drawn from a distribution $P^*$, and write $Z\in\bbR^k$ with $k\le p$ for the latent variable, also referred to as the embedding. The `encoder' $e(.):\bbR^p\to\bbR^k$ is a function from the data space to the latent space. The `decoder' is typically a function $d(.):\bbR^k\to\bbR^p$ mapping samples in the $k$-dimensional latent space back to the original $p$-dimensional space; for autoencoders, it minimises the mean squared reconstruction error, so that ideally
\[ d(z) = \bbE[X|e(X)=z],\ \forall z.\]

 Instead of mean reconstruction, the aim for our decoder is \emph{distributional reconstruction}: given an encoder $e$, the decoder samples from the conditional distribution of $X$ given $e(X)=z$. Concretely, the decoder takes as inputs a latent value $z$ and an independent noise variable $\varepsilon\in\bbR^l$ drawn from a pre-specified distribution such as a standard Gaussian, and maps them to the data space, i.e.\ $d(.,.):\bbR^k\times\bbR^l\to\bbR^p$. The noise input is the technical means by which the decoder produces a distribution over reconstructions; what matters is that this induced distribution should match the conditional $X|e(X)=z$. Formally, the decoder aims to achieve
 \begin{equation}\label{eq:dist_recon}
	d(z,\varepsilon) \overset{d}= \big(X|e(X)=z\big),\ \forall z.
 \end{equation}
That is, given an embedding $z$, the distribution of the random variable $d(z,\varepsilon)$ for a random noise input $\varepsilon$ is ideally identical to the conditional distribution of the original data that are mapped by the encoder $e$ to the  embedding $z$. 
 
If \eqref{eq:dist_recon} holds true, the distribution of reconstructed data from our decoder is guaranteed to be the same as the original distribution $P^*$, i.e.
 \begin{equation}\label{eq:eq_marg_dist}
 	d(e(X),\varepsilon) \overset{d}= X.
 \end{equation}
 One can view this equality in distribution \eqref{eq:eq_marg_dist} as a distributional criterion for lossless compression in the sense of retaining the original data distribution. As we will show later, such distributionally lossless compression can be achieved for any compression rate characterised here by the latent dimension $k$.

As an immediate illustration of the distributionally lossless data compression, we show how DPA performs in the global precipitation example earlier. In Figure~\ref{fig:gcm_illus}, the last three rows are three reconstructed samples, i.e.\ samples from the conditional distribution in \eqref{eq:dist_recon} with $z=e(x)$ for the test data $x$ on the top row. Whichever $k$ we pick, the reconstructed samples remain visually realistic as precipitation fields, since they follow the same distribution as the original data, as derived in \eqref{eq:eq_marg_dist}. This is validated further by the Q--Q plots in Figure~\ref{fig:qqplot}, where DPA captures the full distribution precisely for all choices of $k$, especially in the high quantiles, in contrast to PCA and AE. If DPA is used as the dimension reduction technique in the distributional regression pipeline described above, the distribution-preserving property ensures that the target distribution is retained in the final predictions, up to estimation error.
The visual results also suggest that the variability among different DPA samples decreases as $k$ grows larger as more information is retained in the embedding. Furthermore, DPA can also provide mean reconstructions by taking the mean of the reconstructed samples, which leads to similar results as AE, as shown in the fourth row in Figure~\ref{fig:gcm_illus}. 

 Our second aim is to learn the `principal' components, analogous to the ones in PCA, while possessing the above distributional properties. Specifically, we want to learn the components in a way that minimises the unexplained variability in a sense made precise in the next section, with the reduction of variance used in PCA being a special case. Moreover,  while the latent embeddings in autoencoders are typically disordered,  imposing an ordering in the latent space would provide the flexibility to keep only a subset of the latent components, thus enabling an adaptive choice of the latent dimension. 

Our proposed method achieves the two goals of distributional reconstruction and minimisation of unexplained variability, simultaneously through a joint formulation of the encoder and decoder. That is, 
\begin{enumerate}[(i)]
    \item \emph{DPA reconstructions} follow the same distribution as the original data, regardless of the number of components retained, and
    \item \emph{DPA embeddings} aim to explain the most variability of the data with the flexibility to keep only the first $\tilde{k}$ components as the first $\tilde{k}$ (nonlinear) principal components for varying $\tilde{k}<k$.
\end{enumerate}
A detailed description of the method is given in Section~\ref{sec:method}.

The DPA framework can also be adapted to other statistical tasks. In particular, replacing the deterministic encoder by a stochastic one enables inference over latent variables given observed data. In an empirical Bayes (EB) setting, where the likelihood is assumed known, the likelihood can be incorporated into DPA by specifying a decoder that entails it; the DPA objective then yields a procedure that learns the posterior distribution corresponding to the empirical Bayes prior. Further details of this extension are given in Section~\ref{sec:eb}.

In the remainder of the paper, after reviewing the literature, we introduce our methods in Section~\ref{sec:method} together with some theoretical justifications. Section~\ref{sec:empirical} presents empirical results that demonstrate the effectiveness of DPA in distributional reconstructions and in learning meaningful low-dimensional embeddings across a wide range of scientific data and image benchmarks. Section~\ref{sec:eb} extends the framework with a stochastic encoder for posterior inference in an empirical Bayes setting. In Section~\ref{sec:discuss}, we conclude and discuss directions for future work. All proofs are deferred to Appendix~\ref{app:proof}.

\subsection{Related work}\label{sec:related}

\begin{table}[h]
\centering
\caption{Structural comparison of DPA with related encoder-decoder methods. The first column states the mathematical object each method fits. The remaining columns indicate which of the three structural properties of DPA the method targets by design: distributional reconstruction (the encoder-decoder pair preserves the data distribution in population), ordered (principal) embedding, and an adaptive choice of the latent dimension from a single trained model.}\label{tab:related_work}
\resizebox{\textwidth}{!}{%
\begin{tabular}{lcccc}
\toprule
Method & Target object & Distributional & Ordered & Adaptive \\
       &               & reconstruction & embedding & $k$ \\
\midrule
\textbf{DPA} (this paper) & $X|e(X)=z$ for $e$ minimising unexplained variability & yes & yes & yes \\
\midrule
PCA  & $\bbE[X|e(X)=z]$, linear $e$ & no & yes & yes \\
AE   & $\bbE[X|e(X)=z]$, nonlinear $e$ & no & no & no \\
PCA-AE & $\bbE[X|e(X)=z]$, ordered, $k$ separate AE models & no & yes & no \\
Nested / Triangular Dropout & $\bbE[X|e(X)=z]$, ordered, single AE model & no & yes & yes \\
\midrule
VAE, $\beta$-VAE & joint $p(z)p_d(x|z)$, prior $p(z)$ prescribed & no & no & no \\
WAE, AAE & joint, aggregated posterior matched to prior & no & no & no \\
BiGAN & joint matched adversarially & no & no & no \\
\midrule
Deep InfoMax & mutual information $I(X, e(X))$, no decoder & no & no & no \\
\bottomrule
\end{tabular}%
}
\end{table}

Table~\ref{tab:related_work} summarises the relationship between DPA and related encoder-decoder methods. The methods divide into two lines of work with distinct targets. PCA, autoencoders, and their ordered-embedding variants target a conditional of the data given a learned encoder, $X|e(X)=z$, and reconstruct its mean. DPA replaces the mean by the full conditional and belongs to this line of work.

Likelihood-based latent-variable generative models, including VAE and its variants, instead model the joint distribution $p(z)p_d(x|z)$ on $(z,x)$, in which the prior $p(z)$ is prescribed independently of the data and the decoder defines the likelihood. DPA is not a variant of these methods. The DPA target $X|e(X)=z$ and the VAE decoder $p_d(x|z)$ coincide only when the prescribed prior matches the marginal that $e(X)$ induces on the latent space, which VAE does not enforce. The remainder of this section expands on each row of the table.

\paragraph{PCA, autoencoders, and the principal-components line of work.}
DPA generalises the encoder-decoder dimensionality reduction methods exemplified by PCA \citep{jolliffe2002principal} and autoencoders \citep{rumelhart1986learning,hinton2006reducing}. Both target the conditional mean of the data given the embedding. As shown in Section~\ref{sec:method}, taking the DPA exponent $\beta=2$ together with a deterministic decoder recovers the autoencoder objective, and additionally restricting the encoder to be linear recovers PCA. DPA generalises this line of work by replacing mean reconstruction with reconstruction of the full conditional distribution $X|e(X)=z$, while preserving the principal ordering of latent components. Existing nonlinear extensions of PCA in the same vein include PCA-AE \citep{pham2022pca}, which enforces ordering by learning $k$ separate encoders and decoders and thus has a high computational cost, and Nested Dropout \citep{rippel2014learning} and Triangular Dropout \citep{staley2022triangular}, which encourage ordering through random or deterministic truncation of a hidden layer. \citet{ho2023information} present a unified formulation of the Dropout variants:
\begin{equation}\label{eq:iob}
	\sum_{k\in\cK}\omega_k\bbE\big[\|X-d(e_{1:k}(X),\mathbf{0}_{p-k})\|^2\big],
\end{equation}
where $\cK=\{0,1,\dots,p\}$ and $\mathbf{0}_k$ denotes the zero vector of dimension $k$; \eqref{eq:iob} recovers Nested Dropout when $\omega_k$ is geometric and Triangular Dropout when uniform. The weighted-truncation structure of \eqref{eq:iob} resembles our procedure for an adaptive latent dimension in Section~\ref{subsec:adaptive}, but the underlying target is still mean reconstruction.

\paragraph{Likelihood-based latent-variable generative models.}
A separate line of work uses encoder-decoder architectures to fit likelihood-based latent-variable generative models. The Variational Autoencoder (VAE) \citep{kingma2013auto} prescribes a prior $p(z)$ on the latent variable and a parametric family for the decoder likelihood $p_d(x|z)$, with the encoder serving as a variational approximation $q_e(z|x)$ to the intractable posterior. With Gaussian encoder and decoder families, the VAE objective takes the form
\begin{equation*}
	\bbE\big[\|X-d(e(X))\|^2 + \beta\dkl(q_e(z|X),p_z(z))\big],
\end{equation*}
where the Kullback--Leibler term regularises the variational posterior toward the prior; see also $\beta$-VAE \citep{higgins2017beta}. AAE \citep{makhzani2015adversarial} and WAE \citep{tolstikhin2018wasserstein} replace the Kullback--Leibler term by a distance between the aggregated posterior $q_e(z):=\bbE[q_e(z|X)]$ and the prior. The mathematical object targeted by these methods is the joint distribution $p(z)p_d(x|z)$ on $(z,x)$, in which the prior $p(z)$ is prescribed independently of the data. DPA targets a different object: the conditional distribution $X|e(X)=z$ for an encoder $e$ defined by minimising unexplained variability, with no prior over $z$. Even at the population optimum with infinite-capacity encoder and decoder families, the two conditionals coincide only when the prescribed prior happens to match the distribution that $e(X)$ induces on the latent space, which is a constraint VAE does not impose on the encoder. Moreover, since the standard isotropic Gaussian prior used in VAE is rotationally symmetric, its latent coordinates are exchangeable and there is no natural analogue of the principal ordering DPA produces. We give a more detailed comparison in Appendix~\ref{app:vae}.

\paragraph{Other generative models with encoder-decoder structure.}
Generative adversarial networks \citep{goodfellow2014generative} were originally formulated for unconditional generation; subsequent work \citep{Dumoulin2017AdversariallyLI,Donahue2017AdversarialFL,shen2020bidirectional} extends GANs to an encoder-decoder framework, but without guarantees on the conditional distribution \eqref{eq:dist_recon}. Normalising flows \citep{papamakarios2021normalizing} rely on invertible transformations, and diffusion models \citep{sohl2015deep,ho2020denoising} implicitly define an encoder via a stochastic differential equation; both require the latent space to have the same dimension as the data space and are thus unsuitable for dimensionality reduction. Manifold-aware extensions such as M-flows \citep{brehmer2020flows} combine normalising flows with manifold learning to jointly model a lower-dimensional manifold and a density on it, with a different focus from the principal-ordered embedding considered here. Deep InfoMax \citep{hjelm2018learning} maximises mutual information between data and embedding as an alternative encoder criterion, but learns no decoder.

\paragraph{Connection to proper scoring rules.}
The way DPA approaches distributional reconstruction is based on proper scoring rules \citep{gneiting2007strictly}, which \citet{shen2023engression} showed to be effective for learning conditional distributions in regression, including consistency of the population minimiser and stability of the small-$m$ estimator. Here we consider the unsupervised setting in which the latent variable is itself a learned embedding of $X$, and the decoder objective \eqref{eq:obj_dec} inherits these properties for fixed encoder.

\subsection{Software}
Our method is available in the \texttt{Python} package {\footnotesize \texttt{DistributionalPrincipalAutoencoder}} with the source code at {\footnotesize \url{github.com/xwshen51/DistributionalPrincipalAutoencoder}}. 

\section{Distributional Principal Autoencoders}\label{sec:method} 
We describe our method in this section.
We first provide some simple analysis for our encoder and discuss its connections to PCA and AE. In a second step, we propose an approach for learning the decoder that achieves distributional reconstruction. Based on both, we suggest a joint formulation to learn both simultaneously. Finally, we present a procedure that allows an adaptive choice of the latent dimension.

\subsection{DPA encoder}\label{subsec:dpa_enc}

The objective of the encoder is to choose an encoding function of data that retains as much information as possible about the original data. To make this notion precise, we first define the so-called oracle reconstructed distribution. 

\begin{definition}[Oracle reconstructed distribution] \label{def:matching} 
	For a given encoder $e(.):\bbR^p\to\bbR^k$ and a sample $x\in \mathbb{R}^p$, the oracle reconstructed distribution, denoted by $P^*_{e,x}$, is defined as the conditional distribution of $X$, given that its embedding $e(X)$ matches the embedding $e(x)$ of $x$, i.e.\ \[ X|e(X)=e(x) .\]
\end{definition}

This distribution reflects how much information about the data is left unexplained knowing its embedding. For example, for a constant encoder $e(x)\equiv c$, $P^*_{e,x}$ is the same as the original data distribution $P^*$, as $e(x)$ contains no information about $x$; for an invertible function $e$, $P^*_{e,x}$ is a point mass at $x$. Note that $P^*_{e,x}$ is fully determined by the encoder $e$ and the data distribution $P^*$. No prior on the latent variable enters the definition.

Then we look for the encoder that maps data to a principal subspace in the sense of minimising the variability in the oracle reconstructed distribution:
 \begin{equation}\label{eq:obj_enc}
 	\argmin_e \bbE_{X\sim P^*}\bbE_{Y,Y'\overset{\rm iid}\sim P^*_{e,X}}\big[\|Y - Y'\|^\beta\big],
 \end{equation} 
 where $\|.\|$ denotes the Euclidean norm and $\beta\in(0,2]$ is a hyperparameter. The optimisation is over a pre-specified encoder class $\mathcal{E}$, which can be a family of linear functions or a neural network class; we drop this specification from the notation for the encoder and the decoder when there is no ambiguity. The variability of a distribution is measured in \eqref{eq:obj_enc} by the expected distance between two independent draws from it, so the encoder ideally collects the maximal information about $X$ such that the remaining variability is the smallest possible. As we show below, taking $\beta=2$ recovers, after reparametrisation, the encoder objectives of AE and PCA, although we generally adopt $\beta<2$.
 
The following proposition indicates an alternative interpretation of our criterion \eqref{eq:obj_enc} by showing its equivalence to minimising the reconstruction error, which sheds light on the connections to PCA and AE. 
\begin{proposition}\label{prop:two_terms_es_equal}
	For any $\beta\in(0,2]$, we have 
	\begin{equation*}
		\bbE_{X\sim P^*}\bbE_{Y,Y'\overset{\rm iid}\sim P^*_{e,X}}\big[\|Y - Y'\|^{\beta}\big] = \bbE_{X\sim P^*}\bbE_{Y\sim P^*_{e,X}}\big[\|X - Y\|^{\beta}\big].
	\end{equation*}
\end{proposition}
 The reconstruction error on the right-hand side is measured by the expected distance between a sample $X$ of the original data and a sample $Y$ of the oracle reconstructed distribution $P^*_{e,X}$.

We will illustrate the connections between our method and PCA through  the following example.
\begin{example}[Gaussian data and linear encoders]\label{ex:lin_gauss}
	Assume that $X$ follows a multivariate Gaussian distribution with a mean vector $\mu^*$ and a covariance matrix $\Sigma^*$. Let $\Sigma^*=Q\Lambda Q^\top$ be the eigendecomposition of $\Sigma^*$, where $QQ^\top=I_p$ and $\Lambda=\mathrm{diag}(\lambda_1,\dots,\lambda_p)$ with $\lambda_1\ge\dots\ge\lambda_p$ being the eigenvalues. The encoder class is $\{e(x)=M^\top x:M\in\bbR^{p\times k}, M^\top M=I_k\}$.
\end{example}

The following proposition implies that \eqref{eq:obj_enc} yields the same subspace as the one spanned by the first $k$ principal components from PCA.
\begin{proposition}\label{prop:lin_gauss1}
	In the setting of Example~\ref{ex:lin_gauss}, the solution set to \eqref{eq:obj_enc} with $\beta=2$ is uniquely given by $e^*(x)=\Pi Q_{:k}^\top x$, where $\Pi\in\bbR^{k\times k}$ is a permutation matrix.
\end{proposition}

In general, when taking $\beta=2$, the following proposition indicates that \eqref{eq:obj_enc} yields a solution $e^*$ which together with its corresponding conditional mean $\bbE_{Y\sim P^*_{e^*,X}}[Y]$ leads to the minimum mean squared reconstruction error. 
\begin{proposition}\label{prop:mean_recon}
	When taking $\beta=2$, \eqref{eq:obj_enc} is equal to 
\begin{equation*}
	\argmin_e \bbE_{X}\big[\|X - \bbE_{Y\sim P^*_{e,X}}[Y]\|^2\big].
\end{equation*}
\end{proposition}
This reveals the connection between our formulation and autoencoder which is defined by
\begin{equation} \label{eq:obj_ae}
	\argmin_{e,d} \bbE_{X}\big[\|X - d(e(X))\|^2\big].
\end{equation}
Given an encoder $e$, the optimal AE decoder is given by the conditional mean $\bar{d}_e(z)=\bbE_{Y\sim P^*_{e,x}}[Y]$ for $z=e(x)$.
Thus, when the decoder is expressive enough, AE obtains the optimal mean reconstruction. The autoencoder objective is therefore the special case of \eqref{eq:obj_enc} obtained by taking $\beta=2$ and a decoder fixed to the conditional mean; combined with a linear encoder class as in Example~\ref{ex:lin_gauss}, it further specialises to PCA. However, AE does not allow sampling random reconstructions from $P^*_{e,X}$. In the following, we describe a procedure to address this.

\subsection{DPA decoder}\label{subsec:dpa_dec}
As stated in Section~\ref{sec:intro}, given an encoder, the aim for our decoder is distributional reconstruction \eqref{eq:dist_recon}, which can equivalently be written as $d(z, \varepsilon) \sim P^*_{e,x}$ for $z=e(x)$, where $P^*_{e,x}$ is the oracle reconstructed distribution as in Definition \ref{def:matching} and $z=e(x)$ is the embedding value. Note that the oracle reconstructed distribution, by definition, is supported on lower dimensional manifolds (except for $k=0$), so is the distribution of $d(z,\varepsilon)$. Thus it is unlikely for such two distributions to have a non-empty intersection of their respective support. This makes it difficult to achieve \eqref{eq:dist_recon}  with classical approaches such as maximum likelihood estimation, as the KL divergence is often infinite. We will propose a new approach for this objective. 

In addition to distributional reconstruction as its primary objective, a distributional decoder also enables optimisation for the encoder. Proposition~\ref{prop:two_terms_es_equal} indicates the equivalence between the optimisation problem in \eqref{eq:obj_enc} and choosing an encoder $e$ that minimises the following objective,  for any $\lambda<1$, 
\begin{equation}\label{eq:obj_enc1}
	\bbE_{X}\bbE_{Y\sim P^*_{e,X}}\big[\|X - Y\|^{\beta}\big] - \lambda\bbE_{X}\bbE_{Y,Y'\overset{\rm iid}\sim P^*_{e,X}}\big[\|Y - Y'\|^{\beta}\big],
\end{equation}
We cannot directly minimise this objective function yet, though, as we do not have access to  the oracle reconstructed distribution $P^*_{e,X}$. To this end, we need our decoder to enable sampling from $P^*_{e,X}$.  

We propose to replace the oracle $P^*_{e,X}$ in \eqref{eq:obj_enc1} by the distribution induced by a decoder. Specifically, for any fixed encoder $e$, we define the corresponding decoder by the following optimisation problem:
\begin{equation}\label{eq:obj_dec}
	\min_d \left\{\bbE_{X}\bbE_{Y\sim P_{d,e(X)}}\big[\|X - Y\|^{\beta}\big] - \lambda\bbE_{X}\bbE_{Y,Y'\overset{\rm iid}\sim P_{d,e(X)}}\big[\|Y - Y'\|^{\beta}\big]\right\},
\end{equation} 
where $P_{d,z}$ denotes the distribution of $d(z,\varepsilon)$ for any $z\in\bbR^k$. 

For the result below and what follows, we make the following expressivity assumption on the decoder class.

\begin{assumption}[Decoder expressivity]\label{ass:dec_express}
For every $e\in\cE$, there exists $d\in\cD$ such that $d(e(x),\varepsilon)\sim P^*_{e,x}$ for all $x$, with $\varepsilon$ a standard Gaussian noise input.
\end{assumption}

The following proposition suggests that when optimised, the minimal value of \eqref{eq:obj_dec} is the same as the objective function \eqref{eq:obj_enc1}.
\begin{proposition}\label{prop:opt_d}
	Under Assumption~\ref{ass:dec_express}, for any encoder $e$ let $d^*\in\cD$ be such that $P^*_{e,X}=P_{d^*,e(X)}$. Then when taking $\lambda=1/2$ and any $\beta\in(0,2)$, the minimum of \eqref{eq:obj_dec} is achieved if $d=d^*$ and the minimal value is equal to \eqref{eq:obj_enc1}.
\end{proposition}

Note that the objective function in \eqref{eq:obj_dec} with $\lambda= \frac 1 2$ is in fact the expected negative \emph{energy score} between $X$ and the distributional fit $P_{d,e(X)}$. 
Formally, given an observation $x$ and a distribution $P$, the {energy score}, introduced by \citet{gneiting2007strictly} as a popular proper scoring rule for evaluating multivariate distributional predictions, is defined as
$
	\bbE_{X,X'\sim P}[\|X-X'\|^\beta]/2 - \bbE_{X\sim P}[\|X-x\|^\beta],
$ where $X,X'$ are two independent draws from $P$ and $\beta\in(0,2)$. The respective distance for two distributions $P$ and $Q$ is the energy distance~\citep{szekely2003statistics}, which can equivalently be viewed as a maximum mean discrepancy with a distance-induced kernel \citep{sriperumbudur2010hilbert,sejdinovic2013equivalence}, defined as
\begin{equation}\label{eq:energy_distance}
	D(P,Q)=2\bbE\|X-Y\|^\beta-\bbE\|X-X'\|^\beta-\bbE\|Y-Y'\|^\beta
\end{equation}
where $X,X'\overset{\rm iid}\sim P$, $Y,Y'\overset{\rm iid}\sim Q$ and $\beta\in(0,2)$.

\subsection{Joint formulation for encoder and decoder}
The above results suggest to use \eqref{eq:obj_dec} as the objective equivalent to the original objectives \eqref{eq:obj_enc1} or \eqref{eq:obj_enc} for the encoder, which leads to a joint formulation for both the encoder and decoder.
Concretely, we define the population version of DPA for a fixed $k$ as the solution to the joint optimisation problem that combines \eqref{eq:obj_enc1} and \eqref{eq:obj_dec} with $\lambda=1/2$:
\begin{equation}\label{eq:obj_joint_fix_k}
	(e^*,d^*)\in\argmin_{e,d} \left\{\bbE_{X}\bbE_{Y\sim P_{d,e(X)}}\big[\|X - Y\|^{\beta}\big] - \frac{1}{2}\bbE_{X}\bbE_{Y,Y'\overset{\rm iid}\sim P_{d,e(X)}}\big[\|Y - Y'\|^{\beta}\big]\right\},
\end{equation}
where $\beta\in(0,2)$. 

The following theorem characterises the DPA solution at the population level.

\begin{theorem}\label{thm:dpa_onek}
	Under Assumption~\ref{ass:dec_express}, the DPA solution $(e^*,d^*)$ defined in \eqref{eq:obj_joint_fix_k} satisfies the following.
\begin{enumerate}[(i)]
\item The DPA encoder $e^*$ minimises the unexplained variability in the sense of \eqref{eq:obj_enc}:
\begin{equation*}
 	e^* \in \argmin_e \bbE_{X\sim P^*}\bbE_{Y,Y'\overset{\rm iid}\sim P^*_{e,X}}\big[\|Y - Y'\|^\beta\big].
 \end{equation*}
\item The DPA decoder $d^*$ induces the oracle reconstructed distribution and hence satisfies for all $x$
$$d^*(e^*(x),\varepsilon)\sim P^*_{e^*,x}.$$
\end{enumerate}
\end{theorem}

\begin{remark}\label{rem:expressivity}
Assumption~\ref{ass:dec_express} is a population-level expressivity condition on the chosen decoder class. By the noise outsourcing lemma \citep{austin2015exchangeable}, for each fixed $e$ there exists \emph{some} measurable function that represents $P^*_{e,x}$ as the pushforward of a Gaussian noise input; the assumption strengthens this to existence within $\cD$. In practice $\cD$ is a neural network class, and Theorem~\ref{thm:dpa_onek} should be read as a consistency statement: if the architecture is expressive enough to represent the oracle reconstructed distribution, the joint DPA objective recovers it. The assumption is not directly verifiable, and finite-sample optimisation introduces additional approximation and estimation error not addressed by the theorem.
\end{remark}

After reparametrisation via the decoder, \eqref{eq:obj_joint_fix_k} is equivalent to 
\begin{equation*}
	\argmin_{e,d} \left\{\bbE_{X}\bbE_{\varepsilon}\big[\|X - d(e(X),\varepsilon)\|^{\beta}\big] - \frac{1}{2}\bbE_{X}\bbE_{\varepsilon,\varepsilon'}\big[\|d(e(X),\varepsilon) - d(e(X),\varepsilon')\|^{\beta}\big]\right\},
\end{equation*}
where $\varepsilon$ and $\varepsilon'$ are two independent draws from a standard Gaussian.
Note that for a deterministic decoder $d$ and $\beta=2$, the second term in the objective vanishes and we again recover the AE objective as in \eqref{eq:obj_ae}.

Given an iid sample $\{X_1,\dots, X_n\}$ of data $X$ from $P^*$ and an iid sample $\{\varepsilon_{ij}:i=1,\dots,n,j=1,\dots,m\}$ of noise $\varepsilon$ from $\cN(0,I_l)$, the finite-sample version of DPA is defined via plug-in estimates as
\begin{equation*}
	\argmin_{e,d} \frac{1}{n}\sum_{i=1}^n\left[\frac{1}{m}\sum_{j=1}^m\|X_i - d(e(X_i),\varepsilon_{ij})\|^\beta - \frac{1}{2m(m-1)}\sum_{j,j'=1}^m\|d(e(X_i),\varepsilon_{ij})-d(e(X_i),\varepsilon_{ij'})\|^\beta\right].
\end{equation*}
A more detailed algorithm based on the finite-sample loss is given at the end of this section.

In some applications it is useful to regularise the embedding space to satisfy distributional constraints, for example to encourage the embedding to follow a pre-specified distribution as in VAE. Such constraints can be accommodated by adding to the objective in \eqref{eq:obj_joint_fix_k} the regularisation term
\begin{equation*}
	\lambda_1\left[\bbE\|Z-e(X)\|-\frac12\bbE\|e(X)-e(X')\|\right],
\end{equation*}
where $Z\sim Q_z$ is drawn from the target distribution and $X,X'\overset{\rm iid}\sim P^*$. This is the expected negative energy score between $Z$ and the distribution of $e(X)$, which is minimised if and only if $e(X)\sim Q_z$.

\subsection{DPA for an adaptive latent dimension}\label{subsec:adaptive}
To allow flexibility in selecting the latent dimension, we extend the above DPA procedure for a fixed latent dimension $k$ to be valid for multiple $k\in\{0,1,\dots,p\}$, thus allowing  an adaptive choice of the retained dimension of the data. Applications for which the reconstruction quality is of high priority would favour a larger $k$, while users with a limited computational or memory budget would typically prefer a smaller $k$, but the choice will also depend on how much information is lost when reducing to smaller latent dimensions. Note that we also include the case of $k=0$ which corresponds to an unconditional generation problem where the decoder, a.k.a.\ the generator, takes only noise $\varepsilon$ as its input.

We start with describing how the DPA procedure for a fixed $k$ proposed in the previous section fits into our current model framework with a varying $k$. 
For the following, we first expand the latent space to the full dimension of data and later impose constraints to the subspaces in a way that induces an ordering. 
Specifically, we consider an encoder $e(.):\bbR^p\to\bbR^p$ and a decoder $d(.):\bbR^p\to\bbR^p$ (part of whose arguments can be noise). 
Let noise $\varepsilon\sim\cN(0,I_p)$ which also has the same dimension as data. 
Consider now any $k\in\{0,1,\dots,p\}$. 
 For the encoder, we simply adopt the criterion in \eqref{eq:obj_enc} to the first $k$ components of our current encoder in $\bbR^p$, which becomes
\begin{equation*}
	\argmin_e \bbE_{X}\bbE_{Y,Y'\overset{\rm iid}\sim P^*_{e_{1:k},X}}\big[\|Y - Y'\|^{\beta}\big],
\end{equation*}
while the remaining $p-k$ components of the encoder are left unconstrained for now.
For the decoder, the following proposition inspires an adaptive way that allows the decoder to take a varying number of components from the encoder. 
\begin{proposition}\label{prop:simu_opt_d}
	For any invertible map $e$ such that $e(X)\overset{d}= \varepsilon$ (assume it exists), there exists an optimal decoder $d^*$ such that $d^*(e_{1:k}(x),\varepsilon_{(k+1):p})\sim P^*_{e_{1:k},x}$ for all $k$. 
\end{proposition}
It indicates that for any encoder $e$ whose push-forward of $X$ coincides with the noise distribution, there exists a joint optimal decoder that takes as arguments the first $k$ components of the encoder $e_{1:k}(x)$ together with additional $(p-k)$-dimensional noise $\varepsilon_{(k+1):p}$ and that matches the oracle reconstructed distribution $P^*_{e_{1:k},X}$ for all $k$ simultaneously. The condition $e(X)\overset{d}=\varepsilon$ is not a real restriction: any encoder can be composed with a normalising transformation to satisfy it without changing the subspace it identifies, and in practice the neural network classes we use are expressive enough for this to hold.

Motivated by this result, we devise the following procedure to sample from a reconstructed distribution for any $k\in\{0,1,\dots,p\}$:
\begin{enumerate}[(i)]
	\item For a data sample $X$, obtain the latent variable from the encoder, i.e.\ $Z=e(X)$.
	\item Keep only the first $k$ components of the latent variable while filling the remaining $p-k$ dimensions with i.i.d.\ standard Gaussian noise entries, i.e.\ $\tilde{Z}=(Z_1,\dots,Z_k,\varepsilon_{k+1},\dots,\varepsilon_p)$. 
	\item Obtain a reconstructed sample from the decoder, i.e.\ $Y=d(\tilde{Z})$.
\end{enumerate}
To learn the oracle reconstructed distribution $P^*_{e_{1:k},X}$, the decoder formulation in \eqref{eq:obj_dec} becomes
\begin{equation}\label{eq:obj_dec_k}
	\argmin_d \left\{\bbE_{X}\bbE_{Y\sim P_{d,e_{1:k}(X)}}\big[\|X - Y\|^{\beta}\big] - \frac{1}{2}\bbE_{X}\bbE_{Y,Y'\overset{\rm iid}\sim P_{d,e_{1:k}(X)}}\big[\|Y - Y'\|^{\beta}\big]\right\},
\end{equation}
where $P_{d,e_{1:k}(x)}$ denotes the distribution of $d(e_{1:k}(x),\varepsilon_{(k+1):p})$ for any $x$, and $\beta\in(0,2)$. Proposition~\ref{prop:simu_opt_d} suggests that $d^*$ defined therein is a solution to \eqref{eq:obj_dec_k}, simultaneously for all $k\in\{0,1,\dots,p\}$.

Next, to accommodate for multiple $k$'s simultaneously, we propose to minimise a weighted average of the loss functions for all $k$ with respect to a common encoder:
\begin{equation}\label{eq:obj_enc_all_k}
	\argmin_e \sum_{k=0}^p\omega_k\bbE_{X}\bbE_{Y,Y'\overset{\rm iid}\sim P^*_{e_{1:k},X}}\big[\|Y - Y'\|^{\beta}\big],
\end{equation}
where the weights $\omega_k\in[0,1]$ for all $k$ and $\sum_{k=0}^p\omega_k=1$; for example, one can take uniform weights $\omega_k\equiv 1/(p+1)$. Note that DPA with a fixed $k=\tilde{k}$ introduced in the previous section is a special case of \eqref{eq:obj_enc_all_k} with $\omega_k=1\{k=\tilde{k}\}$. 
In the following proposition, we connect this formulation to PCA.
\begin{proposition}\label{prop:lin_gauss2}
	In the setting of Example~\ref{ex:lin_gauss}, the solution to optimisation problem \eqref{eq:obj_enc_all_k} with $\beta=2$ is uniquely given by $e^*(x)=Q^\top x$, independent of the choice of weights. 
\end{proposition}
Compared to the result in Proposition~\ref{prop:lin_gauss1}, with \eqref{eq:obj_enc_all_k}, we can recover not only the same subspace as the one obtained by PCA, but also the ordering of the principal directions. Moreover, in this case, the solution to \eqref{eq:obj_enc_all_k} does not depend on the choice of weights as long as all weights are strictly positive, and each term within the summation can be minimised simultaneously. It remains an open problem whether such simultaneous optimality holds more generally. If it cannot hold, the weights will influence the trade-off among different $k$'s in the optimal solution. We will suggest a default choice of the weights that work well in all our numerical experiments in Section~\ref{sec:empirical}. In addition, DPA for a varying latent dimension could help guide the choice of the most appropriate latent dimension, a long-standing question for dimension reduction techniques in general. 

We complete the algorithm for solving \eqref{eq:obj_enc_all_k} via a decoder to match the oracle reconstructed distributions $P^*_{e_{1:k},X}$ for any given $e(.)$ and all $k$, defined as follows:
\begin{equation}\label{eq:obj_d_avg}
	\argmin_d \sum_{k=0}^p\omega_k\left[\bbE_{X}\bbE_{Y\sim P_{d,e_{1:k}(X)}}\big[\|X - Y\|^{\beta}\big] - \frac{1}{2}\bbE_{X}\bbE_{Y,Y'\overset{\rm iid}\sim P_{d,e_{1:k}(X)}}\big[\|Y - Y'\|^{\beta}\big]\right].
\end{equation}
Based on Proposition~\ref{prop:simu_opt_d}, we know that there exists a $d^*$ that minimises all the terms for all $k$ in the above objective function.

Finally, similar to \eqref{eq:obj_joint_fix_k}, we define the population version of DPA for varying latent dimensions $k=0,1,\dots,p$ by  a joint formulation that combines \eqref{eq:obj_enc_all_k} and \eqref{eq:obj_d_avg}:
\begin{equation}\label{eq:obj_joint}
	\argmin_{e,d}\sum_{k=0}^p\omega_k\left[\bbE_{X}\bbE_{Y\sim P_{d,e_{1:k}(X)}}\big[\|X - Y\|^{\beta}\big] - \frac{1}{2}\bbE_{X}\bbE_{Y,Y'\overset{\rm iid}\sim P_{d,e_{1:k}(X)}}\big[\|Y - Y'\|^{\beta}\big]\right].
\end{equation}
Similar to Theorem~\ref{thm:dpa_onek}, the solution to \eqref{eq:obj_joint} consists of an optimal encoder and decoder that satisfy \eqref{eq:obj_enc_all_k} and \eqref{eq:obj_d_avg}, respectively.

In practice, enumerating over $\{0,1,\dots,p\}$ is sometimes unnecessary and computationally expensive, especially when $p$ is very large. A more efficient and practical way is to consider a subset $\cK\subseteq\{0,1,\dots,p\}$ consisting of all possible numbers of components that one may want to keep. 

We parametrize our encoder $e$ and decoder $d$ by neural networks and solve problem \eqref{eq:obj_joint} by gradient descent algorithms.
We summarise the finite-sample procedure for DPA for varying latent dimensions in Algorithm~\ref{alg:dpa}. As noted above, the algorithm includes DPA for a fixed $k$ as a special case. For ease of presentation, we present here a gradient descent algorithm with a full-batch of $X$ but a mini-batch of $\varepsilon$ of size 2 for each $X$ sample. In practice, one can also use a mini-batch of $X$ at each gradient step. For the exponent in the objective function, we take $\beta=1$ by default.

{\centering
\begin{minipage}{\linewidth}
\vskip 0.1in
\begin{algorithm}[H]
\DontPrintSemicolon
\KwInput{Sample $\{X_1,\dots,X_n\}$, set $\cK$, $K=\max(\cK)$, weights $\omega_k$, initial $e,d$}
\For{each iteration}{
\For{$k\in\cK$}{
\For{$i\in\{1,\dots,n\}$}{
Compute $Z_i=e(X_i)$ and sample $\varepsilon_i,\varepsilon'_i\overset{\rm iid}\sim \cN(0,I_p)$\\
Draw  samples $\hat{X}_i=d(Z_{i,1:k},\varepsilon_{i,(k+1):p})$ and $\hat{X}'_i=d(Z_{i,1:k},\varepsilon'_{i,(k+1):p})$ }
Compute loss $L_k=\frac{1}{2n}\sum_{i=1}^n\big[\|X_i-\hat{X}_i\|+\|X_i-\hat{X}'_i\|-\|\hat{X}_i-\hat{X}'_i\|\big]$
}
Compute loss $\sum_{k\in\cK}\omega_k L_k$ and update parameters of $e,d$ by descending its gradients
}
\KwReturn{$e,d$}
\caption{Distributional Principal Autoencoder (DPA)}
\label{alg:dpa}
\end{algorithm}
\end{minipage}
\vskip 0.1in
\par
}

The choice of $m=2$ noise draws per sample in Algorithm~\ref{alg:dpa} matches the estimator used in the regression setting by \citet{shen2023engression}, where its stability, including the absence of the variance-collapse failure mode one might expect from such a small $m$, is studied both theoretically and empirically. The same analysis carries over to the unsupervised setting considered here.

\section{Empirical results}\label{sec:empirical}

We present empirical studies on data sets across various domains. We first consider two benchmark image data sets: \textsc{mnist}~\citep{lecun1998mnist}, a data set of $28\times28$ hand-written digits, and \textsc{disk}~\citep{ho2023information} where each sample is a $32\times32$ image consisting of two disks with randomly generated radiuses and x/y-positions for each disk, leading to an intrinsic dimension of 6. We then turn our focus to two scientific scenarios: climate data and single-cell RNA sequencing data. Specifically, we consider monthly regional temperature (\textsc{r-temp}) and precipitation (\textsc{r-precip}) data in central Europe with dimensions of $128\times128$, and global precipitation fields (\textsc{g-precip}) with a dimension of $360\times180$. Additionally, we consider 8 publicly available single-cell data sets (\textsc{sc1-8}) with dimensions ranging from 1k to 6k. See Appendix~\ref{app:expe_details} for more details of data and all experimental settings. 
\subsection{Reconstructions}\label{subsec:recon}
We investigate the reconstruction performance through visualisations and quantitative metrics. The empirical results below show two things: (a) DPA reconstructions follow the same distribution as the original data at every retained latent dimension $k$, and (b) the latent ordering induced by the encoder objective is meaningful, with the first few components carrying the most informative directions of variation. We compare DPA with methods that produce ordered embeddings, namely PCA and the PCA-variants of AE given in \eqref{eq:iob} with the same $\cK$ as DPA, which we still call `AE'. Throughout all experiments, we use uniform weights $\omega_k\equiv1/|\cK|$. Comparisons against VAE and WAE are deferred to Appendix~\ref{app:add_exp}, where we report results at a fixed $k=2$ since these methods do not produce a single embedding valid simultaneously across multiple latent dimensions.

Figures~\ref{fig:mnist_recon}-\ref{fig:rcmp_recon} demonstrate visually the reconstructions for the 4 image data sets on the held-out test set (results on \textsc{g-precip} were illustrated in Figure~\ref{fig:gcm_illus}). 
For $k=0$, DPA produces samples from the unconditional data distribution, while PCA and AE degenerate into a point mass at the mean. As shown in the panels of the last column, DPA samples are realistic digits, disks, and regional maps for temperature or precipitation, while PCA and AE reconstructions are just the unconditional means of the full data set, with little informative content. Similar behaviour persists for small but positive $k$: DPA generates reasonable samples regardless of $k$, whereas PCA and AE tend to blur out the original data.

Even for sufficiently large $k$ (panels on the left), PCA still fails to produce high-quality results due to its linearity restriction (e.g.\ on \textsc{mnist} and \textsc{disk}), while AE sometimes overfits (e.g.\ in Figure~\ref{fig:disk_recon}). In contrast, DPA reconstructions become fairly close to the raw data as $k$ grows, benefiting from its expressive model capacity (versus PCA) without suffering severely from overfitting (versus AE).

\begin{figure}
\centering
\begin{tabular}{c@{}cccc}
	& \small $k=32$ & \small $k=8$ & \small $k=2$ & \small $k=0$\\
	\rotatebox[origin=c]{90}{\small{true}}\hspace{4pt}\smallskip &
	\includegraphics[width=0.2\textwidth,align=c]{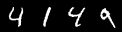} &
	\includegraphics[width=0.2\textwidth,align=c]{fig/mnist/true_i1.png} &
	\includegraphics[width=0.2\textwidth,align=c]{fig/mnist/true_i1.png} &
	\includegraphics[width=0.2\textwidth,align=c]{fig/mnist/true_i1.png} \\
	\rotatebox[origin=c]{90}{\small{PCA}}\hspace{4pt}\smallskip &
	\includegraphics[width=0.2\textwidth,align=c]{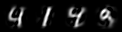} &
	\includegraphics[width=0.2\textwidth,align=c]{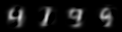} &
	\includegraphics[width=0.2\textwidth,align=c]{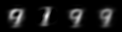} &
	\includegraphics[width=0.2\textwidth,align=c]{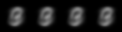}\\
	\rotatebox[origin=c]{90}{\small{AE}}\hspace{4pt}\smallskip &
	\includegraphics[width=0.2\textwidth,align=c]{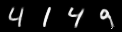} &
	\includegraphics[width=0.2\textwidth,align=c]{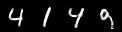} &
	\includegraphics[width=0.2\textwidth,align=c]{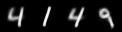} &
	\includegraphics[width=0.2\textwidth,align=c]{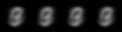} \\
	\rotatebox[origin=c]{90}{\small{DPA}}\hspace{4pt} &
	\includegraphics[width=0.2\textwidth,align=c]{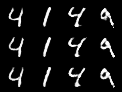} &
	\includegraphics[width=0.2\textwidth,align=c]{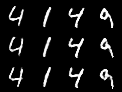} &
	\includegraphics[width=0.2\textwidth,align=c]{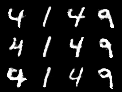} &
	\includegraphics[width=0.2\textwidth,align=c]{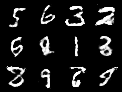} 
\end{tabular}
\caption{Reconstructions for \textsc{mnist}.}\label{fig:mnist_recon}
\end{figure}

\begin{figure}
\centering
\begin{tabular}{@{}c@{}DDDDD}
	& \small $k=8$ & \small $k=6$ & \small $k=4$ & \small $k=2$ & \small $k=0$\\
	\rotatebox[origin=c]{90}{\small{true}}\hspace{4pt}\smallskip &
	\includegraphics[width=0.18\textwidth,align=c]{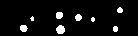} &
	\includegraphics[width=0.18\textwidth,align=c]{fig/kdisk/true_i0.png} &
	\includegraphics[width=0.18\textwidth,align=c]{fig/kdisk/true_i0.png} &
	\includegraphics[width=0.18\textwidth,align=c]{fig/kdisk/true_i0.png} &
	\includegraphics[width=0.18\textwidth,align=c]{fig/kdisk/true_i0.png} \\
	\rotatebox[origin=c]{90}{\small{PCA}}\hspace{4pt}\smallskip &
	\includegraphics[width=0.18\textwidth,align=c]{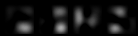} &
	\includegraphics[width=0.18\textwidth,align=c]{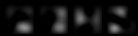} &
	\includegraphics[width=0.18\textwidth,align=c]{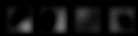} &
	\includegraphics[width=0.18\textwidth,align=c]{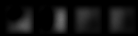} &
	\includegraphics[width=0.18\textwidth,align=c]{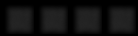} \\
	\rotatebox[origin=c]{90}{\small{AE}}\hspace{4pt}\smallskip &
	\includegraphics[width=0.18\textwidth,align=c]{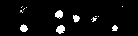} &
	\includegraphics[width=0.18\textwidth,align=c]{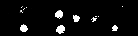} &
	\includegraphics[width=0.18\textwidth,align=c]{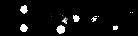} &
	\includegraphics[width=0.18\textwidth,align=c]{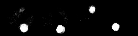} &
	\includegraphics[width=0.18\textwidth,align=c]{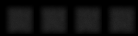} \\
	\rotatebox[origin=c]{90}{\small{DPA}}\hspace{4pt}\smallskip &
	\includegraphics[width=0.18\textwidth,align=c]{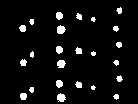} &
	\includegraphics[width=0.18\textwidth,align=c]{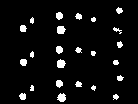} &
	\includegraphics[width=0.18\textwidth,align=c]{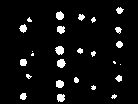} &
	\includegraphics[width=0.18\textwidth,align=c]{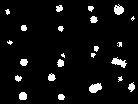} &
	\includegraphics[width=0.18\textwidth,align=c]{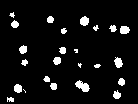} \\
\end{tabular}
\caption{Reconstructions for \textsc{disk}.}\label{fig:disk_recon}
\end{figure}

\begin{figure}
\centering
\begin{tabular}{c@{}BBBB}
	& \small $k=32$ & \small $k=8$ & \small $k=2$ & \small $k=0$\\
	\rotatebox[origin=c]{90}{\small{true}}\hspace{4pt}\smallskip &
	\includegraphics[width=0.23\textwidth,align=c]{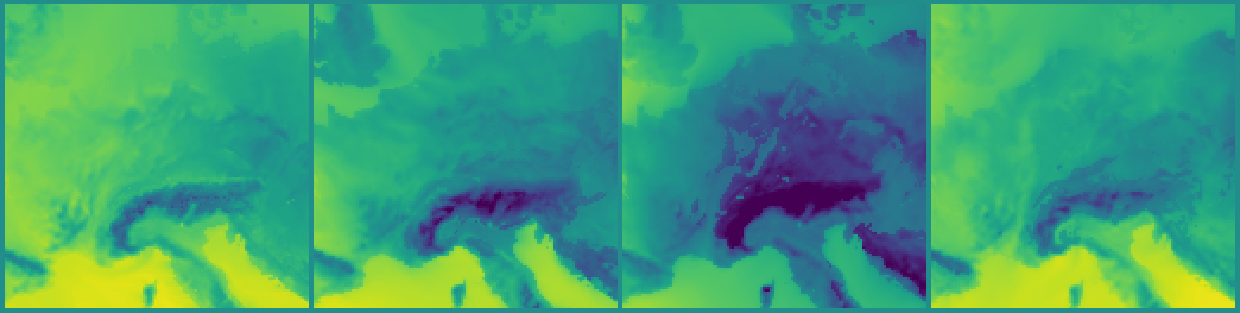} &
	\includegraphics[width=0.23\textwidth,align=c]{fig/rcm_t/true_i7.png} &
	\includegraphics[width=0.23\textwidth,align=c]{fig/rcm_t/true_i7.png} &
	\includegraphics[width=0.23\textwidth,align=c]{fig/rcm_t/true_i7.png} \\
	\rotatebox[origin=c]{90}{\small{PCA}}\hspace{4pt}\smallskip &
	\includegraphics[width=0.23\textwidth,align=c]{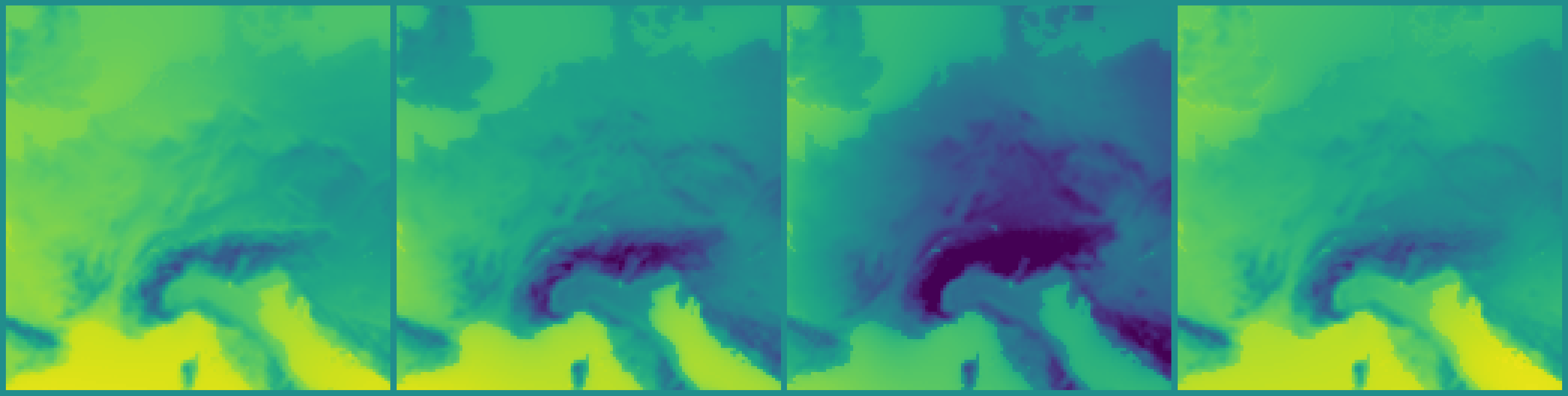} &
	\includegraphics[width=0.23\textwidth,align=c]{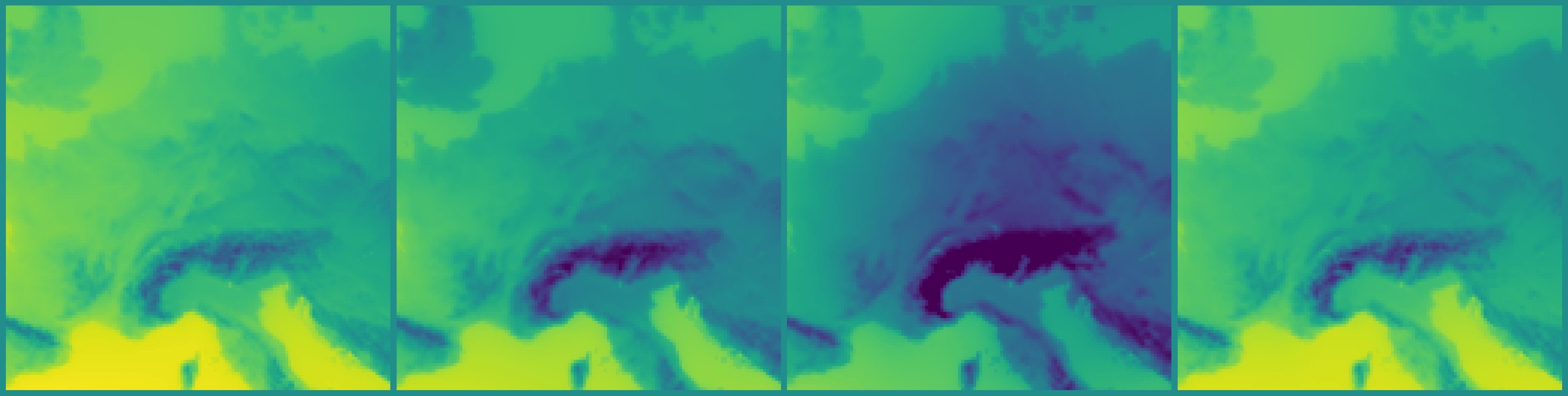} &
	\includegraphics[width=0.23\textwidth,align=c]{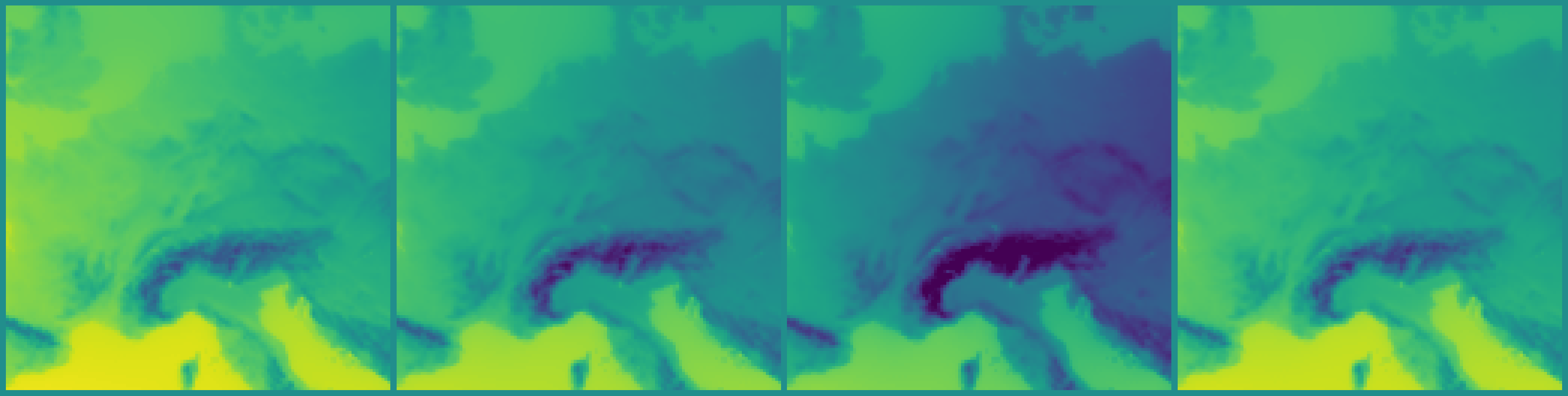} &
	\includegraphics[width=0.23\textwidth,align=c]{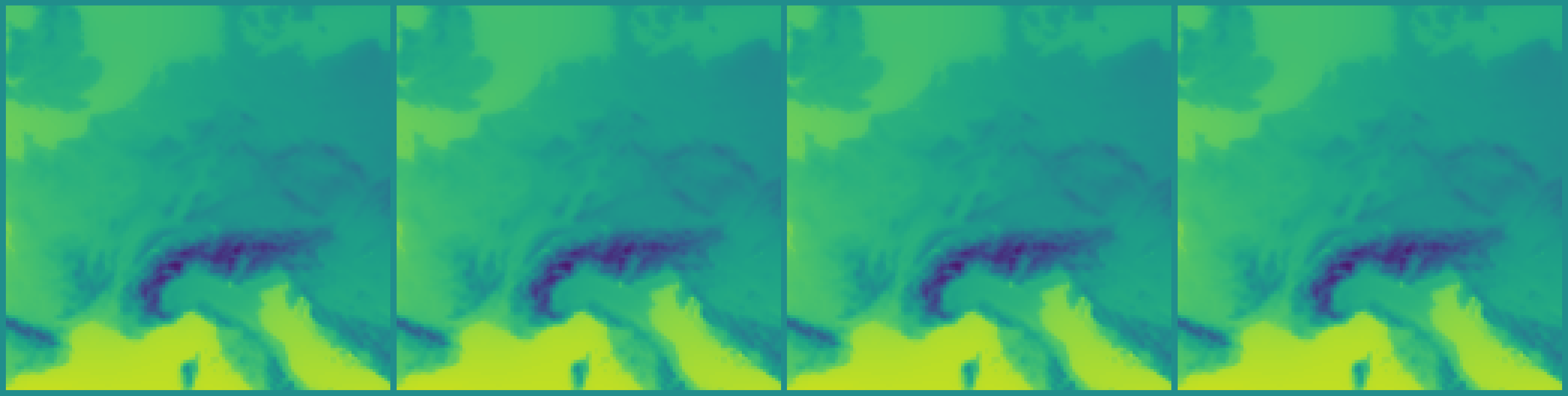}\\
	\rotatebox[origin=c]{90}{\small{AE}}\hspace{4pt}\smallskip &
	\includegraphics[width=0.23\textwidth,align=c]{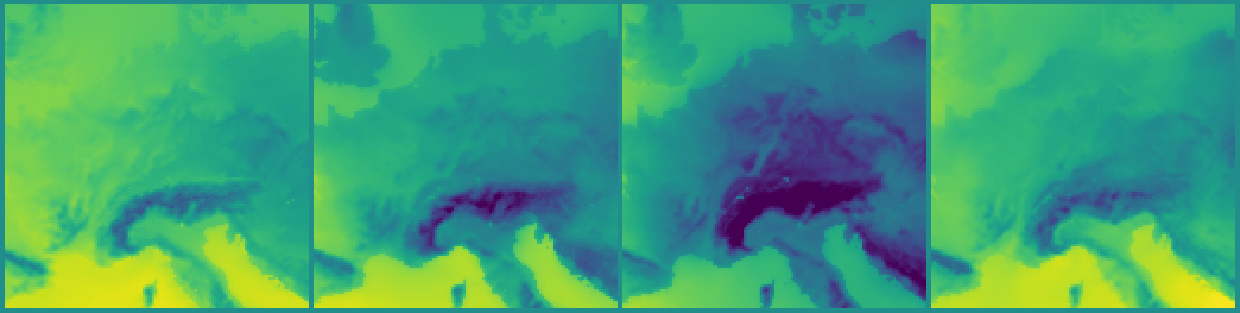} &
	\includegraphics[width=0.23\textwidth,align=c]{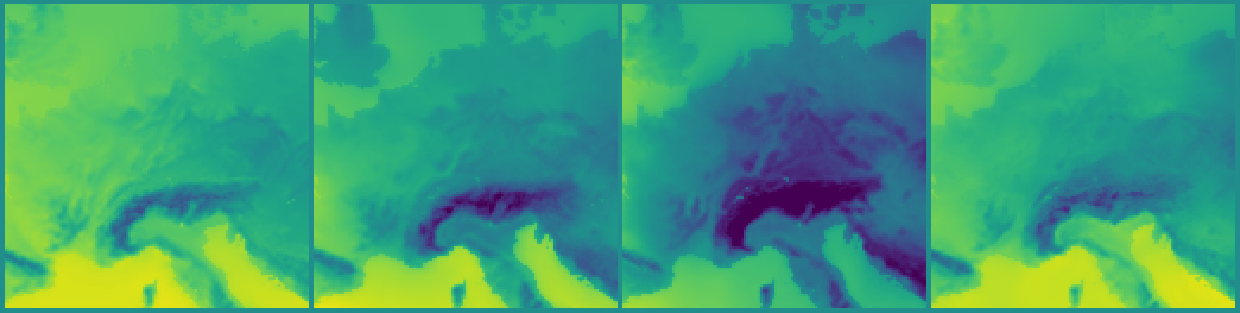} &
	\includegraphics[width=0.23\textwidth,align=c]{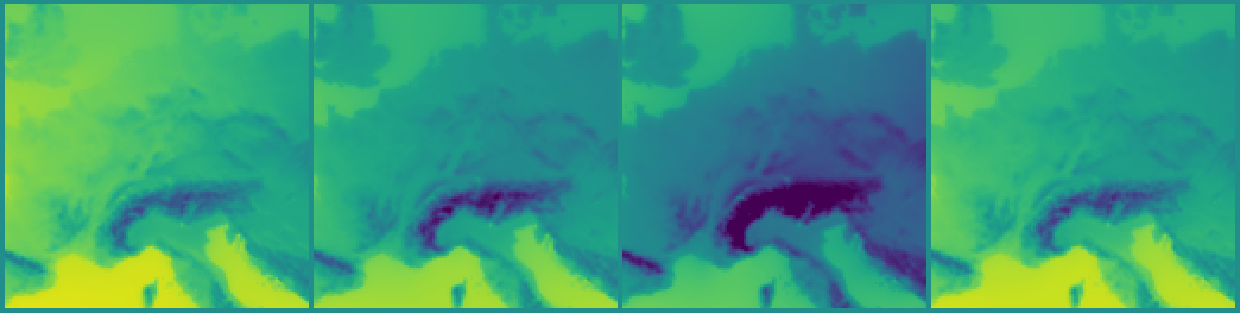} &
	\includegraphics[width=0.23\textwidth,align=c]{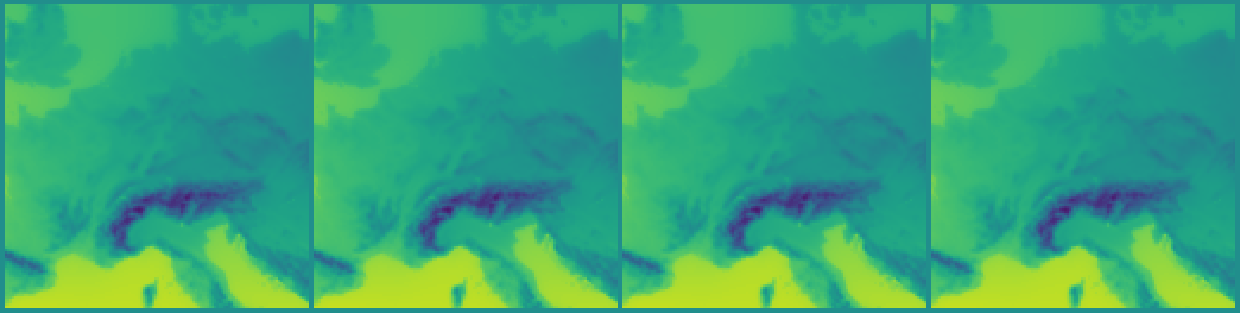} \\
	\rotatebox[origin=c]{90}{\small{DPA}}\hspace{4pt} &
	\includegraphics[width=0.23\textwidth,align=c]{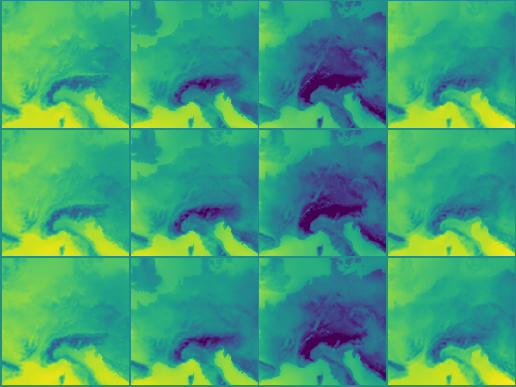} &
	\includegraphics[width=0.23\textwidth,align=c]{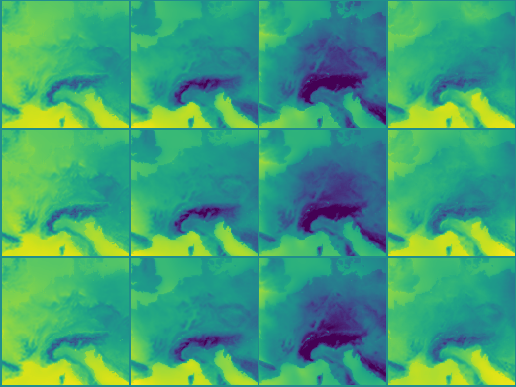} &
	\includegraphics[width=0.23\textwidth,align=c]{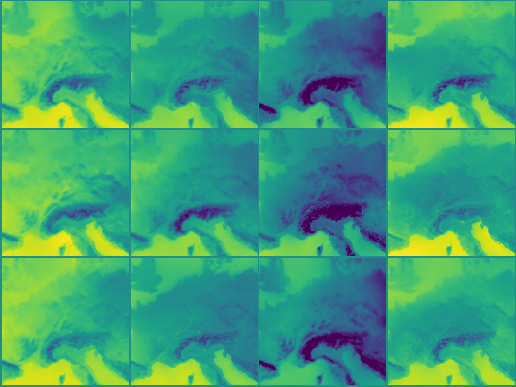} &
	\includegraphics[width=0.23\textwidth,align=c]{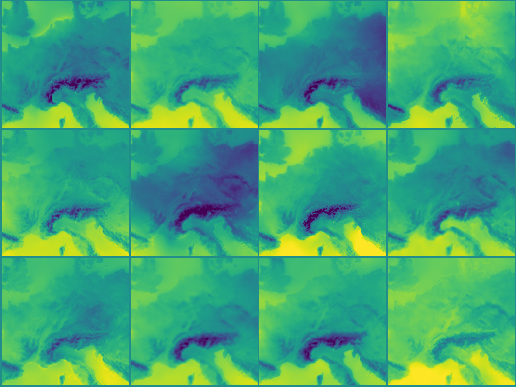} 
\end{tabular}
\caption{Reconstructions for \textsc{r-temp}.}\label{fig:rcmt_recon}
\end{figure}

\begin{figure}
\centering
\begin{tabular}{c@{}BBBB}
	& \small $k=32$ & \small $k=8$ & \small $k=2$ & \small $k=0$\\
	\rotatebox[origin=c]{90}{\small{true}}\hspace{4pt}\smallskip &
	\includegraphics[width=0.23\textwidth,align=c]{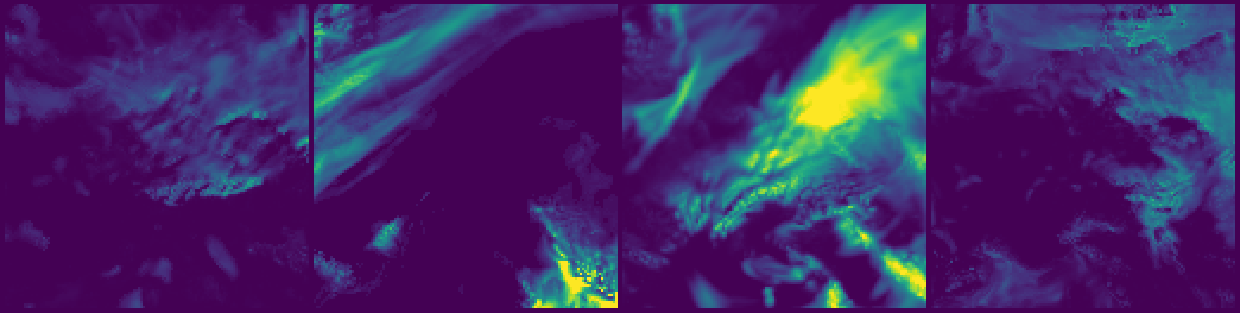} &
	\includegraphics[width=0.23\textwidth,align=c]{fig/rcm_p/true_i0.png} &
	\includegraphics[width=0.23\textwidth,align=c]{fig/rcm_p/true_i0.png} &
	\includegraphics[width=0.23\textwidth,align=c]{fig/rcm_p/true_i0.png} \\
	\rotatebox[origin=c]{90}{\small{PCA}}\hspace{4pt}\smallskip &
	\includegraphics[width=0.23\textwidth,align=c]{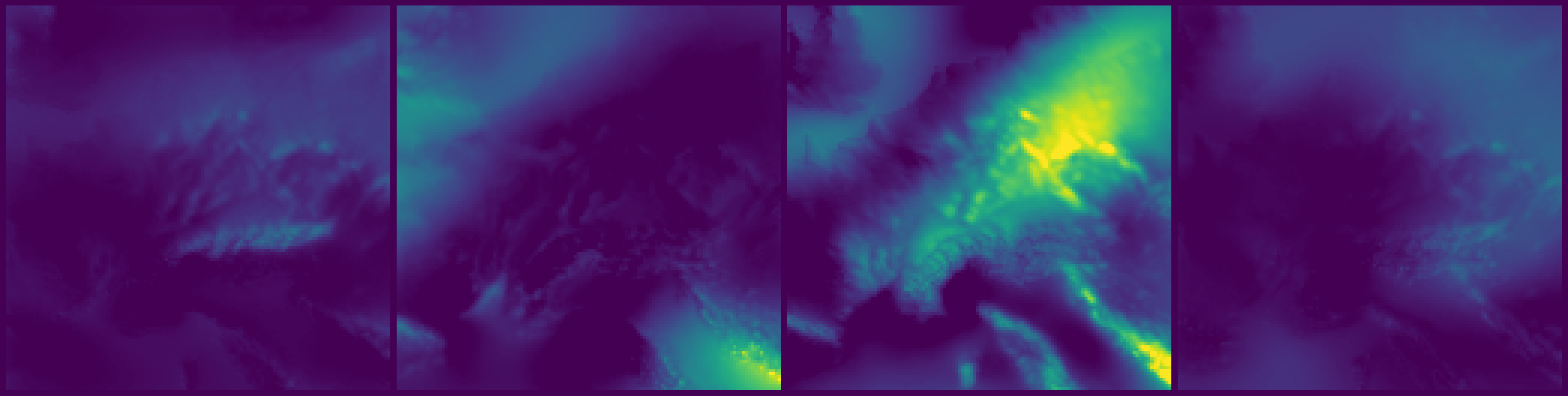} &
	\includegraphics[width=0.23\textwidth,align=c]{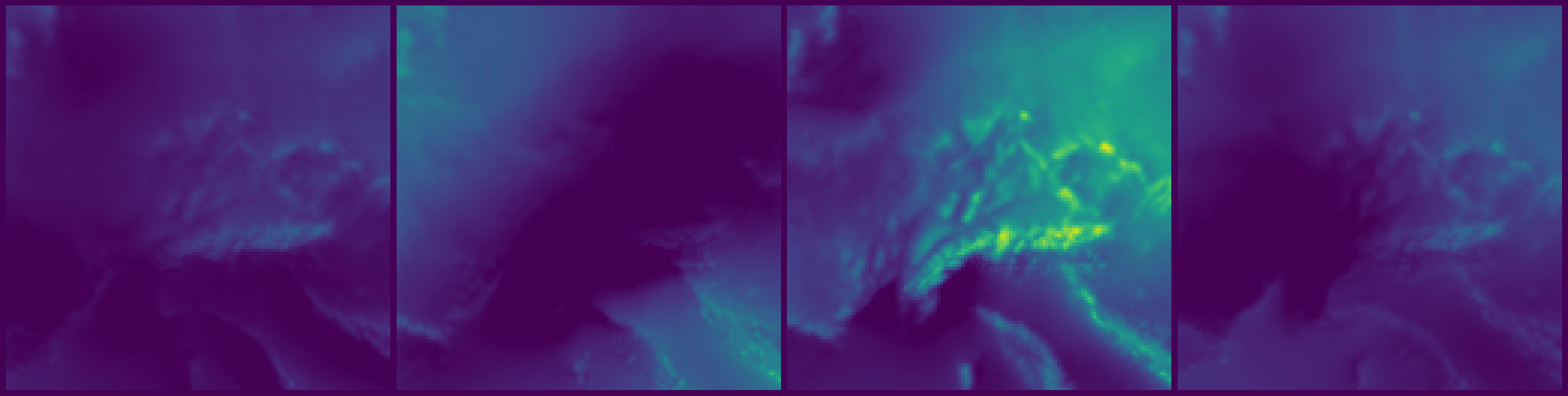} &
	\includegraphics[width=0.23\textwidth,align=c]{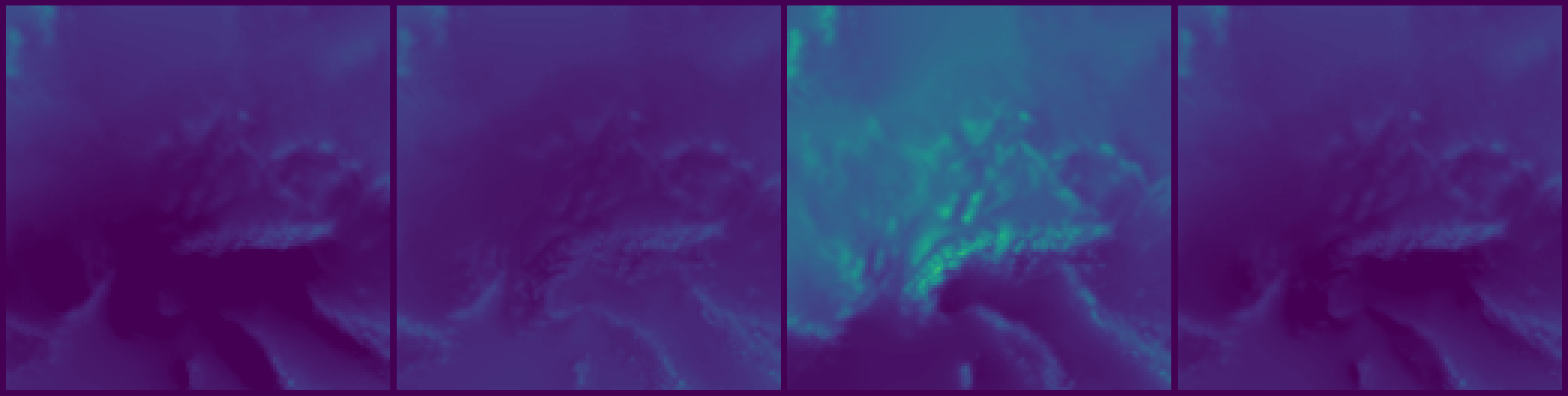} &
	\includegraphics[width=0.23\textwidth,align=c]{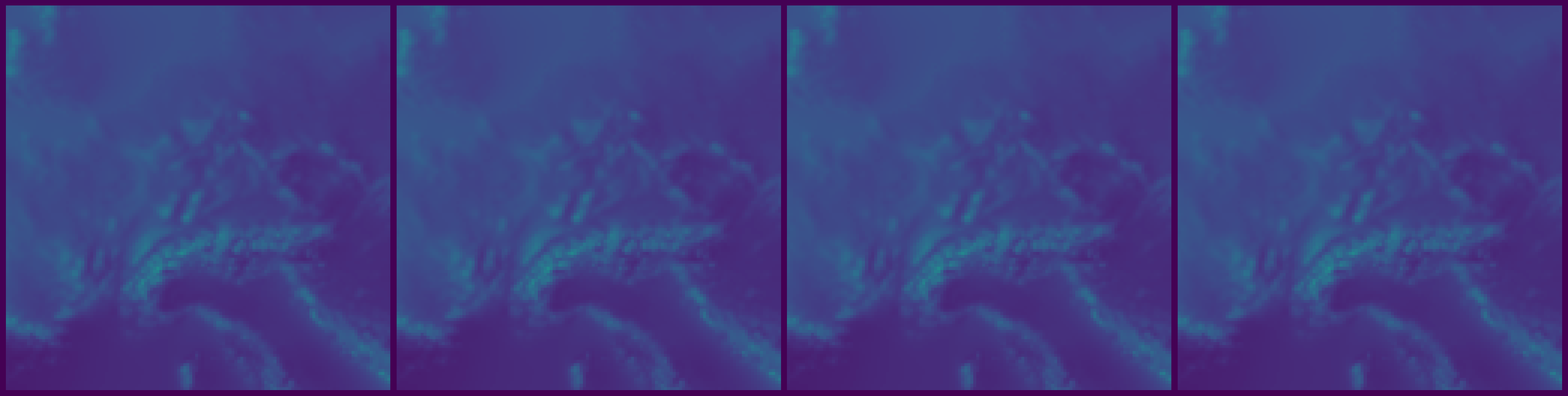}\\
	\rotatebox[origin=c]{90}{\small{AE}}\hspace{4pt}\smallskip &
	\includegraphics[width=0.23\textwidth,align=c]{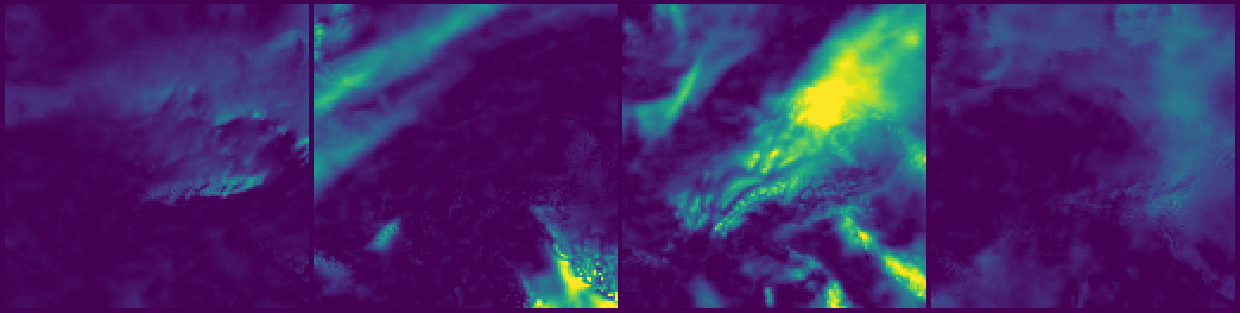} &
	\includegraphics[width=0.23\textwidth,align=c]{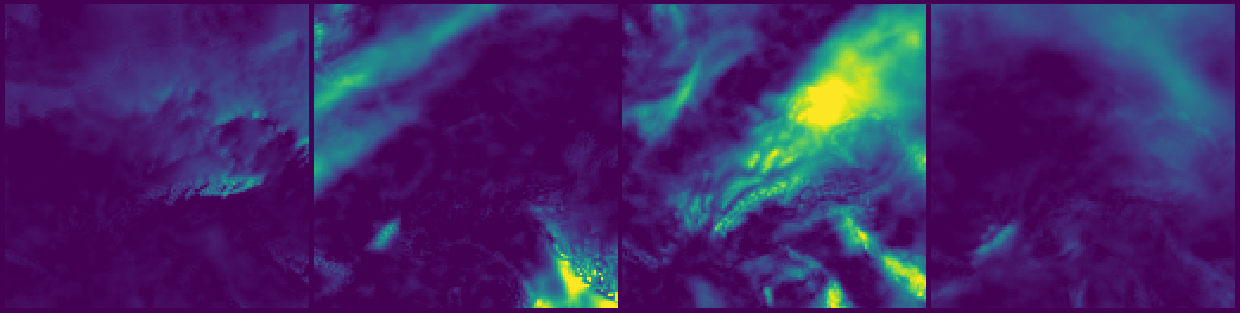} &
	\includegraphics[width=0.23\textwidth,align=c]{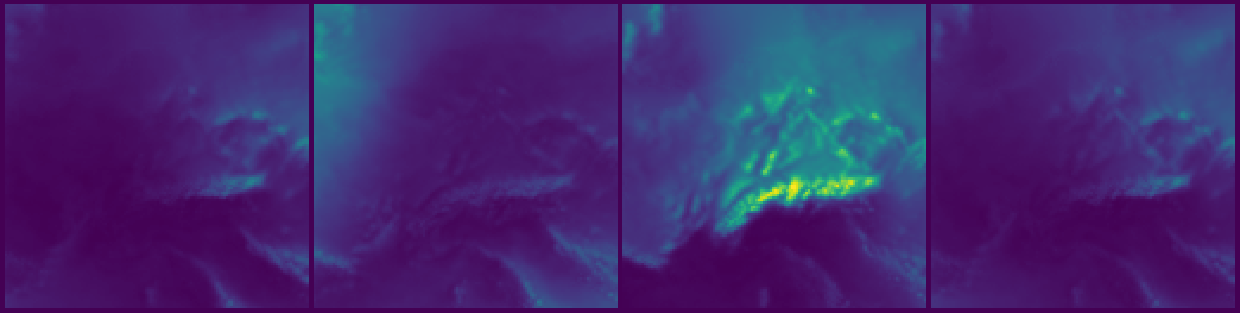} &
	\includegraphics[width=0.23\textwidth,align=c]{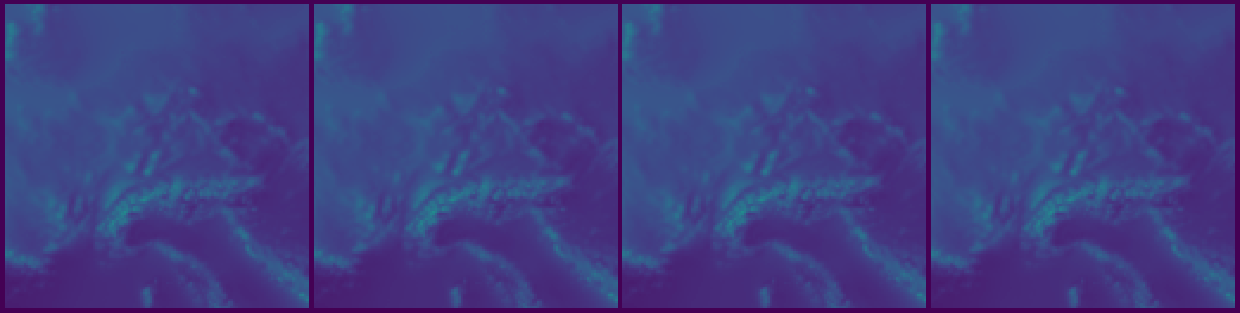} \\
	\rotatebox[origin=c]{90}{\small{DPA}}\hspace{4pt} &
	\includegraphics[width=0.23\textwidth,align=c]{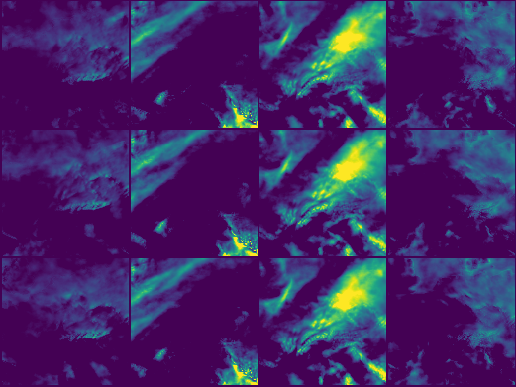} &
	\includegraphics[width=0.23\textwidth,align=c]{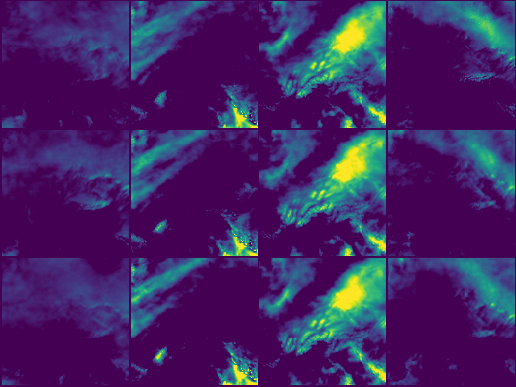} &
	\includegraphics[width=0.23\textwidth,align=c]{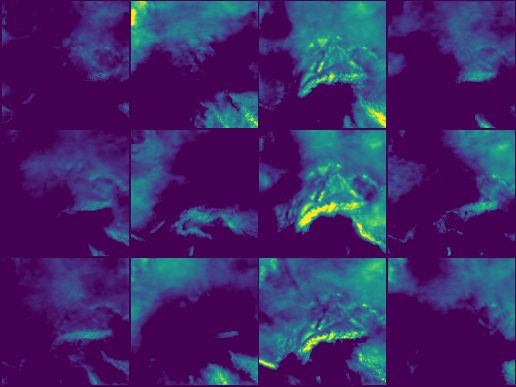} &
	\includegraphics[width=0.23\textwidth,align=c]{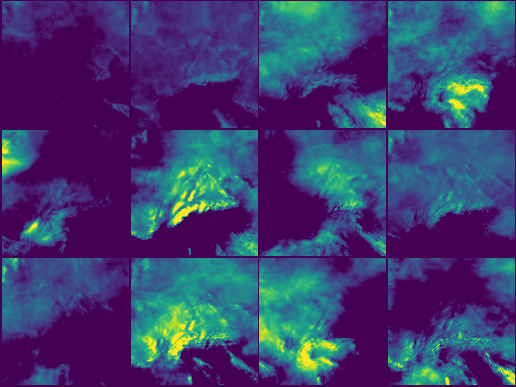} 
\end{tabular}
\caption{Reconstructions for \textsc{r-precip}.}\label{fig:rcmp_recon}
\end{figure}


Figure~\ref{fig:metrics_recon} presents several evaluation metrics as functions of the latent dimension $k$. These include 
\begin{enumerate}[(i)] 
\item two metrics for conditional reconstruction \eqref{eq:dist_recon}:
\begin{enumerate}[(a)] \item conditional energy loss $\bbE[\|X-d(e(X),\varepsilon)\|-\|d(e(X),\varepsilon)-d(e(X),\varepsilon')\|/2]$ and \item mean squared reconstruction error $\bbE[\|X-d(e(X),\varepsilon)\|^2]$, and
\end{enumerate}  
\item two metrics for the unconditional distribution $\hat{P}$ of the reconstructions to evaluate how well it matches the original data distribution $P^*$, i.e.\ \eqref{eq:eq_marg_dist}:
\begin{enumerate}[(a)] \item the energy distance between these two multivariate distributions, $D(P^*,\hat{P})$, as defined in \eqref{eq:energy_distance} with $\beta=1$, and \item the average Wasserstein 1-distance of the marginal distributions at a random pixel/location, i.e.\ $\sum_{i=1}^pW_1(P_i, \hat{P}_i)/p$, where $P_i$ and $\hat{P}_i$ denote the marginal distributions of the $i$-th component of $P^*$ and $\hat{P}$, respectively.
\end{enumerate}
\end{enumerate}

We highlight three observations.

First, as expected, the conditional metrics improve monotonically with increasing latent dimension $k$. This aligns with the visual results, which show more variability among DPA samples (conditional on the same input) when $k$ is small. The conditional metrics typically reach a plateau once $k$ exceeds the intrinsic dimension of the data. For example, the conditional metrics for \textsc{disk} decrease until $k=6$, which matches exactly the intrinsic dimension of this data set (due to the 6 generative factors).

Second, AE and PCA use the mean squared reconstruction error as their loss whereas DPA aims for the full distribution. Ideally, when the model class is expressive enough, AE and DPA should achieve the same minimal conditional MSE, while PCA may incur a small loss due to its linearity restriction. As shown in the second column, this is indeed the case for the three climate data sets with large enough sample sizes. For the remaining three data sets, whose sample sizes are smaller, DPA exhibits smaller MSEs than AE, especially on \textsc{disk}, possibly because DPA suffers less from overfitting than AE. These findings indicate that DPA performs at least as well as AE in terms of mean reconstruction (the objective of AE), even though its objective is distributional reconstruction.

Third, DPA outperforms AE and PCA substantially on the unconditional metrics, since only DPA is guaranteed to produce reconstructions that are identically distributed to the original data. This guarantee holds for all $k$, and the unconditional metrics indeed stay almost constant in $k$.

To further investigate how well our reconstructed distributions match the original ones, in Figure~\ref{fig:qq_cdf}, we show Q--Q plots and empirical cdfs of the true or estimated marginal distributions at a random location. We observe that DPA captures the distribution the best regardless of the choice of $k$. This is especially true for the tail behaviour. In contrast, the tails of the distribution get shifted towards the mean with PCA and AE,  especially for smaller $k$'s and as expected due to their reconstructions trying to minimise the mean squared error.

\begin{figure}
\centering
\begin{tabular}{@{}c@{}c@{}c@{}c@{}c@{}}
	& \scriptsize ~conditional energy loss & \scriptsize ~~conditional MSE & \scriptsize ~uncond energy distance & \scriptsize ~~marginal W distance \\
	\rotatebox[origin=c]{90}{\small{\textsc{mnist}}}\hspace{4pt} &
	\includegraphics[width=0.23\textwidth,align=c]{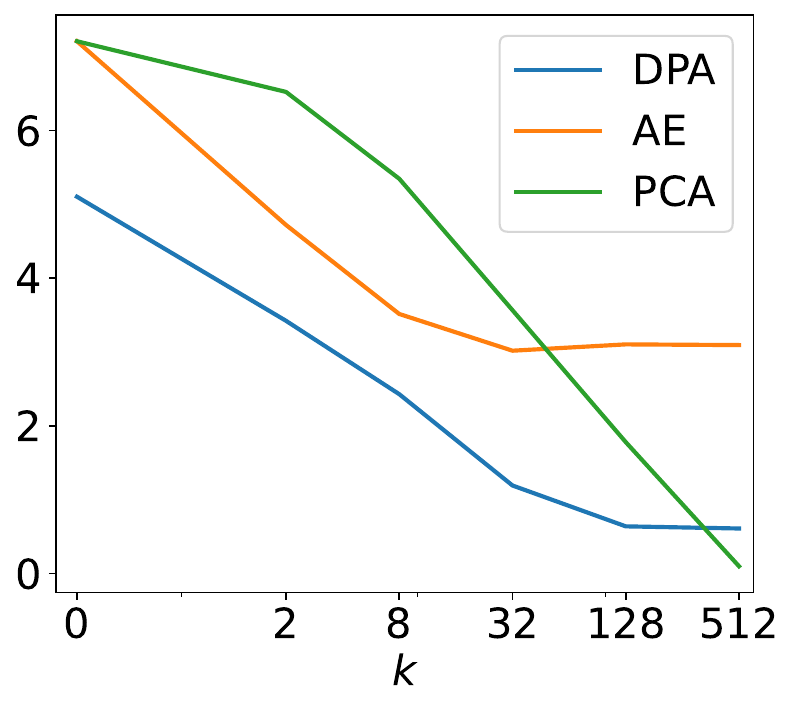} &
	\includegraphics[width=0.23\textwidth,align=c]{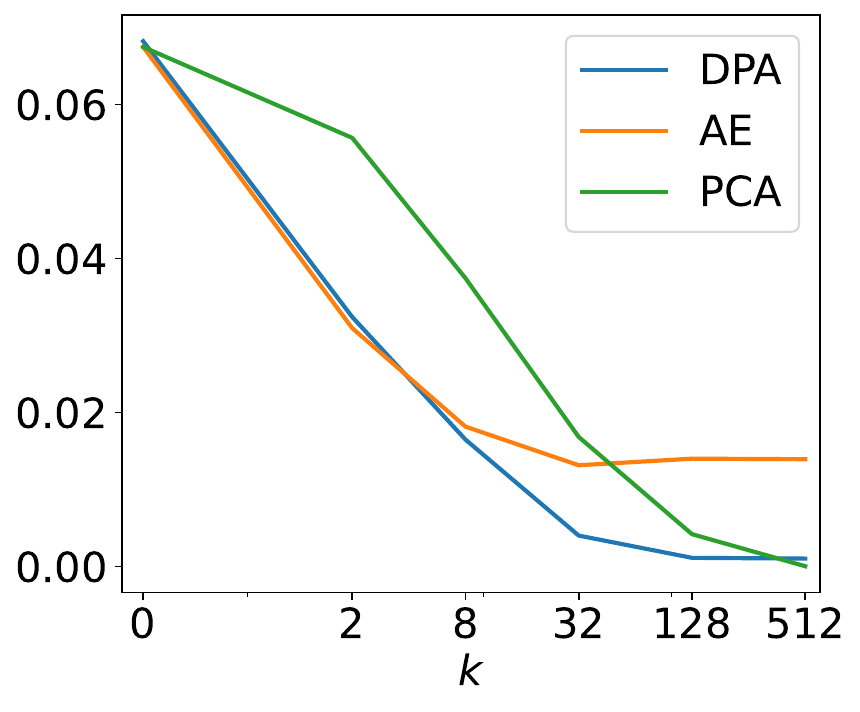} &
	\includegraphics[width=0.23\textwidth,align=c]{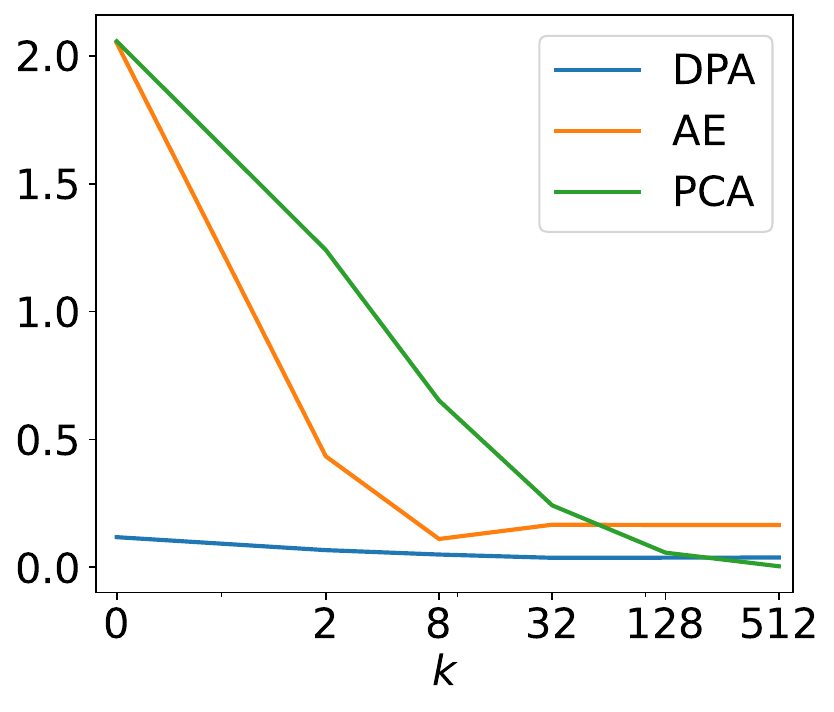} &
	\includegraphics[width=0.23\textwidth,align=c]{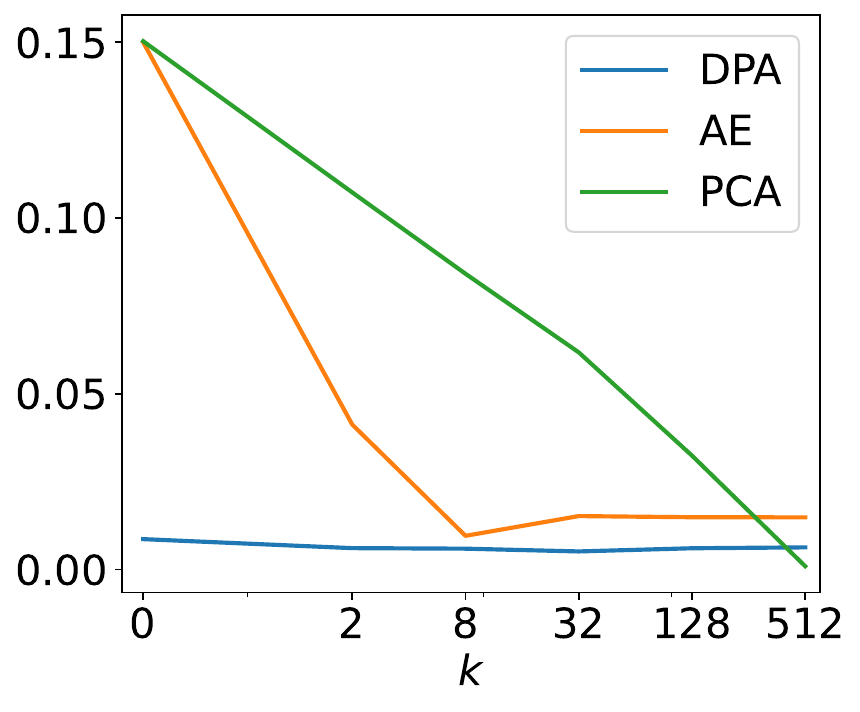} \\
	\rotatebox[origin=c]{90}{\small{\textsc{disk}}}\hspace{4pt} &
	\includegraphics[width=0.23\textwidth,align=c]{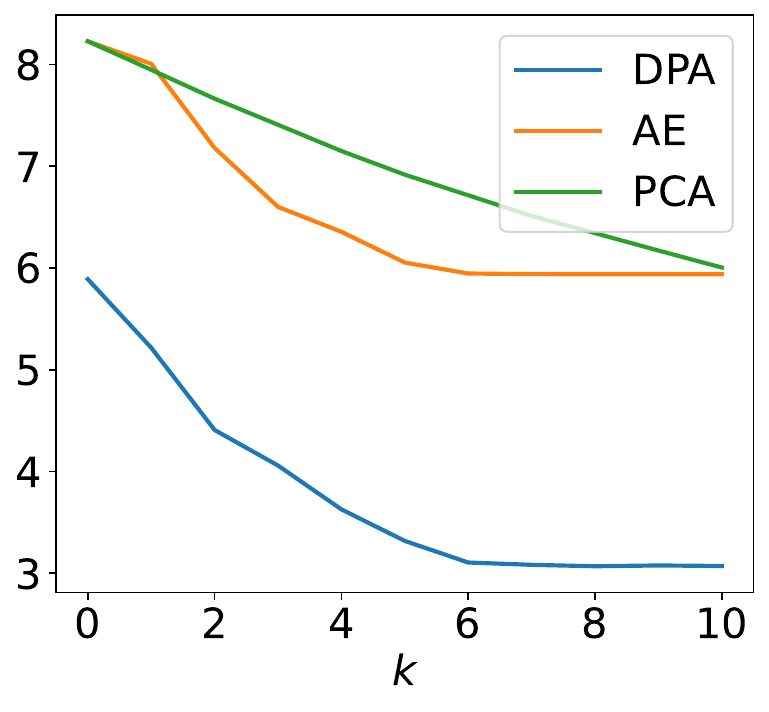} &
	\includegraphics[width=0.23\textwidth,align=c]{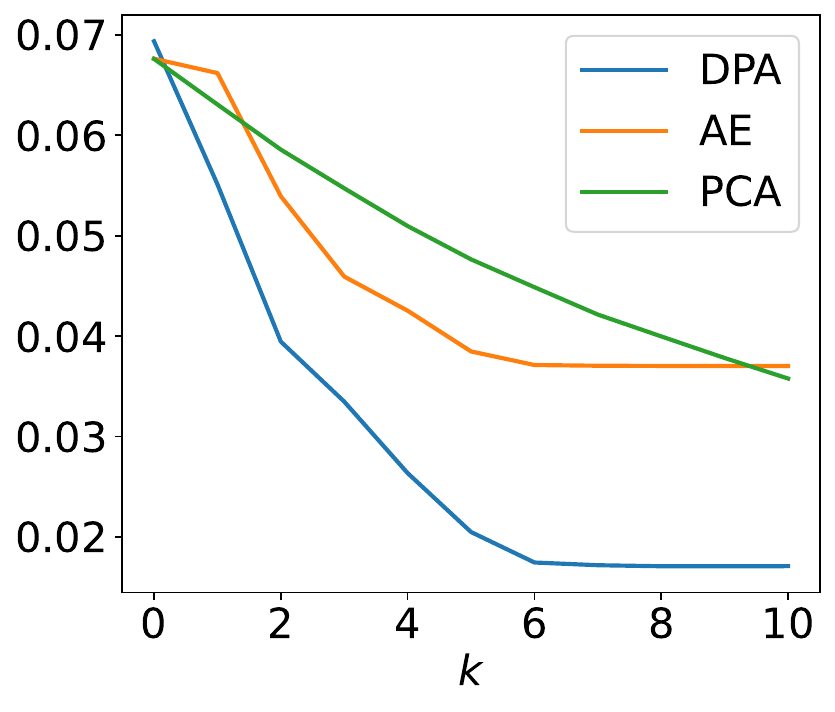} &
	\includegraphics[width=0.23\textwidth,align=c]{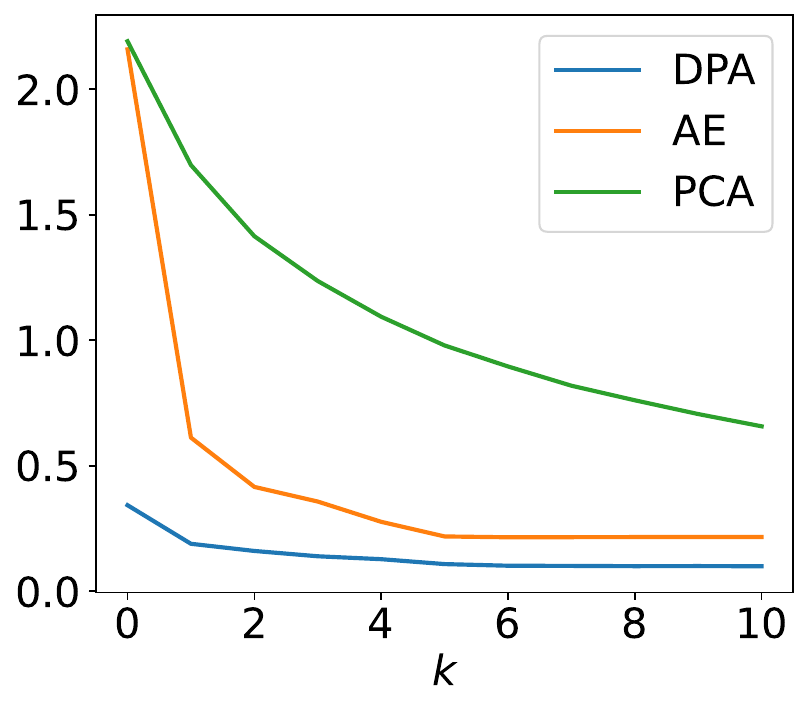} &
	\includegraphics[width=0.23\textwidth,align=c]{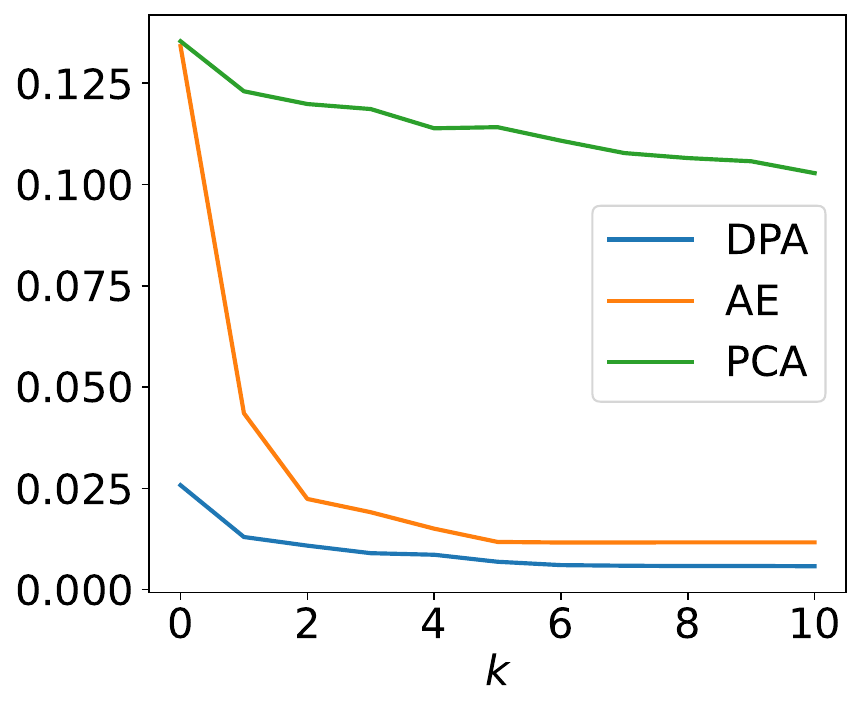} \\
	\rotatebox[origin=c]{90}{\small{\textsc{r-temp}}}\hspace{4pt} &
	\includegraphics[width=0.23\textwidth,align=c]{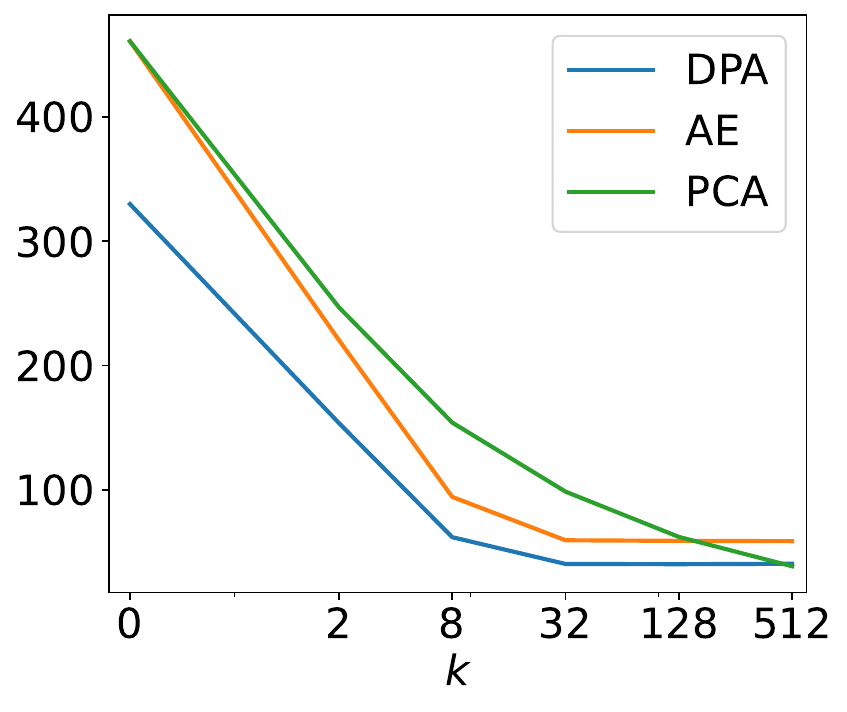} &
	\includegraphics[width=0.23\textwidth,align=c]{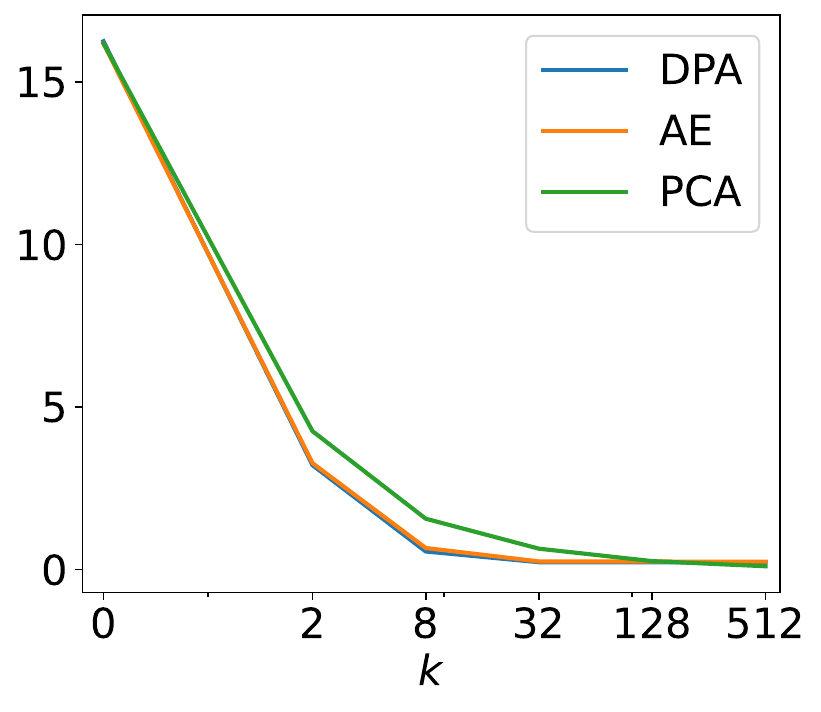} &
	\includegraphics[width=0.23\textwidth,align=c]{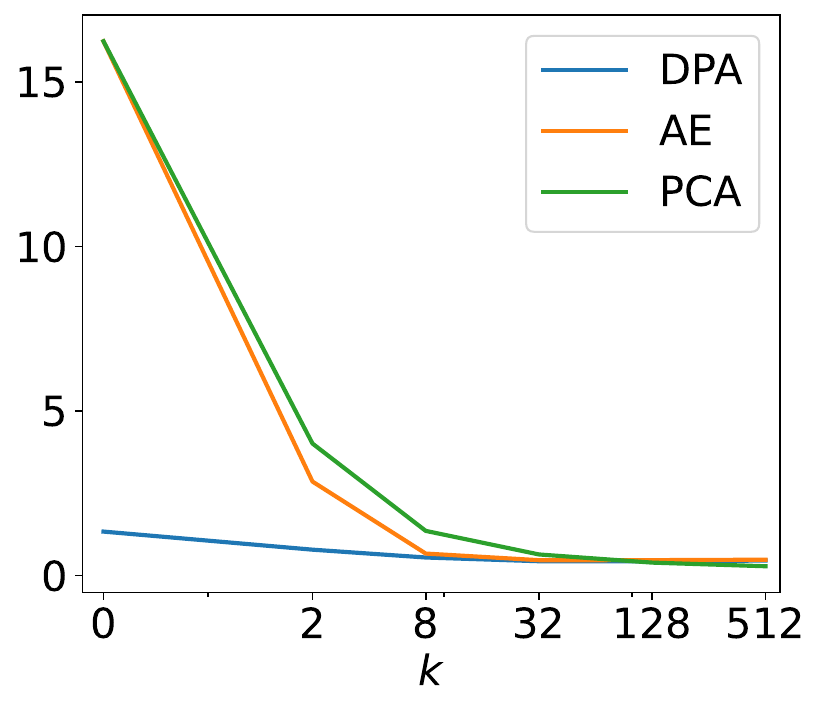} &
	\includegraphics[width=0.23\textwidth,align=c]{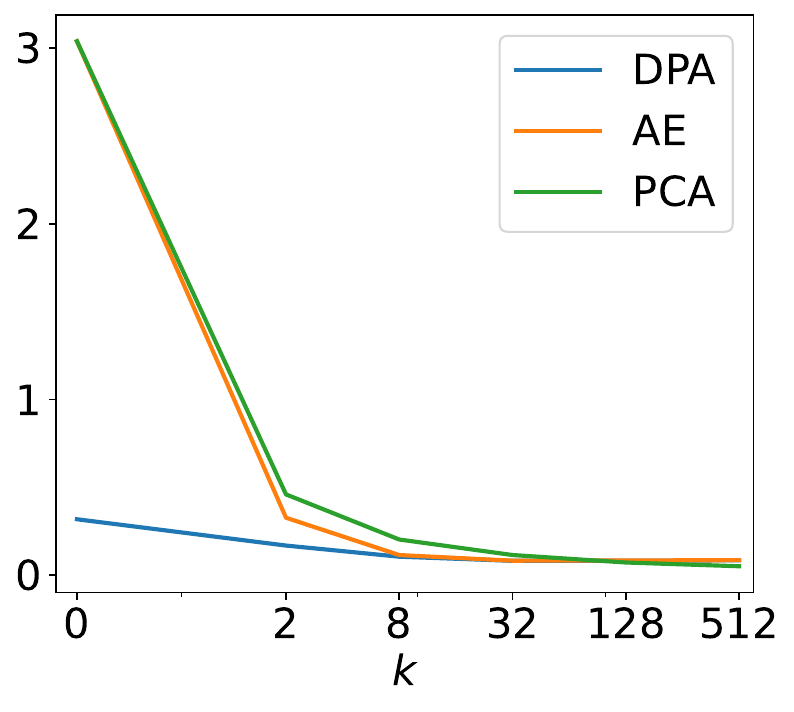} \\
	\rotatebox[origin=c]{90}{\small{\textsc{r-precip}}}\hspace{4pt} &
	\includegraphics[width=0.23\textwidth,align=c]{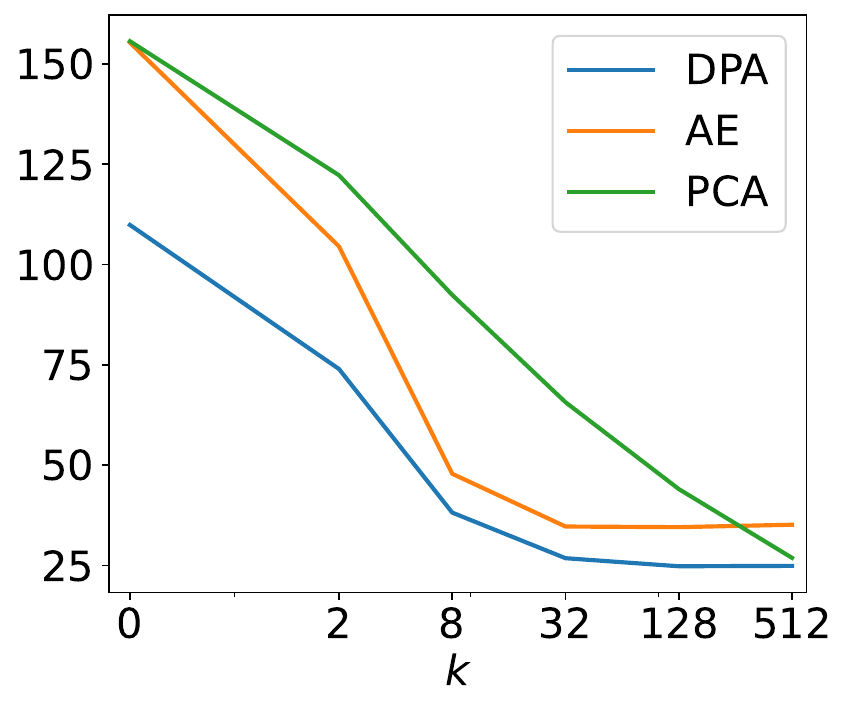} &
	\includegraphics[width=0.23\textwidth,align=c]{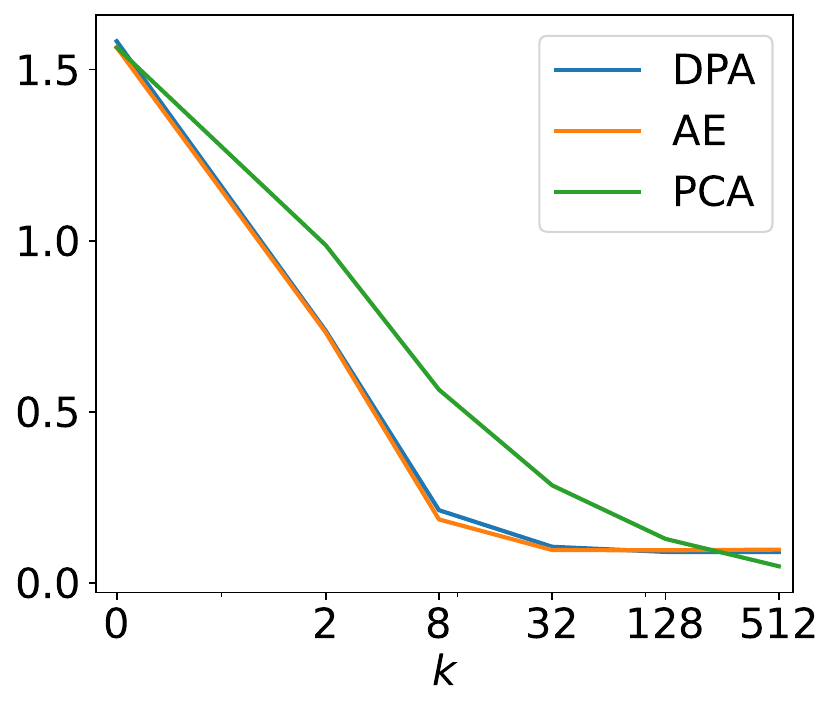} &
	\includegraphics[width=0.23\textwidth,align=c]{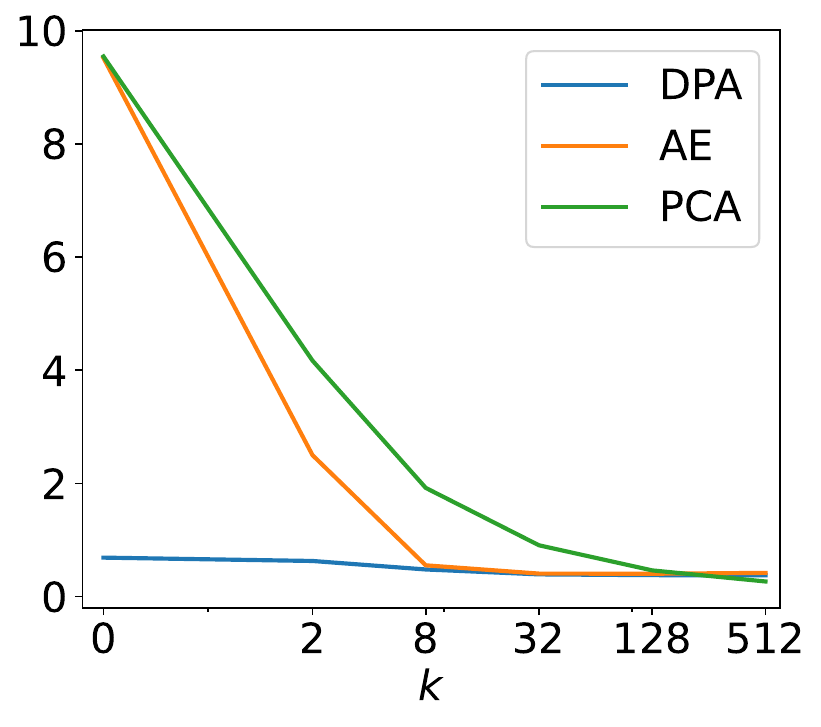} &
	\includegraphics[width=0.23\textwidth,align=c]{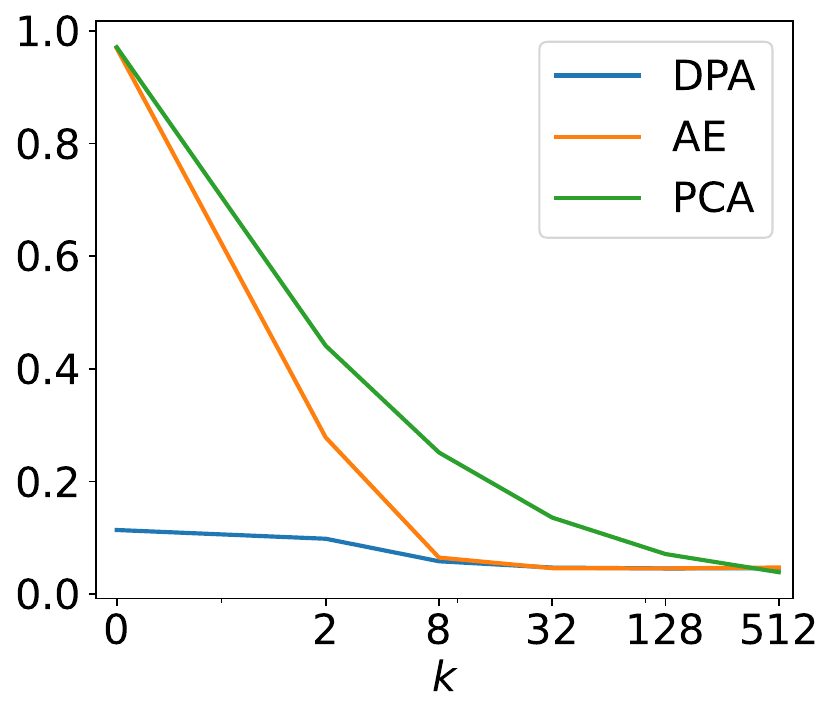} \\
	\rotatebox[origin=c]{90}{\small{\textsc{g-precip}}}\hspace{4pt} &
	\includegraphics[width=0.23\textwidth,align=c]{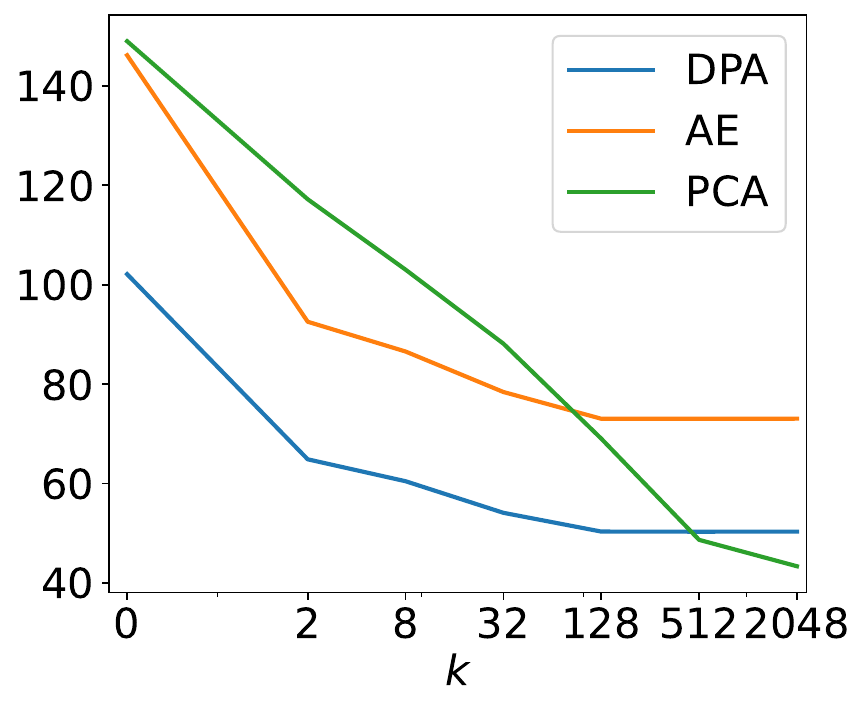} &
	\includegraphics[width=0.23\textwidth,align=c]{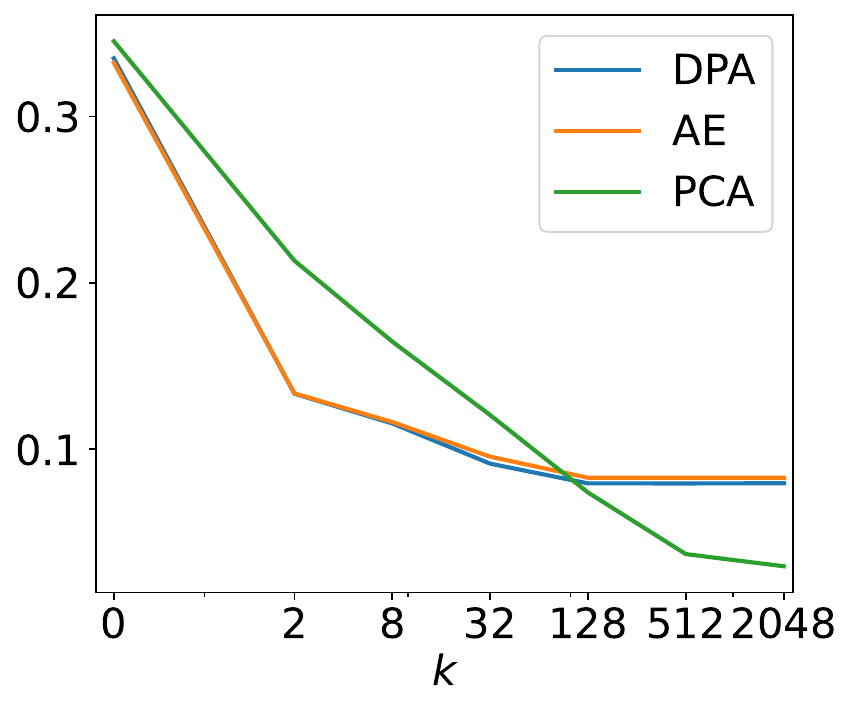} &
	\includegraphics[width=0.23\textwidth,align=c]{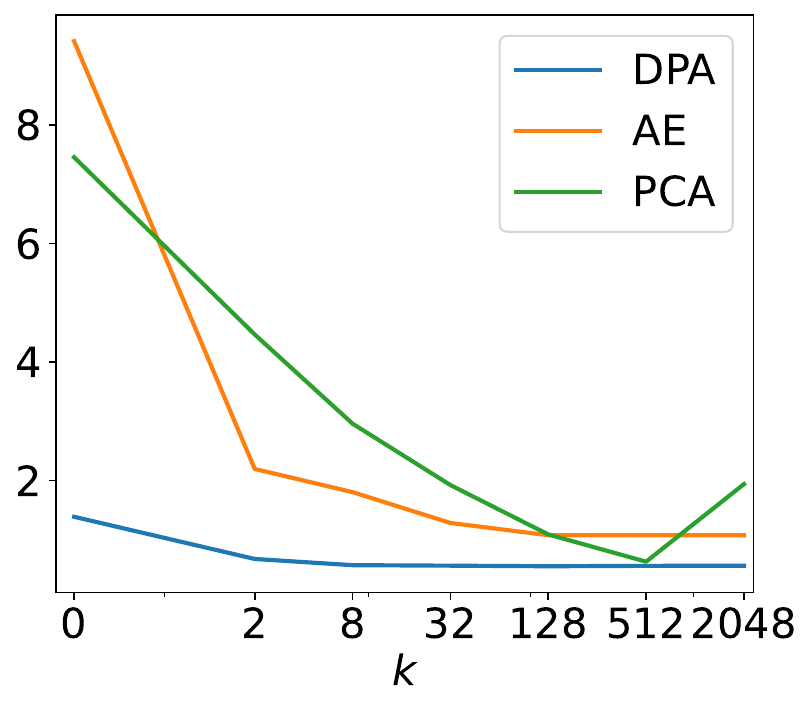} &
	\includegraphics[width=0.23\textwidth,align=c]{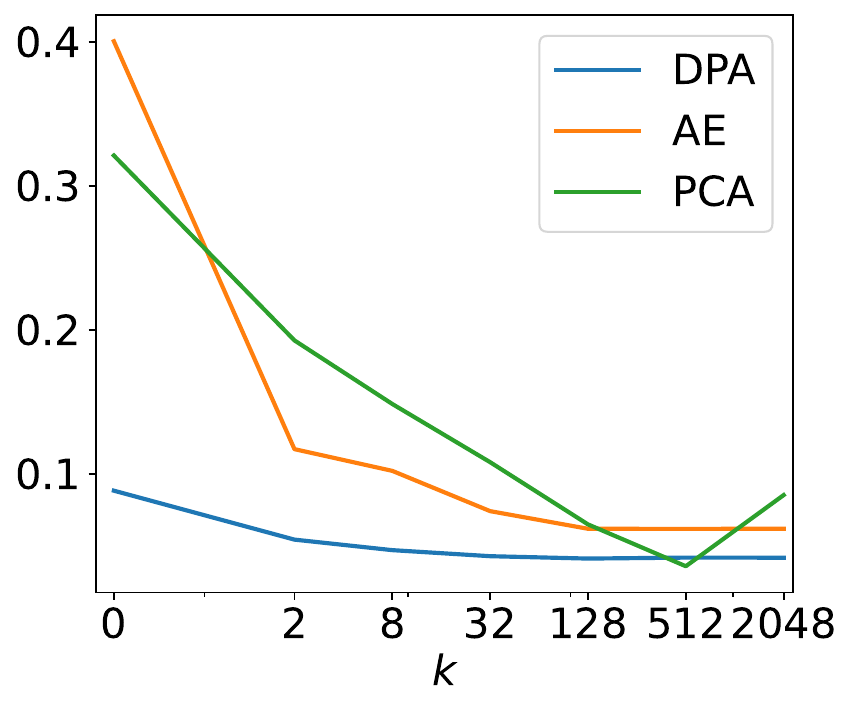} \\
	\rotatebox[origin=c]{90}{\small{\textsc{sc-8}}}\hspace{4pt} &
	\includegraphics[width=0.23\textwidth,align=c]{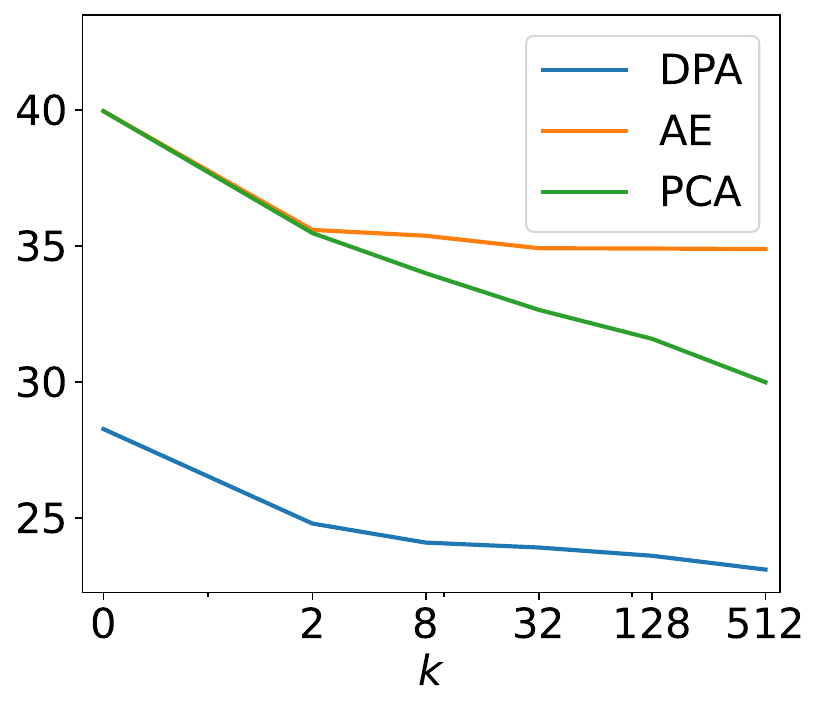} &
	\includegraphics[width=0.23\textwidth,align=c]{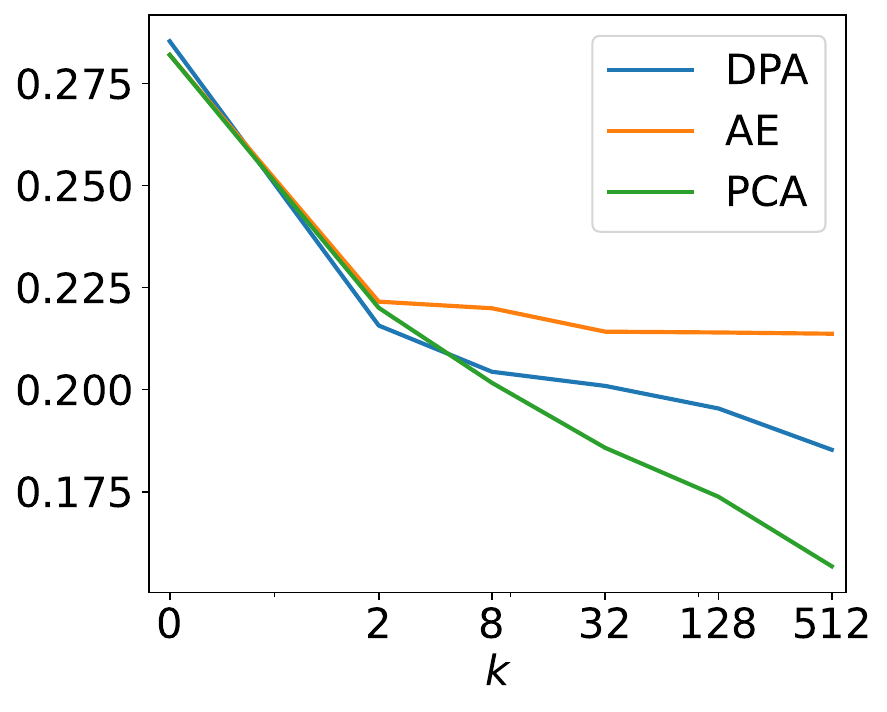} &
	\includegraphics[width=0.23\textwidth,align=c]{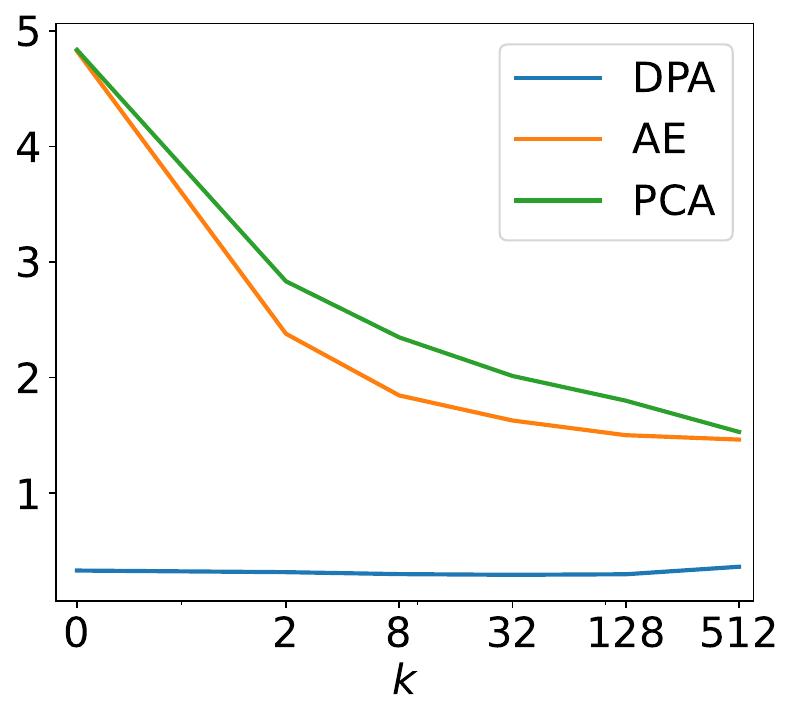} &
	\includegraphics[width=0.23\textwidth,align=c]{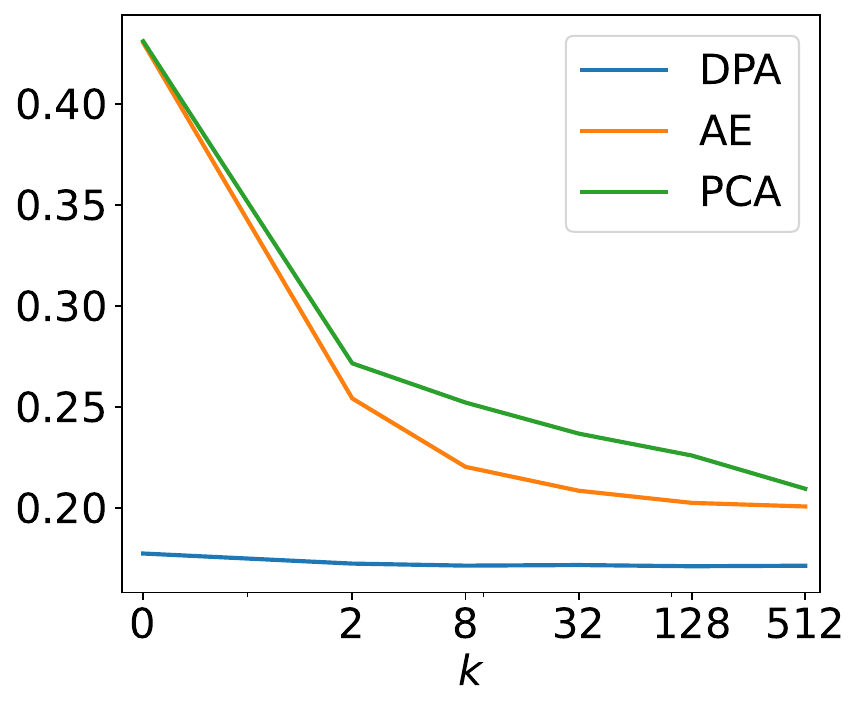} 
\end{tabular}
\caption{Metrics for reconstructions as functions of the latent dimension $k$.}\label{fig:metrics_recon}
\end{figure}

\begin{figure}
\centering
\begin{tabular}{@{}c@{}c@{}c@{}c@{}c@{}c@{}c@{}}
	& \multicolumn{3}{c}{Q--Q plots} & \multicolumn{3}{c}{empirical cdfs}\\
	\rotatebox[origin=c]{90}{\small\textsc{r-temp}}
	& \includegraphics[width=0.16\textwidth,align=c]{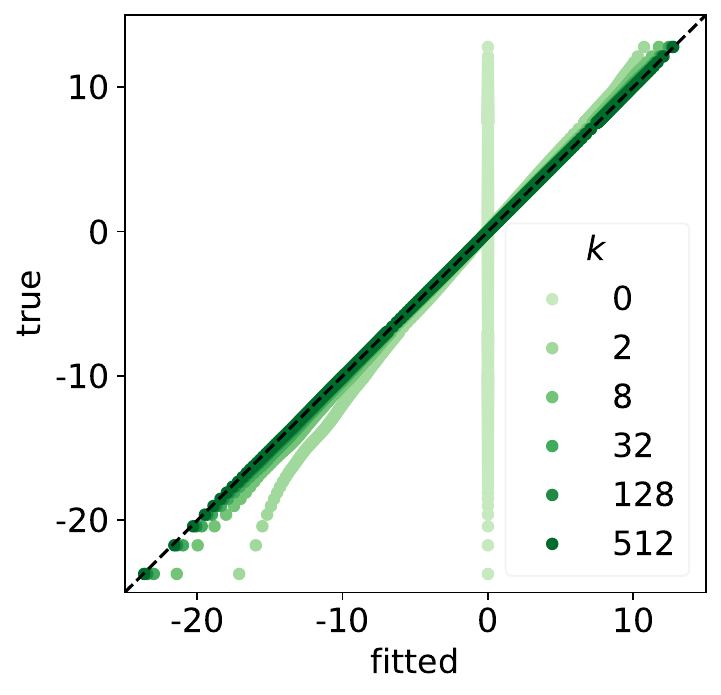} &
	\includegraphics[width=0.16\textwidth,align=c]{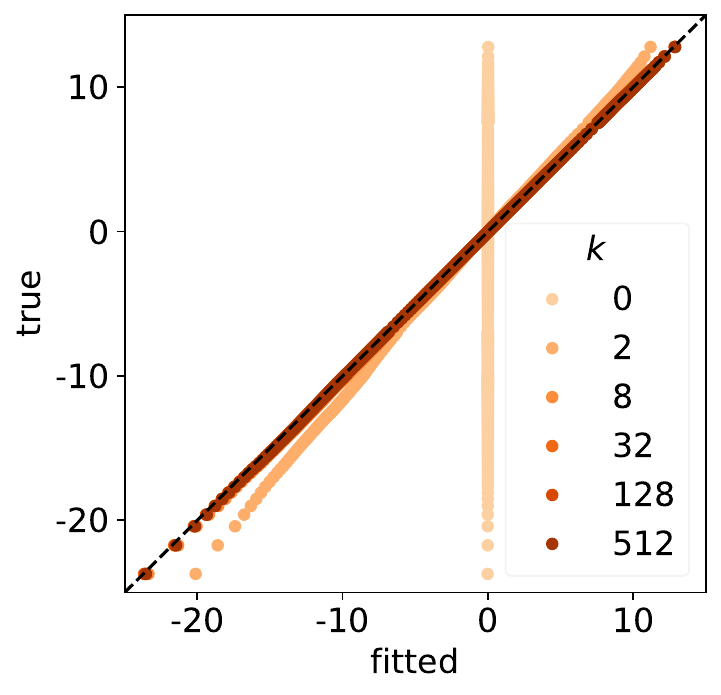} &
	\includegraphics[width=0.16\textwidth,align=c]{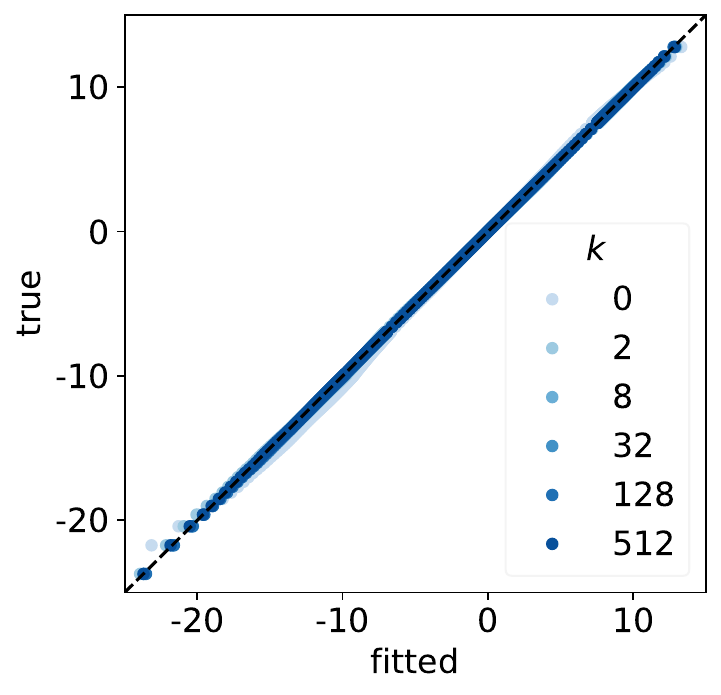} &
	\includegraphics[width=0.16\textwidth,align=c]{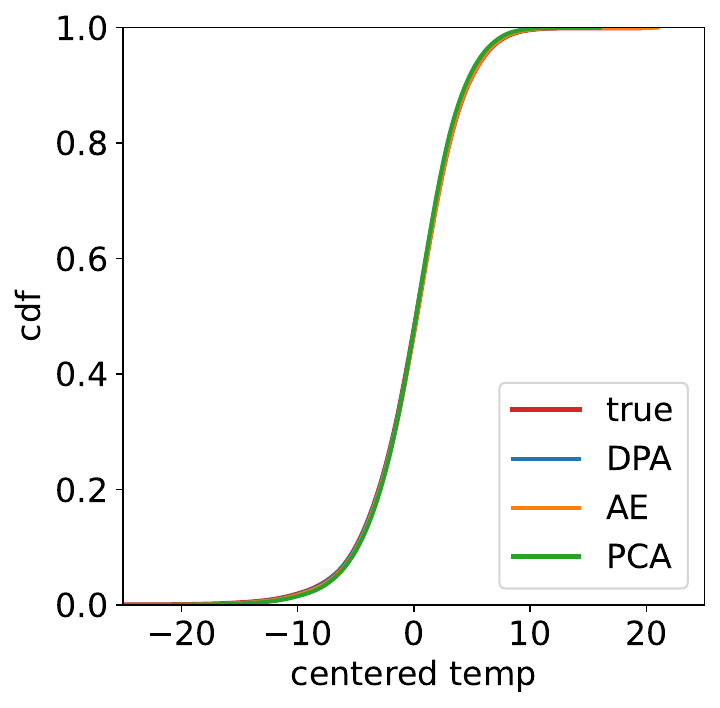} &
	\includegraphics[width=0.16\textwidth,align=c]{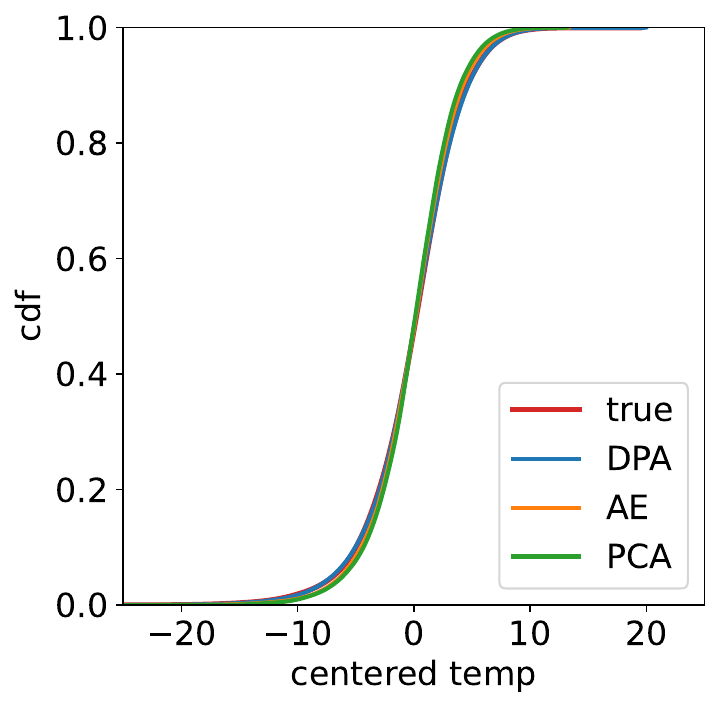} &
	\includegraphics[width=0.16\textwidth,align=c]{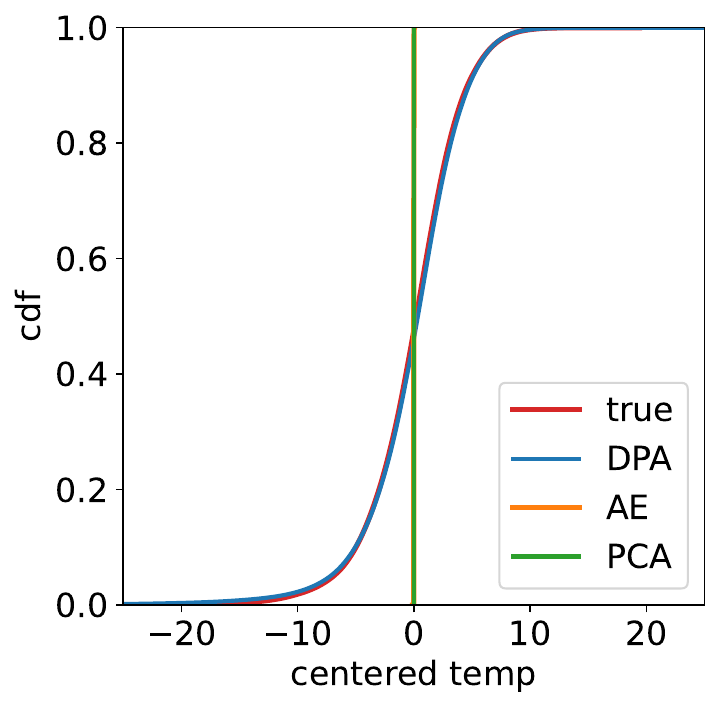} \\
	\rotatebox[origin=c]{90}{\small\textsc{r-precip}}
	& \includegraphics[width=0.16\textwidth,align=c]{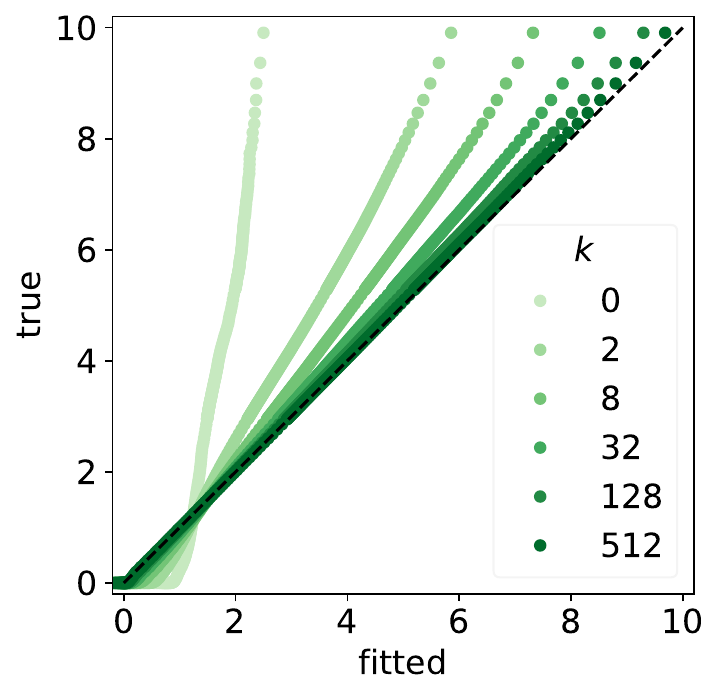} &
	\includegraphics[width=0.16\textwidth,align=c]{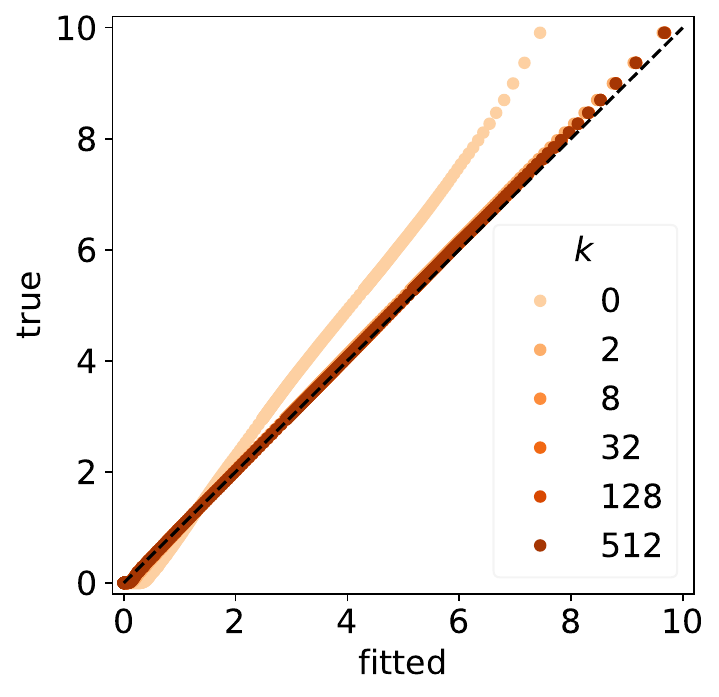} &
	\includegraphics[width=0.16\textwidth,align=c]{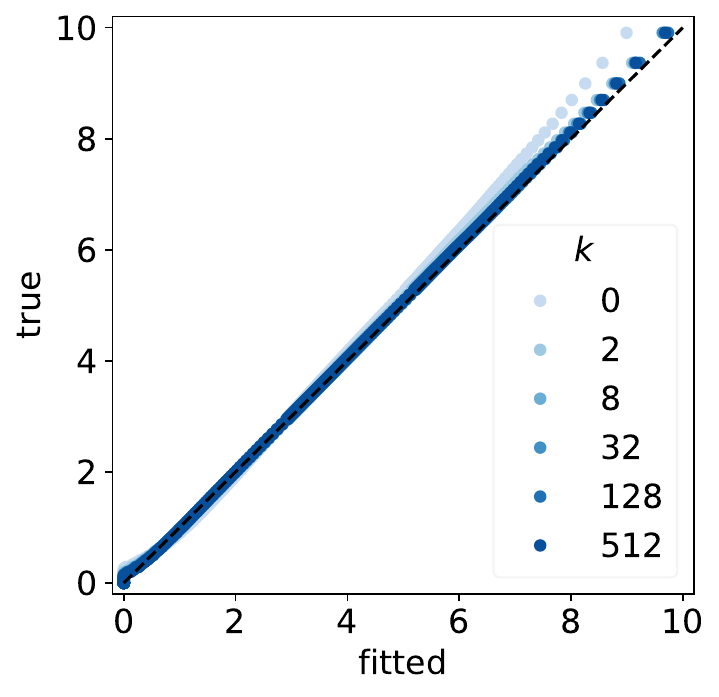} &
	\includegraphics[width=0.16\textwidth,align=c]{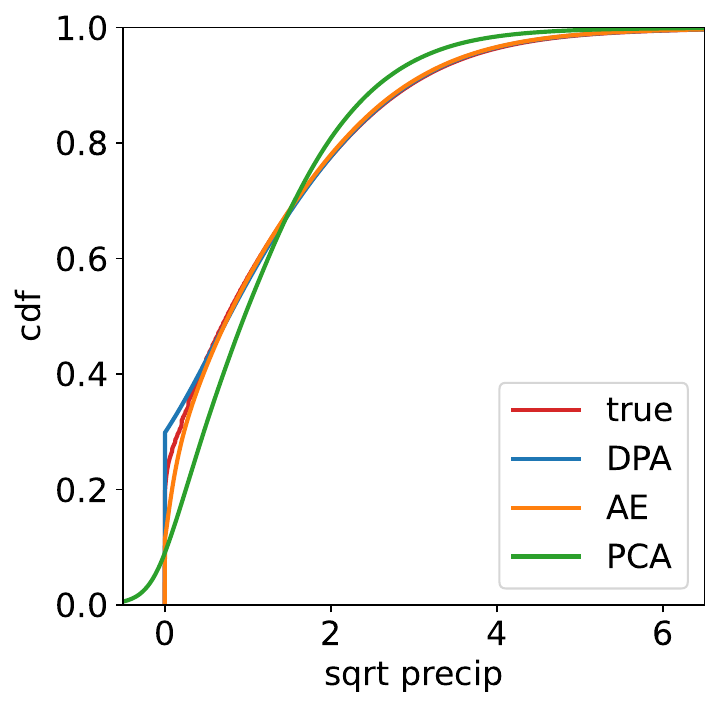} &
	\includegraphics[width=0.16\textwidth,align=c]{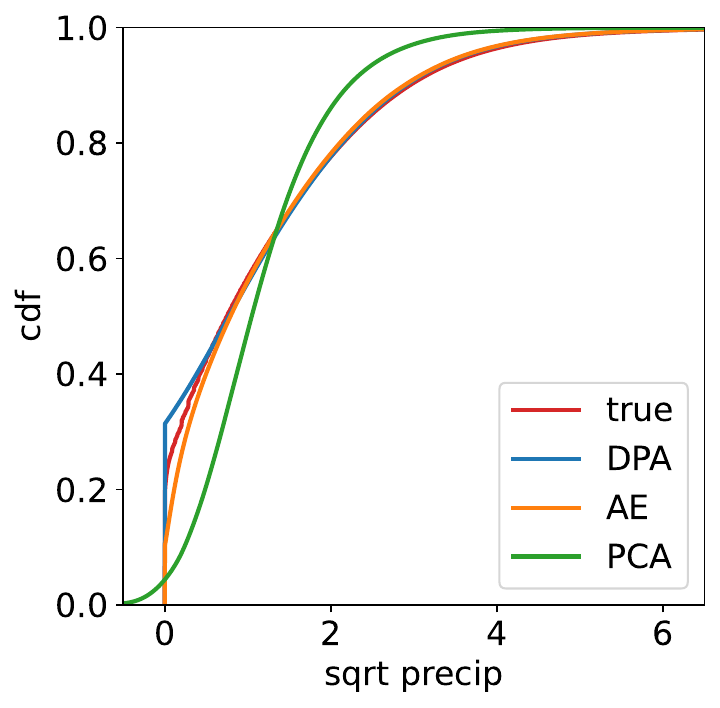} &
	\includegraphics[width=0.16\textwidth,align=c]{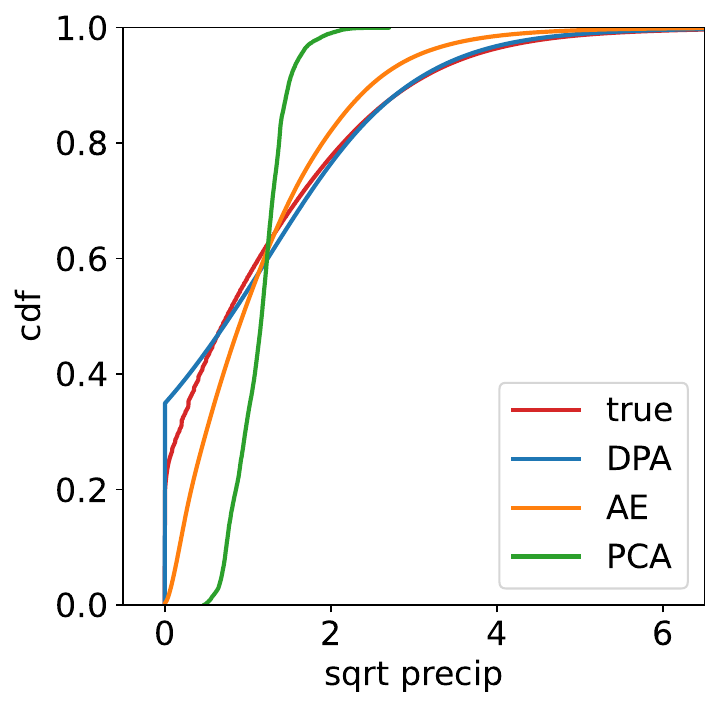} \\
	\rotatebox[origin=c]{90}{\small\textsc{sc-8}}
	& \includegraphics[width=0.16\textwidth,align=c]{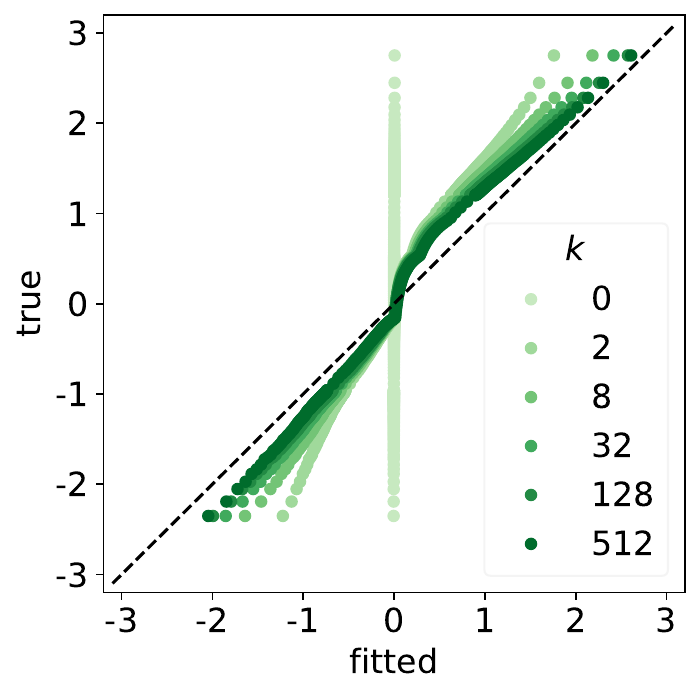} &
	\includegraphics[width=0.16\textwidth,align=c]{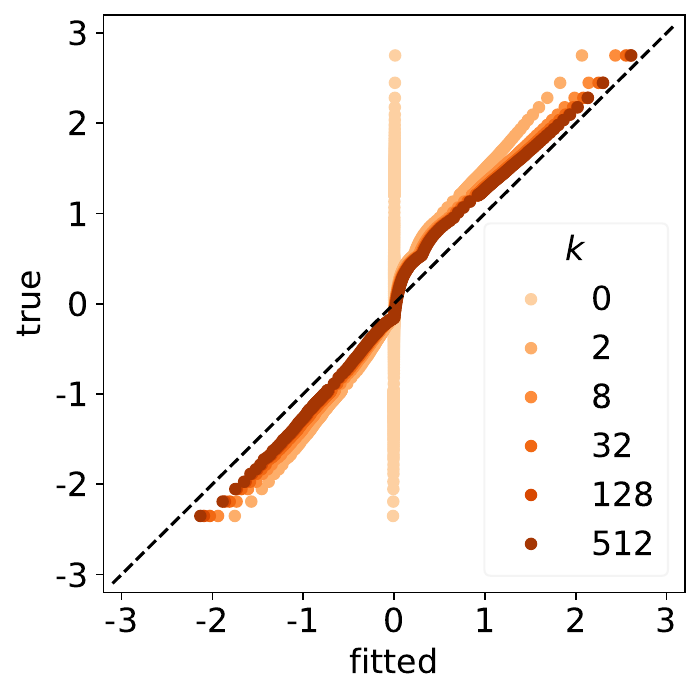} &
	\includegraphics[width=0.16\textwidth,align=c]{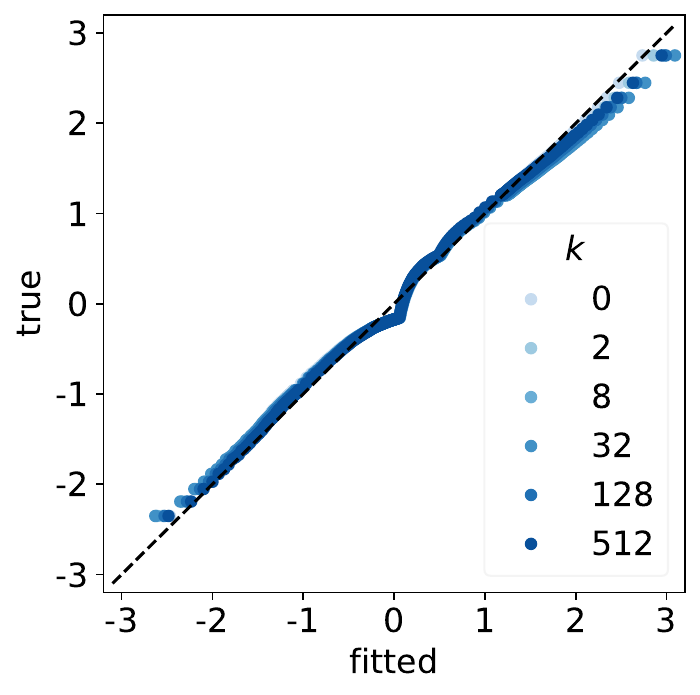} &
	\includegraphics[width=0.16\textwidth,align=c]{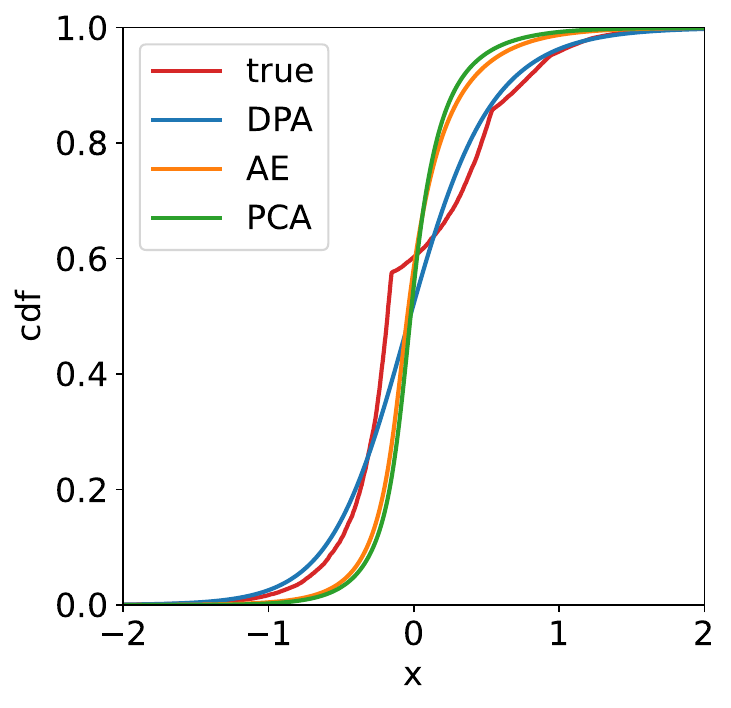} &
	\includegraphics[width=0.16\textwidth,align=c]{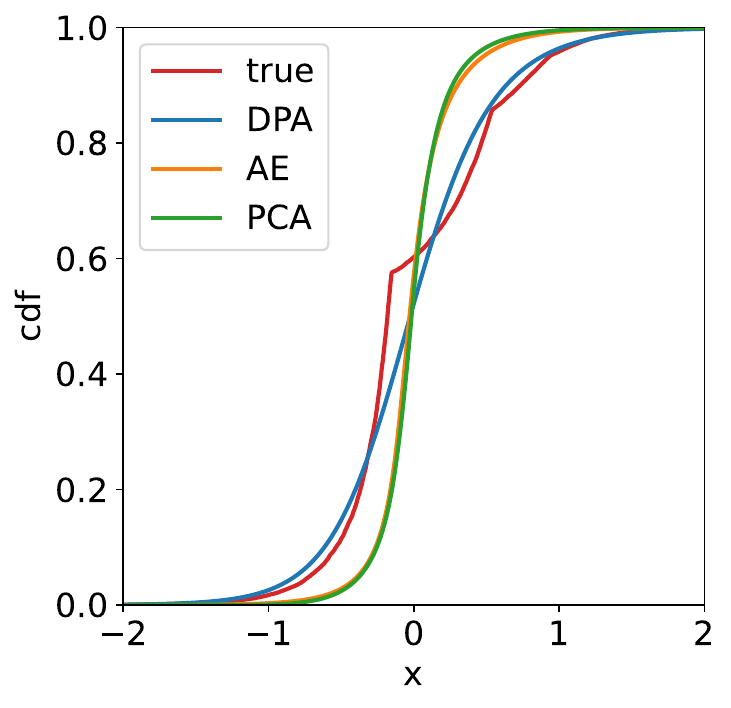} &
	\includegraphics[width=0.16\textwidth,align=c]{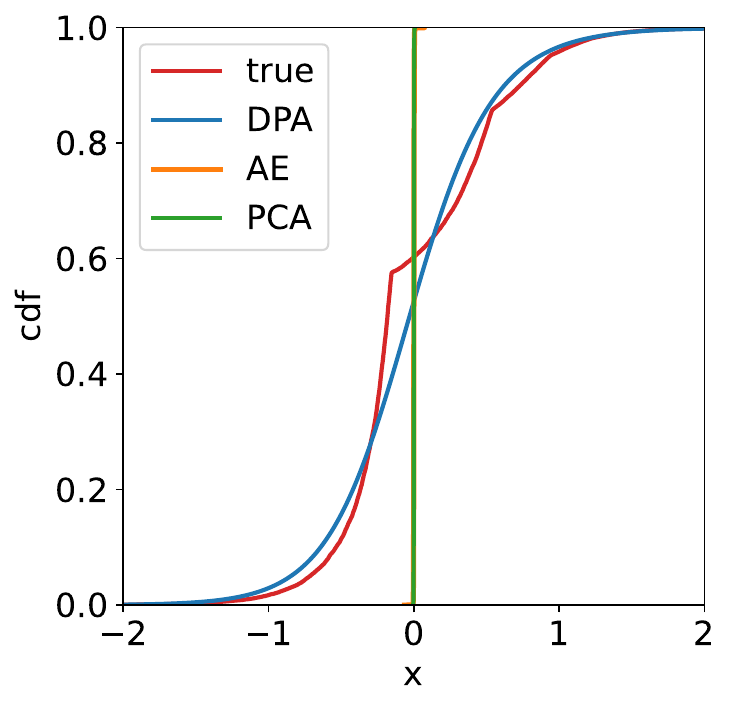} \\
 	&	\small PCA & \small AE & \small DPA & \small $k=8$ & \small $k=2$ & \small $k=0$\vspace{-2pt} \\
\end{tabular}
\caption{Q--Q plots and empirical cdfs of marginal distributions at a random location for test data versus fitted distributions. The vertical lines for PCA and AE with $k=0$ are due to centered data.}\label{fig:qq_cdf}
\end{figure}

\subsection{Embeddings}\label{subsec:embedding}
In many scientific applications, it is useful to visualise high-dimensional data in a 2-dimensional latent space. Here, we investigate how well the 2D principal components obtained by DPA can preserve the underlying structure in data. We consider global precipitation data and 8 single-cell data sets; nonlinear manifold methods tailored to single-cell RNA-seq data, such as those of \citet{verma2020robust}, target a similar low-dimensional embedding goal in this domain. For comparison, we use PCA, AE, VAE, WAE, and the widely used data-visualisation methods t-SNE~\citep{van2008visualizing} and UMAP~\citep{McInnes2018}.

In Figure~\ref{fig:2d_visual}, we visualise the 2D embeddings for global precipitation fields obtained by different methods. The training data consists of 12 months and are simulated by different general circulation models each corresponding to a slightly different physical mechanism and thus a distinct distribution. In the top-left panel, DPA embeddings preserve such underlying structures very well: (i) the seasonal cycle is precisely captured through the angle and (ii) data from different models are well separated according to the radius in the latent space. The embeddings obtained by all other methods share at most one of these two properties. Specifically, PCA and WAE embeddings reflect some seasonal cycle but tend to confuse winter months and do not fully separate data from different models. With t-SNE,  different models and months are well separated, but there is no structural pattern preserved in its 2D latent space. UMAP overly disperses the patterns while AE and VAE embeddings also do not provide much meaningful information. 

In Table~\ref{tab:knn}, we report for the single cell data sets the k-nearest neighbour accuracy of classifying the cell type based on the 2D embeddings. We use five neighbours throughout. 
For all eight data sets, DPA achieves the highest or second-highest accuracy, indicating that the 2D DPA embeddings preserve the local structure of high-dimensional gene expressions. Among existing methods, AE and WAE are the closest competitors but typically fall short of DPA.  
Note that t-SNE and UMAP do not produce an explicit mapping to the low-dimensional latent space and can hence not be evaluated on new test data. For evaluation, we need to run them on test data and hence exclude them in the ranking in Table~\ref{tab:knn}. Despite this, the k-NN accuracy of DPA is often close to those of t-SNE and UMAP, and sometimes even better, e.g.\ on \textsc{sc1}, \textsc{sc6}, and \textsc{sc7}.

\begin{figure}
\centering
\begin{tabular}{@{}c@{}c@{}cc@{}}
	\small ~~~~~DPA & \small ~~~~~AE & \small ~~~~~WAE\vspace{-4pt} \\
	 
	\includegraphics[width=0.29\textwidth]{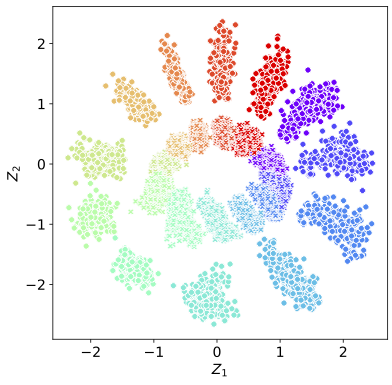} &
	\includegraphics[width=0.29\textwidth]{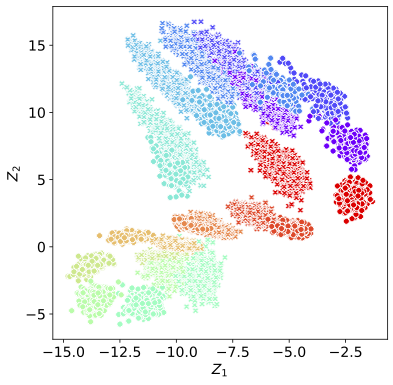} &
	\includegraphics[width=0.29\textwidth]{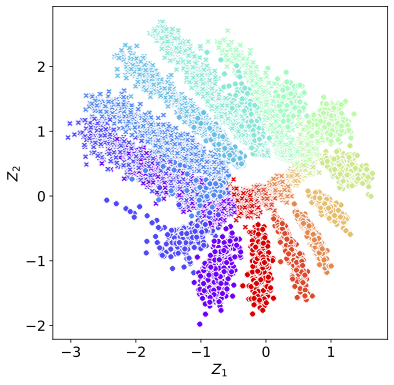} &
	\includegraphics[height=0.28\textwidth]{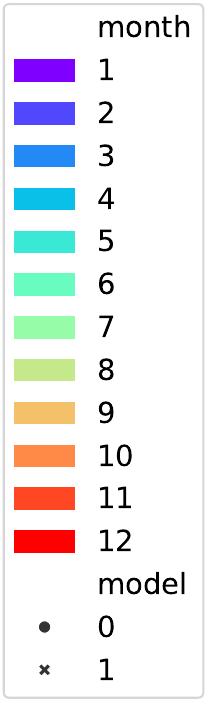}	
\end{tabular}
\begin{tabular}{@{}c@{}c@{}c@{}c@{}}
	\small ~~~~~PCA & \small ~~~~~t-SNE & \small ~~~~~UMAP & \small ~~~~~VAE \vspace{-4pt}\\
	\includegraphics[width=0.25\textwidth]{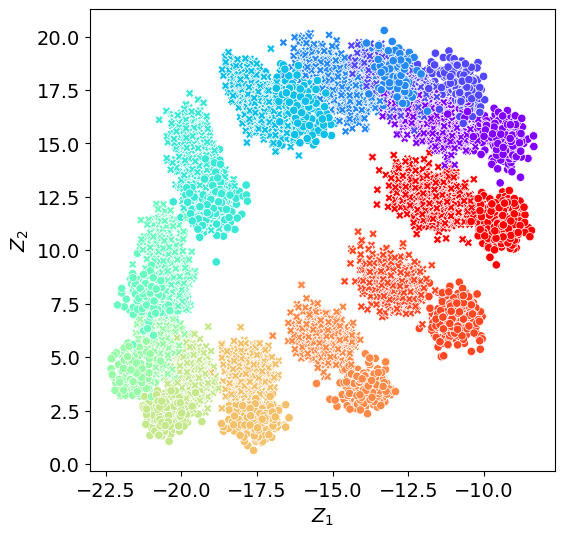} &
	\includegraphics[width=0.25\textwidth]{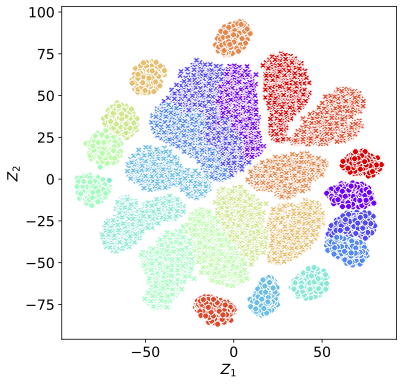} &
	\includegraphics[width=0.25\textwidth]{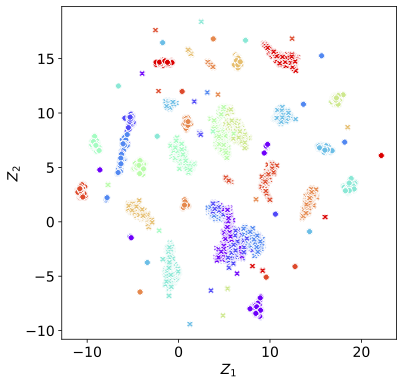} 	&
	\includegraphics[width=0.25\textwidth]{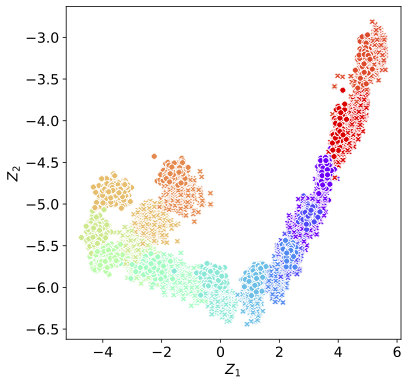}
\end{tabular}\vspace{-0.1in}
\caption{2D visualisation for global precipitation fields with a spatial dimension of $360\times180$. Each color represents a month and each shape stands for a general circulation model. }\label{fig:2d_visual}
\end{figure}

\begin{table}
\caption{\label{tab:knn}k-NN classification accuracy for the cell type using 2D embeddings on 8 single-cell data. All methods in the first panel were evaluated on test data that were not used for training, while t-SNE and UMAP were run and evaluated on the test data.}\vspace{4pt}
\centering
\begin{tabular}{ccccccccc}
\toprule
 & \textsc{sc1} & \textsc{sc2}  & \textsc{sc3} & \textsc{sc4} & \textsc{sc5} & \textsc{sc6} & \textsc{sc7} & \textsc{sc8} \\\midrule
DPA  & \bf 0.659  & \bf 0.882 & \bf 0.285            & \bf 0.830 & \bf 0.858 & \bf 0.802  & \bf 0.792   & \bf 0.521        \\
AE   & \bf 0.643  & \bf 0.877 & 0.219            & \bf 0.793 & \bf 0.850 & \bf 0.764  & 0.790   & 0.513        \\
VAE  & 0.570  & 0.868 & 0.137            & 0.730 & 0.714 & 0.734  & 0.752   & 0.501        \\
WAE  & 0.623  & 0.849 & \bf 0.246            & 0.787 & 0.819 & 0.722  & \bf 0.804   & \bf 0.535        \\
PCA     & 0.483  & 0.839 & 0.170            & 0.670 & 0.540 & 0.513  & 0.300   & 0.440        \\\midrule
t-SNE   & 0.554  & 0.901 & 0.537            & 0.970 & 0.920 & 0.757  & 0.662   & 0.594        \\
UMAP    & 0.393  & 0.890 & 0.430            & 0.960 & 0.892 & 0.757  & 0.562   & 0.559    
\\\bottomrule   
\end{tabular}
\end{table}

\section{Stochastic encoder as a posterior}\label{sec:eb}
An extension of the DPA framework with a stochastic encoder enables inference of latent variables. In latent variable modelling, the observed data $X$ is generated through a two-step process
\begin{equation}\label{eq:twostep}
	p(x)=\int p(z)p(x|z) dz,
\end{equation}
with a prior $p(z)$ on the latent variable $Z$ and a likelihood $p(x|z)$. Here we focus on an empirical Bayes (EB)~\citep{efron2019bayes} setting and demonstrate through the derivations below how the DPA objective can be adapted flexibly for different statistical problems. We defer more systematic studies of the method, including comparisons to other deep EB approaches, and other inference settings to future work. 

In EB, the likelihood is assumed known and fixed, denoted by $\bar{p}(x|z)$. Let $\bar{d}(z,\varepsilon)$ denote the generator that induces this known likelihood, i.e.\ $\bar{d}(z,\varepsilon)\sim \bar{p}(x|z)$, where $\varepsilon$ follows the standard Gaussian. 
The core of EB is the (unknown) prior $\bar{p}(z)$ which depends on the data distribution such that \eqref{eq:twostep} holds for the observed data distribution $\bar{p}(x)$ and known likelihood $\bar{p}(x|z)$, i.e.
\begin{equation*}
	\bar{p}(x)=\int \bar{p}(z)\bar{p}(x|z) dz.
\end{equation*}
This also defines the corresponding posterior distribution for the latent variable
\begin{equation}\label{eq:posterior}
	\bar{p}(z|x)\propto \bar{p}(z)\bar{p}(x|z),
\end{equation}
which is the target of EB inference.

To model a posterior over latent variables, the encoder must be stochastic. We replace the deterministic encoder by a general stochastic mapping $e(x,\eta):\bbR^p\times\bbR^q\to\bbR^k$, where $\eta$ is a multivariate standard Gaussian. The stochastic encoder is unrestricted in distributional form, so $e(x,\eta)$ can induce any posterior over the latent variable given $X=x$.

For a stochastic encoder, we define the DPA objective $\cL_{\rm DPA}(e,d)$ similarly to \eqref{eq:obj_joint_fix_k}, fixing the exponent $\beta=1$ throughout this section:
\begin{equation*}
	\cL_{\rm DPA}(e,d)=\bbE\|X-d(e(X,\eta),\varepsilon)\| - \frac12\bbE\|d(e(X,\eta),\varepsilon) - d(e(X,\eta),\varepsilon')\|,
\end{equation*}
where the expectation is over $X\sim P^*$, the encoder noise $\eta$, and independent decoder noises $\varepsilon,\varepsilon'$. When the encoder is deterministic, this reduces to \eqref{eq:obj_joint_fix_k} with $\beta=1$.

We propose the following formulation for learning the empirical Bayes posterior:
\begin{equation}\label{eq:eb}
	\tilde{e}\in\argmin_e\left\{\cL_{\rm DPA}(e,\bar{d}) - \min_d\cL_{\rm DPA}(e,d)\right\},
\end{equation}
where the first term is the DPA objective with the likelihood generator, and the second term is the minimal value of the DPA objective with the optimal decoder for the current encoder. The following proposition interprets this objective as an energy distance between the modelled likelihood and true likelihood. 
\begin{proposition}\label{prop:eb_obj}
	The objective in \eqref{eq:eb} is equal to
	\begin{equation*}
		\frac{1}{2}\bbE\big[D(P(X|e(X,\eta)),P(\bar{d}(e(X,\eta),\varepsilon)|e(X,\eta)))\big],
	\end{equation*}
	where $D(P,Q)$ is the energy distance defined in \eqref{eq:energy_distance} with $\beta=1$. 
\end{proposition}
This means minimising the objective in fact matches the following two distributions
\begin{equation*}
	X|Z\ \sim\ \bar{d}(Z,\varepsilon)|Z,
\end{equation*}
where $Z:=\tilde{e}(X,\eta)$, which implies $X|Z=z\sim\bar{p}(x|z)$. Then the following result shows the learned encoder recovers the EB posterior distribution.
\begin{theorem}\label{thm:eb}
	Suppose that, given the likelihood, the prior is uniquely identifiable from the observed data marginal, i.e., there exists a unique $\bar{p}(z)$ such that $\int \bar{p}(x|z)\bar{p}(z)dz=\bar{p}(x)$. Then for all $x$, $$\tilde{e}(x,\eta)\sim\bar{p}(z|x),$$ as defined in \eqref{eq:posterior}.
\end{theorem}

In practice, we solve the empirical version of \eqref{eq:eb} using a nested gradient descent algorithm, where we perform multiple gradient updates (5 is used in our experiments) for the inner problem with respect to $d$ before each update of $e$. 
We conclude this section with numerical studies of this approach in classical settings and a climate-denoising application. 

\medskip
\noindent\textbf{Classical setting}\quad 
We consider two simple examples: 
\begin{itemize}\vspace{-4pt}
	\setlength{\itemsep}{0pt}
	\setlength{\parskip}{0pt}
	\item linear-Gaussian: $Z\sim\cN(1,1)$ and $X|Z=z\sim\cN(z,1)$;
	\item Gamma-exponential: $Z\sim\mathrm{Gamma}(5,1)$ and $X|Z=z\sim\mathrm{Exponential}(z)$
\end{itemize}\vspace{-4pt}
and three methods, each adopting a different way to learn the posterior model $e(x,\eta)$:
\begin{itemize}\vspace{-4pt}
	\setlength{\itemsep}{0pt}
	\setlength{\parskip}{0pt}
	\item oracle: (supervised) engression~\citep{shen2023engression} with the latent variable observed from the true prior, i.e.
	\[\min_e \left\{\bbE\|Z-e(X,\eta)\|-\frac12\bbE\|e(X,\eta)-e(X,\eta')\|\right\};\]
	\item fixed prior: engression with the latent variable observed from a fixed, misspecified prior $p'(z)$;
	\item DPA-EB: the proposed method. 
\end{itemize}

Figure~\ref{fig:eb_simu} shows the predicted posterior means for the three methods, and Table~\ref{tab:eb_simu} reports the MSE and expected negative energy scores of the predicted posteriors against the ground truth. The advantage of DPA-EB is substantial over a postulated fixed prior, highlighting the value of empirical Bayes when the prior is unknown, and DPA-EB matches the oracle results obtained when the true prior is known.

\begin{figure}
\centering
	\begin{tabular}{@{}c@{}c@{}c@{}}
		\includegraphics[width=0.3\textwidth]{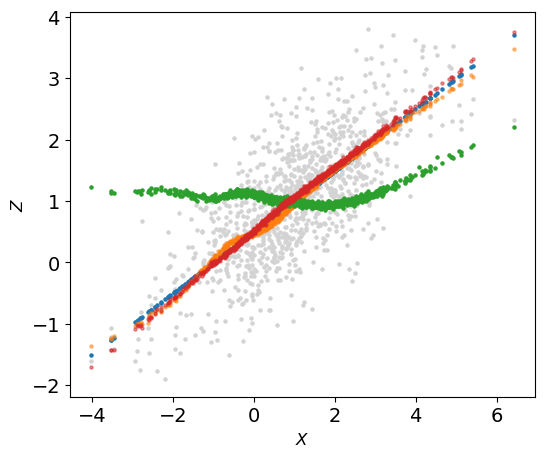}&
		\includegraphics[width=0.3\textwidth]{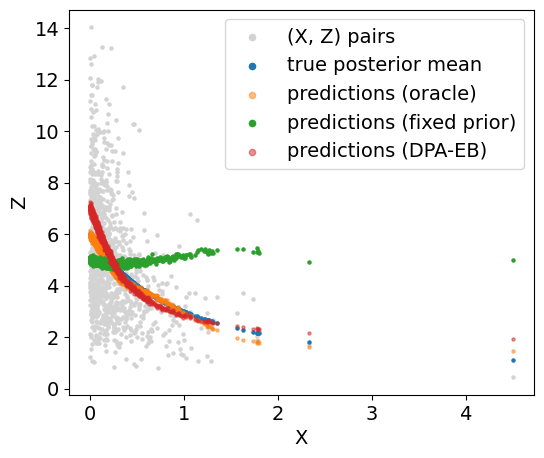}&
		\includegraphics[width=0.4\textwidth]{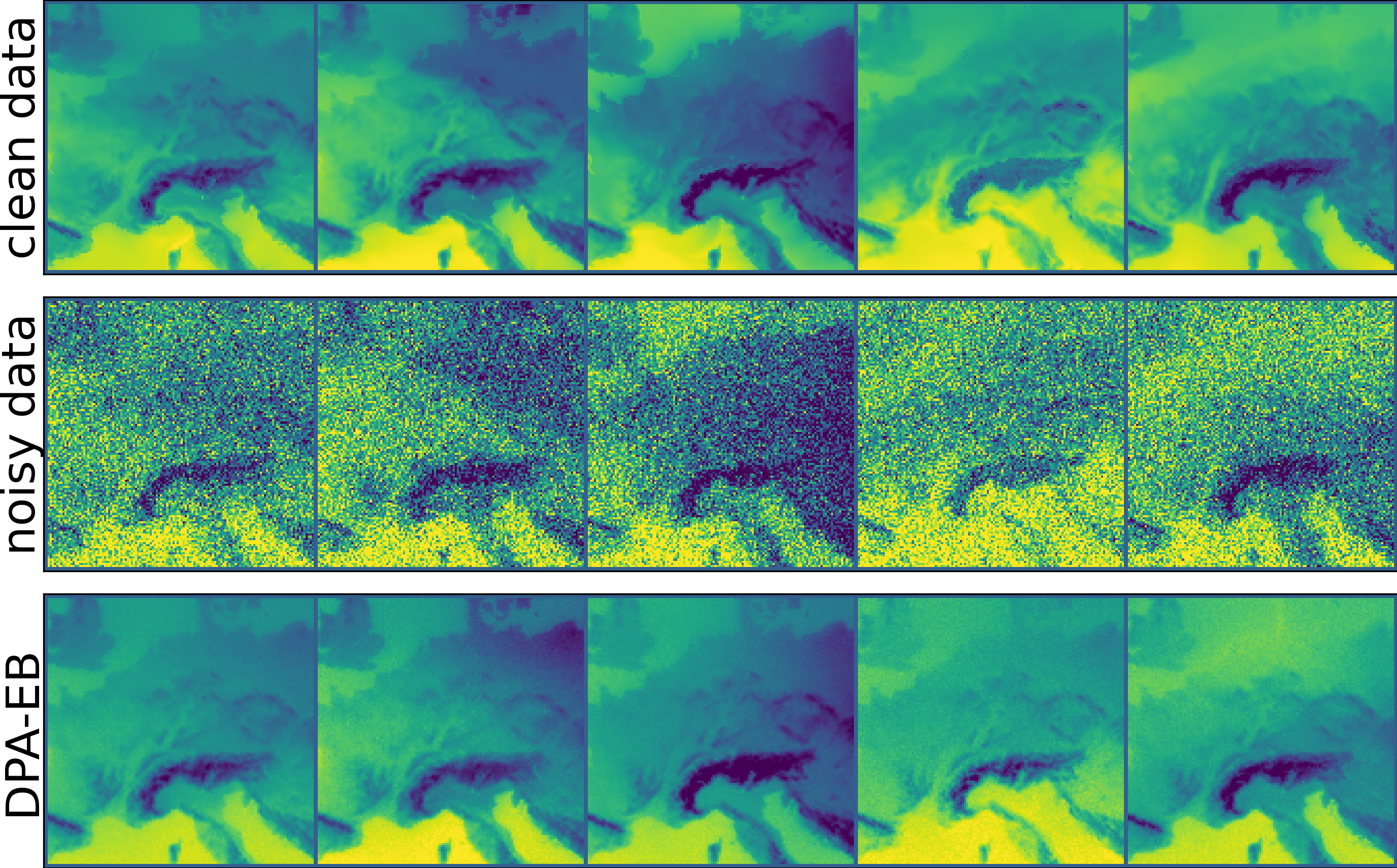}\\
		(a) linear-Gaussian & (b) Gamma-exponential & (c) climate denoising\\
	\end{tabular}
	\caption{(a,b): prediction of posterior means by different approaches compared to the ground truth. (c): climate denoising; from top to bottom are clean temperature fields, noisy fields, and denoised fields by DPA-EB.}
	\label{fig:eb_simu}
\end{figure}

\begin{table}
\caption{MSE and energy score for posterior predictions.}
\centering
\begin{tabular}{lcccc}
\hline
 & \multicolumn{2}{c}{\textbf{linear-Gaussian}} & \multicolumn{2}{c}{\textbf{Gamma-exponential}} \\
\cline{2-3} \cline{4-5}
\textbf{} & \textbf{MSE} & \textbf{energy score} & \textbf{MSE} & \textbf{energy score} \\
\hline
oracle & 0.4784 & 0.3897 & 4.6519 & 1.1808 \\
fixed prior & 0.9548 & 0.5556 & 5.5192 & 1.3083 \\
DPA-EB & 0.4798 & 0.3918 & 4.7758 & 1.2120 \\
\hline
\end{tabular}
\label{tab:eb_simu}
\end{table}

\medskip
\noindent\textbf{Climate denoising}\quad
We consider denoising of regional temperature fields (\textsc{r-temp}). On the original $128\times128$ temperature fields, we add iid Gaussian noise, resulting in noisy fields, as shown in the top two rows in Figure~\ref{fig:eb_simu}(c). The goal is denoising, i.e.\ inferring the clean data given the noisy one. It is essentially an empirical Bayes setting, where we view noisy images as observed data, clean ones as latent variables, and the process of adding noise as the known likelihood. We apply our DPA-EB approach to learn the posterior distribution of the clean data given the noisy one. Note that for training DPA, clean data are unobserved. In Figure~\ref{fig:eb_simu}(c) in the bottom row, we show samples drawn from our DPA-EB posterior given the noisy data in the middle row, which match well the clean data. The MSE between 100 test clean and noisy images is 25.0140, whereas the MSE between clean and predicted posterior mean of DPA-EB is 2.1974.

\section{Discussion}\label{sec:discuss}
We proposed DPA as a dimension reduction technique with the following properties:
\begin{enumerate}[(i)]
    \item As in Principal Component Analysis (PCA), the latent space is ordered according to how much variability of the data is explained, so users can keep few latent components (high compression with more unexplained variability) or many (lower compression with less remaining variability).
    \item As in autoencoders (AE), the encoder and decoder are nonlinear maps parametrised by deep neural networks.
    \item Unlike PCA, AE, or variants such as VAE and WAE, the DPA decoder produces samples from the oracle reconstructed distribution, namely the conditional distribution of the original data given the embedding value, using an energy-score based loss.
\end{enumerate}
The last property ensures that the reconstructed data follow in population the same distribution as the original data, yielding `distributionally lossless' compression irrespective of the latent dimension used.

We further illustrated an extension of DPA with a stochastic encoder for latent variable inference in an empirical Bayes setting, which suggests the flexibility of the DPA framework. Several other directions are worth exploring.

Arguably the main characteristic of DPA is that it reconstructs the full distribution from the latent embeddings: encoding into a low-dimensional latent space and decoding back does not change the distribution in population. In contrast, the same procedure for PCA, AE, and their variants alters the data distribution, especially in the tails, because the decoder is aiming for a conditional mean rather than the full oracle reconstructed distribution. The distribution-preserving property suggests applications in characterising distribution shifts, transport between high-dimensional distributions, and prediction problems with very high-dimensional responses.

One example is transport between high-dimensional distributions through the DPA latent space. For instance, the data is a mixture of two high-dimensional distributions $P_1$ and $P_2$ and the goal is to transport from $P_1$ to $P_2$, for which the optimal transport in the high-dimensional space is often intractable. Now with DPA, the latent variables also follow a mixture of distributions; denote $e(X)\sim Q_i$ when $X\sim P_i$, for $i=1,2$. The optimal transport from $Q_1$ to $Q_2$ in the low-dimensional (e.g.\ 2D) latent space becomes much more tractable. More importantly, the distributional reconstruction ability of DPA guarantees that after transportation from $Q_1$ to $Q_2$ in the latent space and decoding back to the original space, the decoded distribution of transported latents will match the target high-dimensional distribution $P_2$. 
In the context of the climate data examples mentioned, this opens up the possibility of correcting the bias of climate models ($P_1$) with respect to the observational data ($P_2$) by transportation from model data to match the full distribution of the observational data, rather than just to match the first one or two moments as in conventional bias correction techniques. The DPA embeddings in Figure~\ref{fig:2d_visual} further suggest that a transport from model 0 to model 1 in the 2D latent space could be as simple as roughly doubling the radius.

In addition, the latent subspace identified by DPA is generally different from that of existing methods such as PCA and AE. The numerical results for climate and single-cell data in Section~\ref{subsec:embedding} show how DPA improves the embeddings of PCA and AE on some metrics. For example, the DPA embeddings exhibit better separation between different temporal patterns or cell types. As another example, consider bivariate isotropic data where one component follows a uniform distribution and the other a t-distribution. PCA picks an arbitrary direction as its first principal subspace, whereas DPA with a linear encoder would favour the light-tailed component over the heavy-tailed one. This distinction connects to identification results for the latent variables, relevant to causal representation learning \citep{khemakhem2020variational,kivva2022identifiability}.

DPA can also be applied to the regression setting in which a response variable $Y$ is to be predicted from a predictor $X$. Replacing the data $X$ in the encoder objective \eqref{eq:obj_enc} by the response $Y$, while letting the encoder act on the predictor, yields a nonparametric formulation of sufficient dimension reduction in which the embedding $e(X)$ is chosen to minimise the unexplained variability of $Y|e(X)$. The DPA setting considered in this paper is the special case $Y=X$.

\bibliography{ref.bib}
\bibliographystyle{apalike}

\newpage
\appendix
\section{Proofs}\label{app:proof}

\begin{proof}[Proof of Proposition~\ref{prop:two_terms_es_equal}]
Let $Z=e(X)$ with a joint density $p_Z$ and denote by $p_{e(X)}(x|z)$ the conditional density of $X$ given $e(X)=z$. 
Write the density of $P^*$ as
$
	p(x) = \int p_Z(z)p_{e(X)}(x|z) dz.
$
Then we have
\begin{align*}
	\bbE_{X}\bbE_{Y\sim P^*_{e,X}}\|X - Y\|^{\beta} &= \int p_Z(z)p_{e(X)}(x|z) p_{e(X)}(y|e(x)) \|x-y\|^{\beta}dxdydz \\
	&\overset{(a)}= \int p_Z(z)p_{e(X)}(x|z) p_{e(X)}(y|z) \|x-y\|^{\beta}dxdydz\\
	&= \bbE_{Z\sim p_Z}\bbE_{X,Y\sim P^*_{e(X)=Z}}\big[\|X - Y\|^{\beta}\big]\\
	&= \bbE_{X\sim P^*}\bbE_{Y,Y'\sim P^*_{e,X}}\big[\|Y - Y'\|^{\beta}\big],
\end{align*}
where $(a)$ comes from the fact that any samples $Y$ from $P^*_{e,X}$ satisfy $e_i(Y)=e_i(X)$ almost surely for all $i=1,\dots,k$; thus $p_{e(X)}(x|z) p_{e(X)}(y|e(x))\neq0$ only when $e(y)=e(x)=z$.
\end{proof}

\begin{proof}[Proof of Proposition~\ref{prop:lin_gauss1}]
According to \citet[Theorem~4]{majumdar2019conditional}, the conditional distribution of $X$ given $e(X)=M^\top X$, denoted by $P^*_{e,X}$, is $\cN(\nu(X), G)$, where 
\begin{equation*}
	\nu(X)=\Sigma^*M(M^\top \Sigma^*M)^{-1}M^\top X,\quad G=\Sigma^* - \Sigma^*M(M^\top \Sigma^*M)^{-1}M^\top \Sigma^*.
\end{equation*}
Then by Proposition~\ref{prop:two_terms_es_equal}, the objective function in \eqref{eq:obj_enc} becomes
\begin{equation*}
	\bbE_{X}\bbE_{Y\sim P^*_{e,X}}\|X - Y\|^{2} = \bbE_{X}\bbE_{Y,Y'\sim P^*_{e,X}}\|Y - Y'\|^2 = \bbE\|G^{1/2}(\varepsilon-\varepsilon')\|^2 \propto \tr(G)
\end{equation*}
where $\varepsilon$ and $\varepsilon'$ are two independent draws from $\cN(0,I_k)$. Now formulation \eqref{eq:obj_enc} is equivalent to  
\begin{equation*}
	\begin{split}
		&\max_M\quad\tr\big((M^\top \Sigma^*M)^{-1}(M^\top \Sigma^*\Sigma^* M)\big)\\
		&\text{subject to}\quad M^\top M=I_k
	\end{split}
\end{equation*}
According to \citet[Theorem~2.3]{yu2011kernel}, the solution to the above problem is given by $M=Q_{:k}$ up to column permutations. Then we conclude the desired result. 
\end{proof}

\begin{proof}[Proof of Proposition~\ref{prop:mean_recon}]
By Proposition~\ref{prop:two_terms_es_equal}, we have 
	\begin{equation*}
		\bbE_{X}\bbE_{Y\sim P^*_{e,X}}\|X - Y\|^2 = \bbE_{X}\bbE_{Y,Y'\sim P^*_{e,X}}\|Y - Y'\|^2.
	\end{equation*}
Also note
\begin{align*}
	\bbE_{X}\bbE_{Y,Y'\sim P^*_{e,X}}\|Y - Y'\|^2 &= \bbE_{X}\bbE_{Y,Y'\sim P^*_{e,X}}\|Y - \bbE[Y]  + \bbE[Y'] - Y'\|^2\\
	&= 2\bbE_{X}\bbE_{Y\sim P^*_{e,X}}\|Y - \bbE[Y]\|^2,
\end{align*}
and 
\begin{align*}
	\bbE_{X}\bbE_{Y\sim P^*_{e,X}}\|X - Y\|^2 = \bbE_X\|X - \bbE_{Y\sim P^*_{e,X}}[Y]\|^2 + \bbE_X\bbE_{Y\sim P^*_{e,X}}\|Y - \bbE[Y]\|^2.
\end{align*}
Therefore,
\begin{align*}
	\bbE_X\|X - \bbE_{Y\sim P^*_{e,X}}[Y]\|^2 &= \bbE_{X}\bbE_{Y\sim P^*_{e,X}}\|X - Y\|^2 - \frac12\bbE_{X}\bbE_{Y,Y'\sim P^*_{e,X}}\|Y - Y'\|^2\\
	&= \frac12\bbE_{X}\bbE_{Y\sim P^*_{e,X}}\|X - Y\|^2,
\end{align*}
which leads to the desired equivalence. 
\end{proof}

The following lemma based on the results in \citet{szekely2003statistics} and \citet{szekely2023energy} states that the energy score is a strictly proper scoring rule. 
\begin{lemma}\label{lem:es}
	For any distribution $P'$, we have $$\bbE_{X^*\sim P^*,X\sim P}\big[\|X^*-X\|^\beta\big]-\frac{1}{2}\bbE_{X,X'\sim P}\big[\|X-X'\|^\beta\big] \ge \frac{1}{2}\bbE_{X,X'\sim P^*}\big[\|X-X'\|^\beta\big], $$ where the equality holds if and only if $P$ and $P^*$ are identical.
\end{lemma}
The same property underpins the consistency analysis of \citet{shen2023engression} in the regression setting.

\begin{proof}[Proof of Proposition~\ref{prop:opt_d}]
When $\lambda=1/2$ and $\beta\in(0,2)$, the objective function in \eqref{eq:obj_dec} is equal to
\begin{equation}\label{eq:obj_dec1}
\begin{split}
	&\bbE_{X}\bbE_{Y\sim P_{d,e(X)}}\big[\|X - Y\|^{\beta}\big] - \frac{1}{2}\bbE_{X}\bbE_{Y,Y'\overset{\rm iid}\sim P_{d,e(X)}}\big[\|Y - Y'\|^{\beta}\big] \\
	=\ &\bbE_{X}\bbE_{Y^*\sim P^*_{e,X}, Y\sim P_{d,e(X)}}\big[\|Y^* - Y\|^{\beta}\big] - \frac{1}{2}\bbE_{X}\bbE_{Y,Y'\overset{\rm iid}\sim P_{d,e(X)}}\big[\|Y - Y'\|^{\beta}\big].
\end{split}
\end{equation} 
Given $X=x$, it is the expected energy score between the oracle reconstructed distribution $P^*_{e,X}$ and the reconstructed distribution induced by the decoder $P_{d,e(x)}$:
\begin{equation*}
	\bbE_{Y^*\sim P^*_{e,X}, Y\sim P_{d,e(x)}}\big[\|Y^* - Y\|^{\beta}\big] - \frac{1}{2}\bbE_{Y,Y'\overset{\rm iid}\sim P_{d,e(x)}}\big[\|Y - Y'\|^{\beta}\big].
\end{equation*}
By Lemma~\ref{lem:es}, we know that the above objective is minimised if and only if $P_{d,e(x)}$ and $P^*_{e,X}$ are identical with the minimal value 
\begin{equation*}
	\frac{1}{2}\bbE_{Y,Y'\overset{\rm iid}\sim P^*_{e,X}}\big[\|Y - Y'\|^{\beta}\big],
\end{equation*}
for all $x$. The assumption guarantees the existence of such optimal decoders $d^*$ that $P_{d^*,e(x)}=P^*_{e,X}$ for any $e$. Thus $d^*$ also minimises \eqref{eq:obj_dec1} with the minimal value \eqref{eq:obj_dec} given by
\begin{equation*}
\frac{1}{2}\bbE_{X}\bbE_{Y,Y'\overset{\rm iid}\sim P^*_{e,X}}\big[\|Y - Y'\|^{\beta}\big]\\
	=\bbE_{X}\bbE_{Y\sim P^*_{e,X}}\big[\|X - Y\|^{\beta}\big] - \frac{1}{2}\bbE_{X}\bbE_{Y,Y'\overset{\rm iid}\sim P^*_{e,X}}\big[\|Y - Y'\|^{\beta}\big],
\end{equation*}
where the equality is due to Proposition~\ref{prop:two_terms_es_equal} and the right-hand side is exactly \eqref{eq:obj_enc1} with $\lambda=1/2$.
\end{proof}

\begin{proof}[Proof of Theorem~\ref{thm:dpa_onek}]
	For any given encoder $e$, as $\cD$ is rich enough, Proposition~\ref{prop:opt_d} shows that 
\begin{equation*}
	\min_d \left\{\bbE_{X}\bbE_{Y\sim P_{d,e(X)}}\big[\|X - Y\|^{\beta}\big] - \frac{1}{2}\bbE_{X}\bbE_{Y,Y'\overset{\rm iid}\sim P_{d,e(X)}}\big[\|Y - Y'\|^{\beta}\big]\right\} = \frac{1}{2}\bbE_{X}\bbE_{Y,Y'\overset{\rm iid}\sim P^*_{e,X}}\big[\|Y - Y'\|^{\beta}\big].
\end{equation*}
Hence the optimisation problem in \eqref{eq:obj_joint_fix_k} in terms of $e$ is equivalent to 
\begin{equation*}
	\min_e\bbE_{X}\bbE_{Y,Y'\overset{\rm iid}\sim P^*_{e,X}}\big[\|Y - Y'\|^{\beta}\big].
\end{equation*}
Thus (i) holds.
Moreover, for $e=e^*$, Proposition~\ref{prop:opt_d} shows the optimal decoder $d^*$ that minimises
\begin{equation*}
	\bbE_{X}\bbE_{Y\sim P_{d,e^*(X)}}\big[\|X - Y\|^{\beta}\big] - \frac{1}{2}\bbE_{X}\bbE_{Y,Y'\overset{\rm iid}\sim P_{d,e^*(X)}}\big[\|Y - Y'\|^{\beta}\big]
\end{equation*}
satisfies $P_{d^*,e^*(X)}=P^*_{e^*,X}$, which leads to (ii).
\end{proof}

\begin{proof}[Proof of Proposition~\ref{prop:simu_opt_d}]
	Since $e$ is invertible, define $d^*:\bbR^p\to\bbR^p$ as the inverse of $e$, so that $d^*(e(x))=x$ for all $x\in\cX$. We use $d^*$ as a two-argument map by splitting its $p$-dimensional input as $d^*(z_{1:k},z_{(k+1):p}):=d^*((z_{1:k},z_{(k+1):p}))$. As $e(X)\overset{d}= \varepsilon$ with independent components, we have $(e_{(k+1):p}(X)|e_{1:k}(X))\overset{d}=\varepsilon_{(k+1):p}$ for all $k$. This implies that the conditional distribution of $X=d^*(e(X))$ given $e_{1:k}(X)$ is equal to the conditional distribution of $d^*(e_{1:k}(X),\varepsilon_{(k+1):p})$ given $e_{1:k}(X)$.
\end{proof}

\begin{proof}[Proof of Proposition~\ref{prop:lin_gauss2}]
	For a fixed $k$, we know from Proposition~\ref{prop:lin_gauss1} that $e_{1:k}(x)=\Pi_k Q_{:k}^\top x$ minimises each term $\bbE_{X}\bbE_{Y\sim P^*_{e_{1:k}(X)}}\|X - Y\|^{\beta}$ in \eqref{eq:obj_enc_all_k}, for some permutation matrix $\Pi_k$. The constraint that $e_{1:k}$ is the first $k$ rows of a single common encoder $e:\bbR^p\to\bbR^p$ forces the family $\{\Pi_k\}_k$ to be consistent across $k$: each $\Pi_k$ must be the leading $k\times k$ block of the same $p\times p$ permutation. With strictly positive weights $\{\omega_k\}$, this consistency requirement forces $\Pi_k=I_k$ for all $k$, so $e^*(x)=Q^\top x$ uniquely, independent of the choice of weights.
\end{proof}

\begin{proof}[Proof of Proposition~\ref{prop:eb_obj}]
	According to Proposition~\ref{prop:opt_d}, the optimal solution to $\min \cL_{\rm DPA}(e,d)$ for a fixed $e$, denoted by $\tilde{d}$, satisfies
	\begin{equation*}
		\tilde{d}(z,\varepsilon)\sim P(X|e(X,\eta)=z).
	\end{equation*}
	Thus we have
	\begin{align*}
		\min_d \cL_{\rm DPA}(e,d) &= \bbE\|X-\tilde{d}(e(X,\eta),\varepsilon)\| - \frac12\bbE\|\tilde{d}(e(X,\eta),\varepsilon) - \tilde{d}(e(X,\eta),\varepsilon')\| \\
		&= \frac12\bbE\|\tilde{d}(e(X,\eta),\varepsilon) - \tilde{d}(e(X,\eta),\varepsilon')\|.
	\end{align*}
	Then the whole objective function in \eqref{eq:eb} becomes
	\begin{equation*}
		\bbE\|X-\bar{d}(e(X,\eta),\varepsilon)\| - \frac12\bbE\|\bar{d}(e(X,\eta),\varepsilon) - \bar{d}(e(X,\eta),\varepsilon')\| - \frac12\bbE\|\tilde{d}(e(X,\eta),\varepsilon) - \tilde{d}(e(X,\eta),\varepsilon')\|,
	\end{equation*}
	which, by definition, is the energy distance between the conditional distribution of $X|e(X,\eta)$ and that of $\bar{d}(e(X,\eta),\varepsilon)|e(X,\eta)$. 
\end{proof}

\begin{proof}[Proof of Theorem~\ref{thm:eb}]
	Let $Z=\tilde{e}(X,\eta)$. According to Proposition~\ref{prop:eb_obj}, we have
	\begin{equation*}
		X|Z=z\ \overset{d}= \bar{d}(z,\varepsilon) \sim \bar{p}(x|z).
	\end{equation*}
	Note also that $X\sim \bar{p}(x)$ and the prior distribution of $Z$, denoted by $p(z)$ satisfies
	\begin{equation*}
		\bar{p}(x)=\int \bar{p}(x|z)p(z)dz.
	\end{equation*}
	Due to the identifiability assumption, we have $p(z)=\bar{p}(z)$. Thus the posterior of $Z|X$  
	\begin{equation*}
		p(z|x)\propto \bar{p}(z)\bar{p}(x|z)
	\end{equation*}
	is equal to $\bar{p}(z|x)$ by \eqref{eq:posterior}.
\end{proof}

\section{Experimental details}\label{app:expe_details}
\subsection{Data sets and preprocessing}
Benchmark image data:
\begin{itemize}
	\item \textsc{mnist}: 6k training samples. Each sample is a $28\times28$ image of hand-written digits (from 0 to 9) with pixel values in $\{0,1\}$.
	\item \textsc{disk}: 10k training samples. Each sample is a $32\times32$ image of two disks with pixel values in $\{0,1\}$. Each disk is determined by three generative factors: x-position, y-position, and radius, all of which are randomly sampled. In total the intrinsic dimension is 6. 
\end{itemize}
Climate data:
\begin{itemize}
	\item \textsc{r-temp}: regional temperature data from CORDEX models \citep{regionalclimatedata} in Kelvin. The sample size is 167,645. We center the data by the temporal mean per location. 
	\item \textsc{r-precip}: regional precipitation data from CORDEX models \citep{regionalclimatedata}. The unit of the raw precipitation data is $\mbox{kg} \cdot \mbox{m}^{-2} \mbox{s}^{-1}$. Sample size is 167,645.
	\item \textsc{g-precip}: global precipitation data from CMIP6 models~\citep{gmd-9-1937-2016,gmd-8-3379-2015}, with a spatial dimension of $360\times180$. The unit of raw data is $\mbox{kg} \cdot \mbox{m}^{-2} \mbox{s}^{-1}$ and we take a square root transformation. 
\end{itemize}
Single-cell data: we consider 8 data sets from the R package \texttt{SeuratData} and follow the standard preprocessing procedures provided in the R toolkit \texttt{Seurat}~\citep{seurat} for single-cell genomics, including log transformation and standardization. We pre-select the genes with more than 20\% expressed cells (with nonzero counts). Table~\ref{tab:sc_data} lists the details for each data set. 
\begin{table}
\centering
\caption{Single-cell data sets.}\label{tab:sc_data}
\begin{tabular}{cccc}\toprule
name & data set & dimension & sample size \\\midrule
\textsc{sc1} & bmcite & 918 & 30672 \\
\textsc{sc2} & cbmc & 1118 & 7895 \\
\textsc{sc3} & celegans\_embryo & 1019 & 4883 \\
\textsc{sc4} & mousecortex & 3005 & 558 \\
\textsc{sc5} & panc8 & 6810 & 14890 \\
\textsc{sc6} & pbmc3k & 943 & 2638 \\
\textsc{sc7} & pbmcsca & 810 & 31021 \\
\textsc{sc8} & thp1\_eccite & 5806 & 20729 \\\bottomrule
\end{tabular}
\end{table}

\subsection{Hyperparameters}
For the energy score used in DPA, we always use the default choice of $\beta=1$. Throughout all experiments, we keep the hyperparameters for neural network architectures and optimisation the same for all deep learning based methods including AE, VAE, and WAE. Specifically, we use the Adam optimiser \citep{kingma2014adam} with a learning rate of $10^{-4}$, default values for the beta parameters, and a mini-batch size of 512. 

For the neural network architecture, we always adopt multilayer perceptrons (MLPs) (we vectorize image/spatial data). The encoder is a standard MLP while the decoder has to take random noises as arguments, so we adopt the architecture used by \citet{shen2023engression}. As some of our neural nets are fairly deep, we use skip-connections every two-layer. Experiments for DPA and AE for varying $k$'s in Section~\ref{subsec:recon} require rich enough model classes, so we use encoders and decoders with 16 layers; the numbers of neurons per layer are 512 for \textsc{mnist} and \textsc{disk}, 1000 for all single-cell data sets, 5000 for \textsc{r-temp} and \textsc{r-precip}, and 2048 for \textsc{g-precip}, which are rather arbitrary choices mainly to fit into the GPU memory. For the decoder, we concatenate to each layer a 100-dimensional standard Gaussian noise except for the global precipitation data with the highest dimension and complexity for which we concatenate a 500-dimensional Gaussian noise. Experiments in Section~\ref{subsec:embedding} are only for a latent dimension of 2, so we use shallower networks with 4 layers for both encoders and decoders, and we keep the same network width as before. All experiments are run on a single NVIDIA RTX 4090 GPU.

For non-deep learning methods, we use the Python \texttt{scikit-learn}  library for PCA, UMAP, and t-SNE with the default parameters. For \textsc{g-precip} with a very high dimension and sample size so that the data could not be loaded into the memory, we train a linear autoencoder with mean squared reconstruction error (with intercepts) and Adam optimiser with the same hyperparameter settings.
%

\subsection{Additional empirical results}\label{app:add_exp}
We report results from VAE and WAE in addition to the comparisons in Figures~\ref{fig:mnist_recon}-\ref{fig:rcmp_recon} of the main text. Because each VAE or WAE model is trained for one fixed latent dimension, we run them with $k=2$ and show in Figure~\ref{fig:add_exp} their reconstructions on four data sets in comparison to DPA, PCA, and AE with the same fixed $k=2$.
\begin{figure}
	\centering
	\begin{tabular}{c@{}BBBB}
		& \textsc{mnist} & \textsc{disk} & \textsc{r-temp} & \textsc{r-precip}\\
		\rotatebox[origin=c]{90}{\small{true}}\hspace{4pt}\smallskip & 
		\includegraphics[width=0.23\textwidth,align=c]{fig/mnist/true_i1.png} &
		\includegraphics[width=0.23\textwidth,align=c]{fig/kdisk/true_i0.png} &
		\includegraphics[width=0.23\textwidth,align=c]{fig/rcm_t/true_i7.png} &
		\includegraphics[width=0.23\textwidth,align=c]{fig/rcm_p/true_i0.png} \\
		\rotatebox[origin=c]{90}{\small{PCA}}\hspace{4pt}\smallskip & 
		\includegraphics[width=0.23\textwidth,align=c]{fig/mnist/pca_i1_k2.png} &
		\includegraphics[width=0.23\textwidth,align=c]{fig/kdisk/pca_i0_k2.png} &
		\includegraphics[width=0.23\textwidth,align=c]{fig/rcm_t/pca_i7_k2.png} &
		\includegraphics[width=0.23\textwidth,align=c]{fig/rcm_p/pca_i0_k2.png} \\
		\rotatebox[origin=c]{90}{\small{AE}}\hspace{4pt}\smallskip & 
		\includegraphics[width=0.23\textwidth,align=c]{fig/mnist/iob_i1_k2.png} &
		\includegraphics[width=0.23\textwidth,align=c]{fig/kdisk/iob_i0_k2.png} &
		\includegraphics[width=0.23\textwidth,align=c]{fig/rcm_t/iob_i7_k2.png} &
		\includegraphics[width=0.23\textwidth,align=c]{fig/rcm_p/iob_i0_k2.png} \\
		\rotatebox[origin=c]{90}{\small{VAE}}\hspace{4pt}\smallskip & 
		\includegraphics[width=0.23\textwidth,align=c]{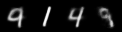} &
		\includegraphics[width=0.23\textwidth,align=c]{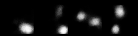} &
		\includegraphics[width=0.23\textwidth,align=c]{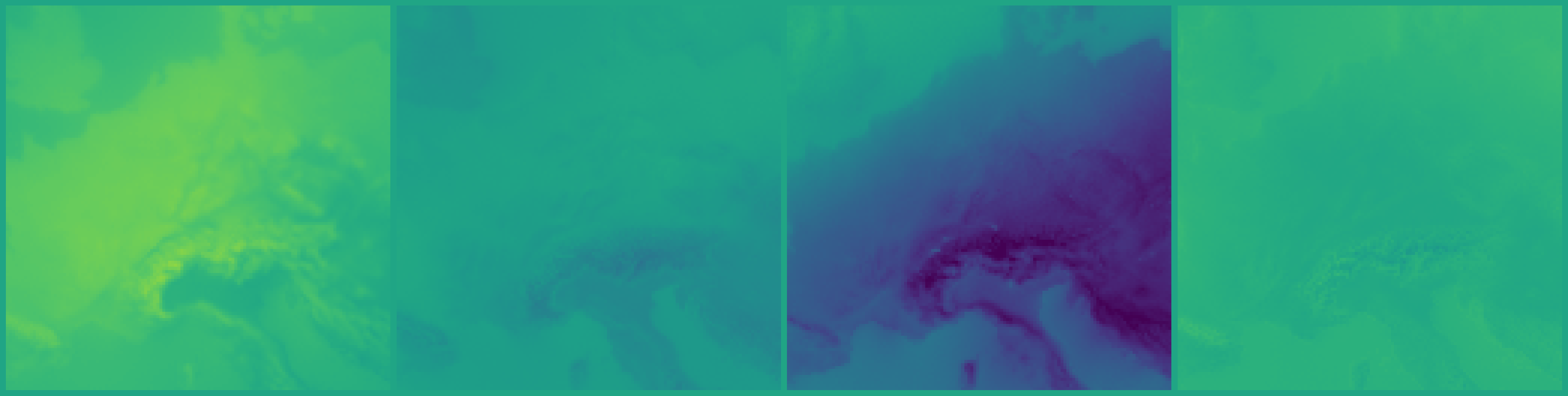} &
		\includegraphics[width=0.23\textwidth,align=c]{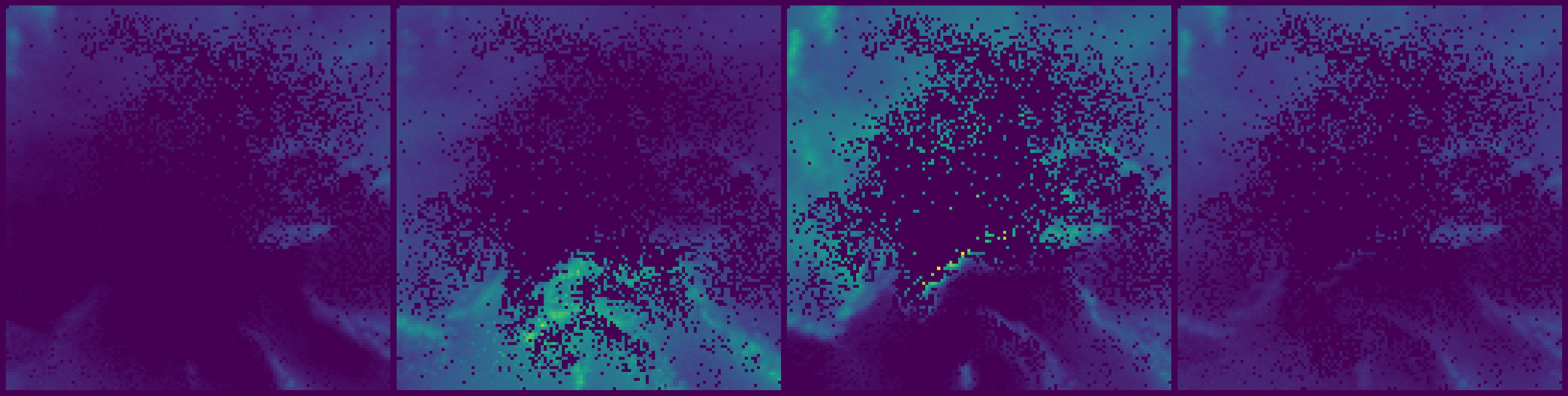} \\
		\rotatebox[origin=c]{90}{\small{WAE}}\hspace{4pt}\smallskip & 
		\includegraphics[width=0.23\textwidth,align=c]{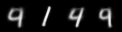} &
		\includegraphics[width=0.23\textwidth,align=c]{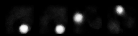} &
		\includegraphics[width=0.23\textwidth,align=c]{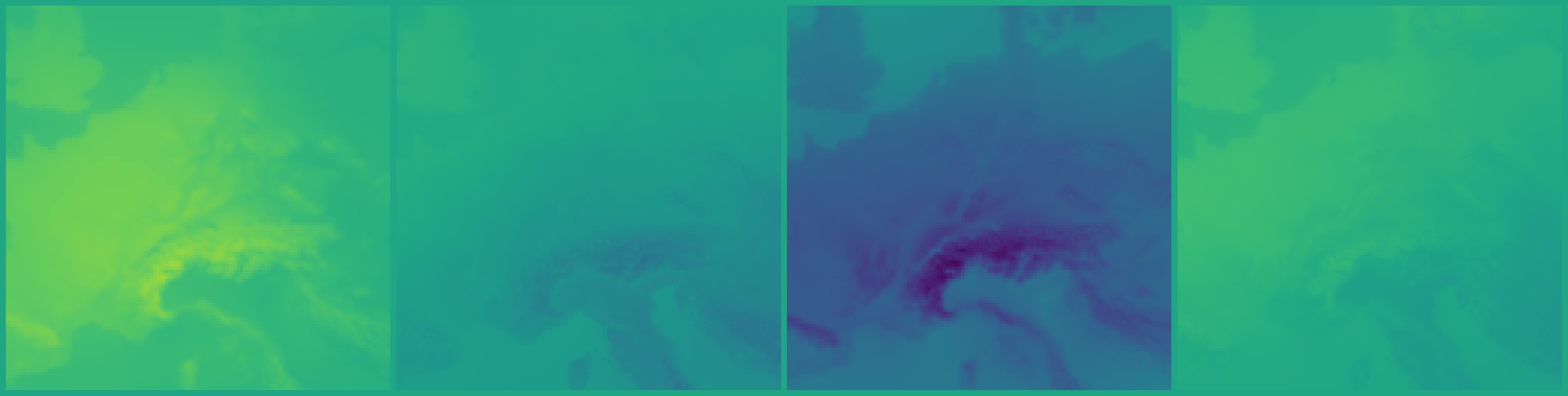} &
		\includegraphics[width=0.23\textwidth,align=c]{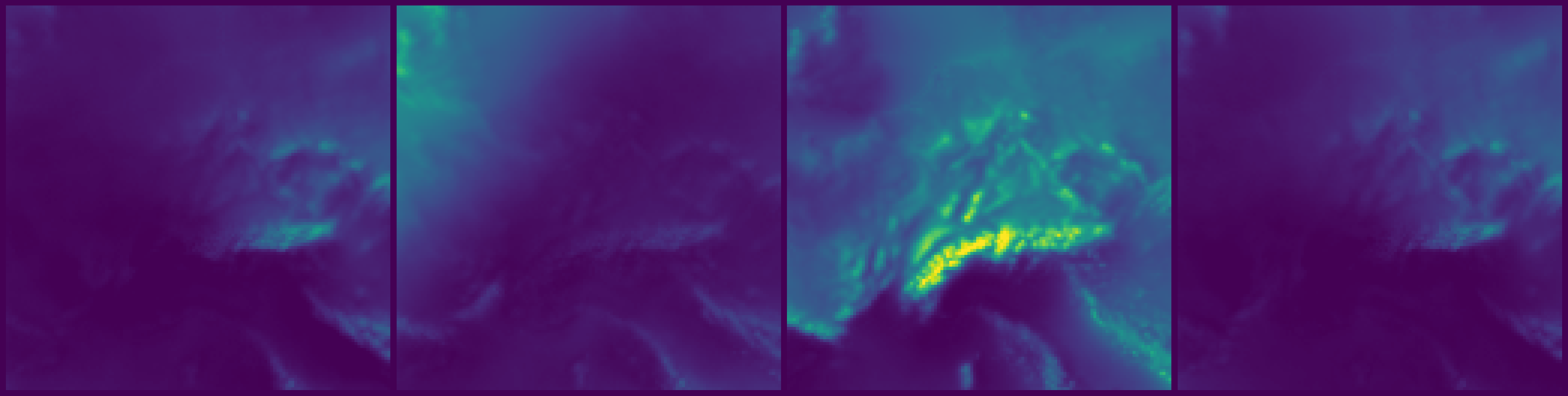} \\
		\rotatebox[origin=c]{90}{\small{DPA}}\hspace{4pt}\smallskip & 
		\includegraphics[width=0.23\textwidth,align=c]{fig/mnist/dpa_i1_k2.png} &
		\includegraphics[width=0.23\textwidth,align=c]{fig/kdisk/dpa_i0_k2.png} &
		\includegraphics[width=0.23\textwidth,align=c]{fig/rcm_t/dpa_i7_k2.png} &
		\includegraphics[width=0.23\textwidth,align=c]{fig/rcm_p/dpa_i0_k2.png} 
	\end{tabular}
	\caption{Reconstructions on \textsc{mnist}, \textsc{disk}, \textsc{r-temp}, and \textsc{r-precip}.}
	\label{fig:add_exp}
\end{figure}

\section{Comparison with VAE}\label{app:vae}
This appendix gives a more detailed comparison between DPA and VAE. The purpose is to make precise in what sense the two methods target different mathematical objects, beyond the brief statement in Section~\ref{sec:related}.

We use notation consistent with the rest of the paper. Denote by $q_e(z|x)$ the density induced by the encoder, by $p_d(x|z)$ the density induced by the decoder, and by $p(z)$ a pre-specified prior on the latent variable. The marginal data density induced by the decoder and prior is
\[p_d(x) = \int p(z)p_d(x|z)\,dz.\]
The VAE objective derives from the standard decomposition of the log-likelihood
\begin{equation*}
	\log p_d(x) = \bbE_{q_e(z|x)}[\log p_d(x|z)] - \mathrm{KL}(q_e(z|x),p(z)) + \mathrm{KL}(q_e(z|x),p_d(z|x)),
\end{equation*}
where $p_d(z|x)\propto p(z)p_d(x|z)$ is the (generally intractable) true posterior. The first two terms form the evidence lower bound (ELBO), which VAE maximises jointly over $e$ and $d$. The gap between the log-likelihood and the ELBO is the third term, the Kullback--Leibler divergence between the variational and true posteriors. To compute the ELBO in closed form, the encoder and decoder families are typically restricted; in the classical Gaussian VAE, encoder, decoder, and prior are all Gaussian.

VAE fits the joint distribution $p(z)p_d(x|z)$ on $(z,x)$. The prior $p(z)$ is prescribed independently of the data; the decoder defines the conditional $p_d(x|z)$. The encoder $q_e(z|x)$ is a variational approximation to $p_d(z|x)$, used as a computational device for variational inference.

DPA fits a different conditional: $X|e(X)=z$, where $e$ is a deterministic encoder defined by the unexplained-variability objective \eqref{eq:obj_enc}. There is no prior on the latent variable; the marginal of $e(X)$ is whatever the data and the encoder induce.

These two conditionals coincide only under specific conditions. For VAE-recovered $p_d(x|z)$ to equal the DPA target $X|e(X)=z$ for some encoder $e$, two things must hold simultaneously: (a) the decoder family must contain a $d^*$ with $p_{d^*}(x)=p^*(x)$, where $p^*$ is the density of $P^*$; and (b) the marginal that $e(X)$ induces must coincide with the prescribed prior $p(z)$. Condition (b) is not enforced by the VAE objective. The standard Gaussian prior used in practice will, for generic data and generic encoders, not match the marginal induced by $e(X)$.

The standard isotropic Gaussian prior used by VAE is rotationally symmetric, so its latent coordinates are exchangeable. There is no natural notion of a `first' or `most informative' latent coordinate within VAE. DPA's principal-ordering property arises from the encoder objective \eqref{eq:obj_enc}, which has no analogue in the VAE formulation. Variants such as $\beta$-VAE \citep{higgins2017beta} adjust the trade-off between reconstruction and KL terms but do not change this feature of the prior.

In addition, in practice VAE rarely achieves a tight ELBO, since the encoder and decoder families are restricted to allow closed-form Kullback--Leibler computation. As a result the variational posterior does not match the true one even at the optimum of the chosen family, which further widens the gap between $p_d(x|z)$ and $X|e(X)=z$. DPA avoids this approximation: the energy score can be computed from samples, so the encoder and decoder families can be expressive generative models of the form $\{d(z,\varepsilon)\}$ without restriction to a closed-form likelihood.

The main differences between DPA and VAE can be summarised as:
\begin{itemize}
	\item \emph{Target.} VAE fits the joint $p(z)p_d(x|z)$ with prescribed prior $p(z)$. DPA targets $X|e(X)=z$ for an encoder defined by minimising unexplained variability of the conditional. These are different conditionals on different conditioning variables.
	\item \emph{Optimisation.} VAE maximises the ELBO, a variational lower bound on the log-likelihood. DPA minimises the expected energy score between the oracle reconstructed distribution and the decoder distribution, with no variational approximation.
	\item \emph{Model class.} VAE restricts encoder and decoder families for closed-form Kullback--Leibler computation. DPA uses general generative model families, since the energy score is computable from samples.
	\item \emph{Latent structure.} The rotationally symmetric Gaussian prior makes VAE latent coordinates exchangeable. DPA imposes an ordering on the latent coordinates that recovers principal components in the linear case and supports an adaptive choice of the latent dimension.
\end{itemize}

\end{document}